\documentclass[11pt]{article}
\usepackage[utf8]{inputenc} 
\usepackage[T1]{fontenc}    
\usepackage{wrapfig}

\usepackage{setspace}
\usepackage{cprotect}
\usepackage{amsmath,amssymb,amsthm}
\usepackage[noend]{algorithmic}
\usepackage[ruled,vlined]{algorithm2e}
\usepackage{hyperref}
\usepackage{fullpage}
\usepackage{makeidx}
\usepackage{enumerate}
\usepackage{graphicx,float,psfrag,epsfig}
\usepackage{epstopdf}
\usepackage{color}
\usepackage{enumitem}
\usepackage{caption}
\usepackage{bigints}
\usepackage{mathtools}
\usepackage[mathscr]{euscript}

\usepackage{float}
\usepackage{tikz}
\usetikzlibrary{shapes,arrows.meta,decorations.markings}
\usepackage{natbib}
\usepackage{caption}
\renewcommand{\algorithmicrequire}{\textbf{Input:}}
\renewcommand{\algorithmicensure}{\textbf{Output:}}
\newcommand{\simonesolved}[1]{\textcolor{black}{#1}}
\newcommand{\simonenew}[1]{\textcolor{black}{#1}}
\usepackage{amsmath}
\usepackage{amssymb}
\usepackage{mathtools}
\usepackage{amsthm}
\usepackage{amsfonts}
\usepackage{soul}
\usepackage{csquotes}
\usepackage{amsmath}
\usepackage{amsthm}
\usepackage{amsfonts}
\usepackage{hyperref}
\usepackage{comment}
\usepackage{graphicx}
\usepackage{xcolor}
\usepackage{parskip}
\usepackage{hyperref}       
\usepackage{url}            
\usepackage{booktabs}       
\usepackage{amsfonts}       
\usepackage{nicefrac}       
\usepackage{microtype}      
\usepackage{xcolor}         
\usepackage{microtype}
\usepackage{graphicx}
\usepackage{subfigure}  
\usepackage{amsfonts}
\usepackage{amsmath}
\usepackage{amssymb}
\usepackage{bbm}
\usepackage{multirow}
\usepackage{mathtools}
\usepackage{relsize}

\newcommand{\clip}{C_{\textup{clip}}}

\newcommand{\diff}{\textup{d}}

\def\Ppm{P_{\Phi_{-1}}}

\def\P{\mathbb{P}}

\def\tr{\mathrm{tr}}

\def\rank{\mathrm{rank}}

\def\R{\mathbb{R}}

\newcommand{\opnorm}[1]{\left\lVert#1\right\rVert_{\textup{op}}}

\def\b0{{0}}

\def\>{\rangle}
\def\rank{\operatorname{\mathop{rank}}}

\newcommand{\E}{\mathbb{E}}

\newcommand{\distas}[1]{\mathbin{\overset{#1}{\sim}}}

\newcommand{\bigO}[1]{\mathcal{O}\left(#1\right)}

\newcommand{\norm}[1]{\left\|#1\right\|}
\newcommand{\subGnorm}[1]{\left\|#1\right\|_{\psi_2}}
\newcommand{\subEnorm}[1]{\left\|#1\right\|_{\psi_1}}

\newcommand{\evmax}[1]{\lambda_{\rm max}\left(#1\right)}
\newcommand{\evmin}[1]{\lambda_{\rm min}\left(#1\right)}

\def\tr{\mathop{\rm tr}\nolimits}

\def\min{\mathop{\rm min}\nolimits}
\def\max{\mathop{\rm max}\nolimits}

\numberwithin{equation}{section}

\newtheoremstyle{myexample} 
    {\topsep}                    
    {\topsep}                    
    {\rm }                   
    {}                           
    {\bf }                   
    {.}                          
    {.5em}                       
    {}  

\newtheoremstyle{myremark} 
    {\topsep}                    
    {\topsep}                    
    {\rm}                        
    {}                           
    {\bf}                        
    {.}                          
    {.5em}                       
    {}  

\newtheorem{claim}{Claim}[section]
\newtheorem{lemma}[claim]{Lemma}

\newtheorem{assumption}{Assumption}

\newtheorem{theorem}{Theorem}
\newtheorem*{theorem*}{Theorem}
\newtheorem{proposition}[claim]{Proposition}

\newtheorem{definition}[claim]{Definition}

\theoremstyle{myremark}

\theoremstyle{myremark}

\theoremstyle{myexample}

\author{Simone Bombari\thanks{Institute of Science and Technology Austria (ISTA). Emails: \texttt{\{simone.bombari, marco.mondelli\}@ist.ac.at}.}\;,
\;\;Marco Mondelli\footnotemark[1]}

\title{Privacy for Free in the Overparameterized Regime}

\begin{document}

\newtheorem*{theoremcentering}{Theorem \ref{thm:maincentering}}
\newtheorem*{theoremcentered}{Theorem \ref{thm:centered}}
\newtheorem*{corhammer}{Corollary \ref{cor:hammer}}
\newtheorem*{cormem}{Corollary \ref{cor:memcap}}
\newtheorem*{theoremoptim}{Theorem \ref{thm:optimization}}

\maketitle

\begin{abstract}

Differentially private gradient descent (DP-GD) is a popular algorithm to train deep learning models with provable guarantees on the privacy of the training data. In the last decade, the problem of understanding its performance cost with respect to standard GD has received remarkable attention from the research community, which formally derived upper bounds on the excess population risk $\mathcal R_{\textup{P}}$ in different learning settings. However, existing bounds typically degrade with over-parameterization, \emph{i.e.}, as the number of parameters $p$ gets larger than the number of training samples $n$ -- a regime which is ubiquitous in current deep-learning practice.
As a result, the lack of theoretical insights leaves practitioners without clear guidance, leading some to reduce the effective number of trainable parameters to improve performance, while others use larger models to achieve better results through scale. In this work, we show that in the popular \emph{random features} model with quadratic loss, for \emph{any} sufficiently large $p$, privacy can be obtained for free, \emph{i.e.}, $\left| \mathcal R_{\textup{P}} \right| = o(1)$, not only when the privacy parameter $\varepsilon$ has constant order, but also in the strongly private setting $\varepsilon = o(1)$. This challenges the common wisdom that over-parameterization inherently hinders performance in private learning.

\end{abstract}

\section{Introduction}

Deep learning models have been shown to be vulnerable to various attacks directed to retrieve information about the training dataset \citep{fredrikson15, carlini2019, carlini2021, haim2022reconstructing}. This vulnerability is a concern when sensitive, private data is included in the training set of a learning pipeline. To allow developers and model providers to still utilize these data, \emph{differential privacy} (DP) \citep{dwork2006} consolidated as the golden standard for 
privacy in learning algorithms. In fact, this framework comes with algorithmic methods \citep{shokri2015privacy, Abadi2016} that provide formal protection guarantees for each individual sample in the training set, which becomes safeguarded (up to some level) by any adversary with access to the trained model and the rest of the dataset.

Neural networks can be trained in a differentially private fashion via algorithms such as DP (stochastic) gradient descent (DP-GD) \citep{Abadi2016}. This method involves minimizing the training loss with additional ``tweaks'' to guarantee protection, which typically boil down to \emph{(i) clipping} the per-sample gradients before averaging, \emph{(ii)} perturbing the parameters updates with \emph{random noise}, and \emph{(iii)} limiting the number of training iterations with \emph{early stopping} (see Algorithm \ref{alg:dp-gd}).
While these adjustments provide privacy guarantees, they often come with a performance cost with respect to standard (stochastic) GD algorithms \citep{tramer2021, Kurakin2022}. The impact on performance tends to worsen as the privacy requirements become stronger, as this requires greater perturbations to the gradients. Thus, private training of deep learning models involves carefully tuning additional hyper-parameters, e.g., the clipping constant, the magnitude of the noise, and the number of training iterations, in order to reach the best performance for a given privacy requirement. Notably, this comes at a remarkable computational cost, also considering the higher training times and memory loads that DP optimization has with respect to standard GD algorithms, mainly due to the computation of the per-sample gradients \citep{de2022}.

The challenging problem of optimizing neural networks with an assigned privacy guarantee has motivated a thriving field of research proposing novel architectures and training algorithms \citep{tramer2021, Papernot2021, de2022}. Concurrently, theoretical studies have emerged with the scope of better understanding and quantifying the \emph{privacy-utility} trade-offs in different learning settings. Here, \emph{privacy} is often defined via the pair of parameters $(\varepsilon, \delta)$: the impact of a single data point on the output of the algorithm is controlled by $\varepsilon$ with probability $1-\delta$, see Definition \ref{def:dp} for the formal introduction of $(\varepsilon, \delta)$-DP. In order to provide meaningful protection, practitioners pick constant-order values of $\varepsilon$, \emph{e.g.}, $\varepsilon\in \{1, 2, 4, 8\}$, and $\delta < 1/n$, where $n$ is the number of training samples \citep{bassily19}.
The \emph{utility}, in this context, is typically measured as the degradation in generalization performance of the DP-trained solution $\theta^p \in \R^p$ compared to a non-private baseline $\theta^* \in \R^p$, where $p$ denotes the number of parameters of the model. In particular, considering the standard supervised setting and denoting by $(x, y) \sim \mathcal P_{XY}$ an input-label pair distributed accordingly to the data distribution $\mathcal P_{XY}$, the \emph{excess population risk} is defined as
\begin{equation}\label{eq:excess}
    \mathcal R_{\textup{P}} = \E_{(x, y) \sim \mathcal P_{XY}} \left[ \ell (x, y, \theta^p) \right] - \E_{(x, y) \sim \mathcal P_{XY}} \left[ \ell (x, y, \theta^*) \right],
\end{equation}
where $\ell \left( x, y, \theta \right)$ is the loss over the sample $(x, y)$ of the model evaluated in $\theta$. Intuitively, the excess population risk worsens with more stringent privacy requirements on $\theta^p$ (\emph{i.e.}, with smaller values of $\varepsilon$ and $\delta$), and a rich line of work spanning over a decade has investigated this trade-off in various settings, see e.g.\ \citep{kifer12, jain14, bassily2014differentially, bassily19, song2021evading, varshney22nearly, brown2024private}.

Despite this flurry of research, existing results are unable to address the over-parameterized regime, \emph{i.e.}, $p = \Omega(n)$. In fact, in this regime, taking $\varepsilon$ of constant order (\emph{e.g.}, $\varepsilon = 1$) typically makes the upper bound on the excess population risk to read (at best) $\mathcal R_{\textup{P}} = {\mathcal O}(1)$, which is vacuous as the performance of a trivial model that constantly outputs zero is of the same order. For additional details, we refer to the discussion on related works in Section \ref{sec:related}. This result is sometimes understood via the qualitative argument that the noise introduced at each iteration by DP-GD increases with the dimension of the parameter space. More specifically, if each entry of $\theta \in \R^p$ is perturbed with standard Gaussian noise, the $\ell_2$ norm of the perturbation scales as $\sqrt p$, which suggests that the privacy-utility trade-off worsens as $p$ grows \citep{yu2021do, mehta2022large}. This motivated a line of research that focuses on DP algorithms acting over lower dimensional subspaces, both in theory \citep{zhou2021bypassing, hilal21private} and in practice \citep{yu2021large, Golatkar2022}.

However, reducing the dimensionality of architectures is in stark contrast with the current trend of increasingly large networks in deep learning. In the non-private setting, the apparent contradiction between excellent generalization of over-parameterized models and classical bias-variance trade-off has been the subject of intense investigation, which has unveiled phenomena such as benign overfitting \citep{bartlett2021deep, bartlett2020benign} and double descent \citep{Belkin19, mei2022generalization}. 

\begin{wrapfigure}{r}{0.5\textwidth}
  \begin{center}
    \includegraphics[width=0.5\textwidth]{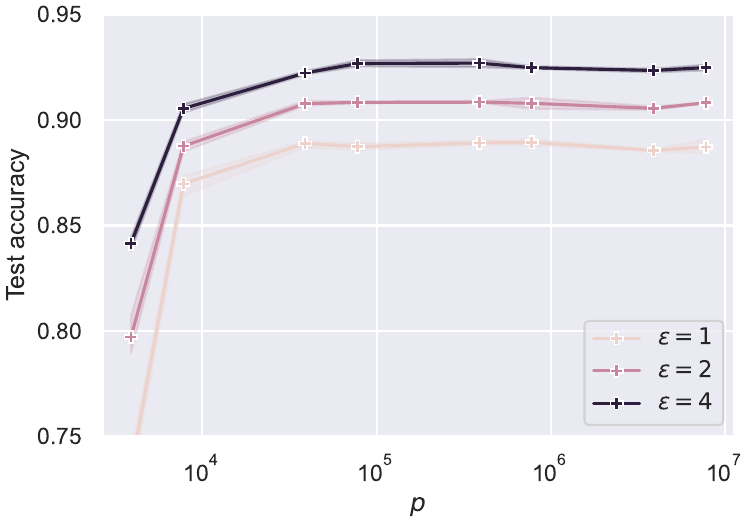}
  \end{center}
  \caption{Test accuracy of DP-GD on MNIST for a 2-layer, fully connected ReLU network, as a function of the number of parameters $p$ with fixed $n = 50000$. Further details on the experimental setting can be found in Section \ref{sec:experiments}.}
  \label{fig:MNISTintro}
\end{wrapfigure}

Recently, larger models have also shown benefits on down-stream tasks requiring private fine-tuning \citep{li2022large, yu2022differentially}, which motivated theoretical studies that provide refined privacy-utility trade-offs in this specific setting \citep{wang2024neural, li2022when}.
Perhaps surprisingly, the recent work by \citet{de2022} gives evidence of the benefits of scale even in the absence of public pre-training data, as the generalization performance is shown to improve with model size on datasets such as CIFAR-10 and ImageNet, following an accurate hyper-parameter search.
In Figure \ref{fig:MNISTintro}, we investigate the interplay between privacy and over-parameterization in a simpler and more controllable setting: training a 2-layer, fully-connected ReLU network on MNIST ($n = 6 \cdot 10^4$) with DP-GD (see Algorithm \ref{alg:dp-gd}). We vary the network width, spanning both the under-parameterized ($p \sim 10^4$) and over-parameterized ($p \sim 10^7$) regime, questioning if the algorithm suffers as the number of parameters grows larger.
The plot shows that this is not the case: the test accuracy initially increases until the network is wide enough ($p \sim 10^5$), and then plateaus without showing any decreasing trend in the right side of the figure.

Then, are larger models truly disadvantageous when doing private optimization? How can the result in the simple experiment in Figure \ref{fig:MNISTintro} be reconciled with existing theoretical work on privacy-utility trade-offs in the over-parameterized regime? How can one solve the disagreement between the experimental evidence in \cite{de2022} and the widespread idea that more parameters hurt the performance of DP-GD?
These questions call for a principled understanding of private training in deep learning. This in particular amounts to \emph{(i)} establishing meaningful generalization guarantees (i.e., $\mathcal R_{\textup{P}}$ significantly smaller than trivial guessing at test time) for over-parameterized models, and \emph{(ii)} characterizing how the hyperparameters introduced by DP-GD (clipping constant, noise magnitude, and number of training iterations) affect performance.

\subsection{Main contribution}

In this work, we provide 
privacy-utility guarantees $\mathcal R_{\textup{P}} = o(1)$ under over-parameterization -- not only when $\varepsilon$ has constant order, but also in the strongly private setting $\varepsilon = o(1)$. We frame this result as achieving \emph{privacy for free} in the over-parameterized regime.

We consider a family of machine learning models where the number of parameters $p$ is a free variable that can be significantly larger than the number of samples $n$ and the input dimension $d$, just as in the experimental setup of Figure \ref{fig:MNISTintro}. This excludes, for example, linear regression, where $p = d$ \citep{varshney22nearly, liu2023near, brown2024private}. Specifically, we focus on the widely studied \emph{random features} (RF) model \citep{rahimi2007random} with quadratic loss, which takes the form
\begin{equation}\label{eq:rfmodelintro}
     \ell(x, y, \theta) =  \left( f_{\textup{RF}}(x, \theta) - y \right)^2, \qquad f_{\textup{RF}}(x, \theta) = \phi (Vx)^\top \theta,
\end{equation}
where $f_{\textup{RF}}$ is a generalized linear model, $\phi: \R \to \R$ a non-linearity applied component-wise to the vector $Vx \in \R^p$, and $V \in \R^{p \times d}$ a random weight matrix. The RF model can be regarded as a 2-layer fully-connected neural network, where only the second layer $\theta$ is trained and the hidden weights $V$ are frozen at initialization. Its appeal comes from the fact that it is simple enough to be analytically tractable and, at the same time, rich enough to exhibit properties occurring in more complex deep learning models \citep{bartlett2021deep,mei2022generalization}.

The DP-trained solution $\theta^p$ is obtained via DP-GD as reported in Algorithm \ref{alg:dp-gd}. Differently from 
\citep{Abadi2016}, we do not perform stochastic batching of the data when aggregating the gradients. 
The cost of privacy typically refers to the performance gap between an algorithm that enforces $(\varepsilon, \delta)$-DP and one that does not. Thus, we set the non-private baseline $\theta^*$ in \eqref{eq:excess} to be the solution obtained by (non-private) GD. In the RF model, the test error of this baseline has been characterized precisely in a recent line of work \citep{mei2022generalization, mmm2022, hu2024asymptotics} and, for a class of data distributions, $\theta^*$ also corresponds to the Bayes-optimal solution given by $\textup{arg\,min}_\theta \E_{(x, y) \sim \mathcal P_{XY}} \left[ \ell (x, y, \theta) \right]$. At this point, we can present an informal version of our result (formally stated as Theorem \ref{thm:mainthm}).
\begin{theorem*}[main result -- informal] Consider the RF model in \eqref{eq:rfmodelintro} with input dimension $d$ and number of features $p$. Let $n$ be the number of training samples and $\mathcal R_{\textup{P}}$ be defined according to \eqref{eq:excess}, where $\theta^*$ is the solution of GD and $\theta^p$ is the $(\varepsilon, \delta)$-differentially private solution of DP-GD (Algorithm \ref{alg:dp-gd}).
Then, for all sufficiently over-parameterized models, under some technical conditions, the following holds with high probability
\begin{equation}\label{eq:RP}
\begin{aligned}
    \left| \mathcal R_{\textup{P}} \right| = \tilde{\mathcal O} \left(\frac{d}{n \varepsilon} + \sqrt{\frac{d}{n}} + \sqrt{\frac{n }{d^{3/2}}} \right) = o(1).
\end{aligned}
\end{equation}
\end{theorem*}

In words, in the regime $d\ll n\ll d^{3/2}$ -- considered e.g.\ in \citep{mmm2022, bombari2023universal}, see Assumption \ref{eq:scalings} for details -- we show that 
$\left| \mathcal R_{\textup{P}} \right| = o(1)$ as long as $\varepsilon \gg d / n$. In fact, when $d\ll n\ll d^{3/2}$ and $\varepsilon \gg d / n$, the three terms $\frac{d}{n \varepsilon}$, $\sqrt{\frac{d}{n}}$ and $\sqrt{\frac{n }{d^{3/2}}}$ appearing in the RHS of \eqref{eq:RP} become $o(1)$. We make two observations: \emph{(i)} as $d\ll n$, our result guarantees vanishing excess population risk, even with a strong privacy requirement $\varepsilon = o(1)$; \emph{(ii)} the bound in \eqref{eq:RP} does not depend on the number of parameters $p$
and Assumption \ref{eq:scalings} only requires a lower bound on $p$,
which allows for arbitrarily large over-parameterization in the model. As typical in the related literature, the dependence of $\mathcal R_{\textup{P}}$ on $\delta$ is only logarithmic and, therefore, it is neglected in the notation $\tilde {\mathcal O} (\cdot)$ that hides poly-logarithmic factors in $\delta$ and $n$.

Achieving \emph{privacy for free} might seem surprising, so we now comment on the intuition behind the result. In the regime $n \gg d$, there is a surplus of samples that can be used to learn privately. In fact, in our model, the test error of (non-private) GD plateaus when $n$ is between $d$ and $d^{3/2}$ \citep{mmm2022, hu2024asymptotics}, hence $\Theta(d)$ samples are enough to achieve \emph{utility}, with the remaining ones playing the role of achieving \emph{privacy}. \simonenew{The plateau in the test loss of GD has been shown for other kernel ridge regression models \citep{ghorbani2021linearized}, and in RF it is exhibited when $d^l \ll n \ll d^{l+1}$ for any $l\in\mathbb N$, as long as $p \gg n$, see Figure 1 of \cite{hu2024asymptotics}. We expect that DP-GD catches up with the performance of GD in any of these plateaus. However, as $n$ approaches $d^{l + 1}$, the test loss of GD sharply decreases and it is unclear whether DP-GD has the same rate of improvement.
This suggests that our result could be extended to the regime $d^l \ll n \ll d^{l+1}$ with  $p \gg n$. In the present paper, we focus on $d \ll n \ll d^{3/2}$ and $p = \Omega(n^2)$ due to the additional challenges in the analysis of clipping, see the discussion after Lemma \ref{lemma:stability} in Section \ref{sec:proof}.}

Figure \ref{fig:MNISTintro2} displays \simonenew{the intuition of having a surplus of samples}. There, we plot the test accuracy of a fixed model as the number of training samples $n$ increases, both for GD and DP-GD\simonenew{, keeping $p$ fixed and much larger than $n$}. In the left side of the figure, the performance of $\theta^*$ sharply increases with $n$, while the private algorithms have very low utility. Moving towards the right side of the plot, the performance of GD tends to saturate, thus approaching the regime where there is a surplus of samples required to learn the task. Here, the extra samples are used to achieve privacy: the utility of DP-GD increases, eventually approaching the non-private baseline.

\begin{wrapfigure}{r}{0.5\textwidth}
  \begin{center}
    \includegraphics[width=0.5\textwidth]{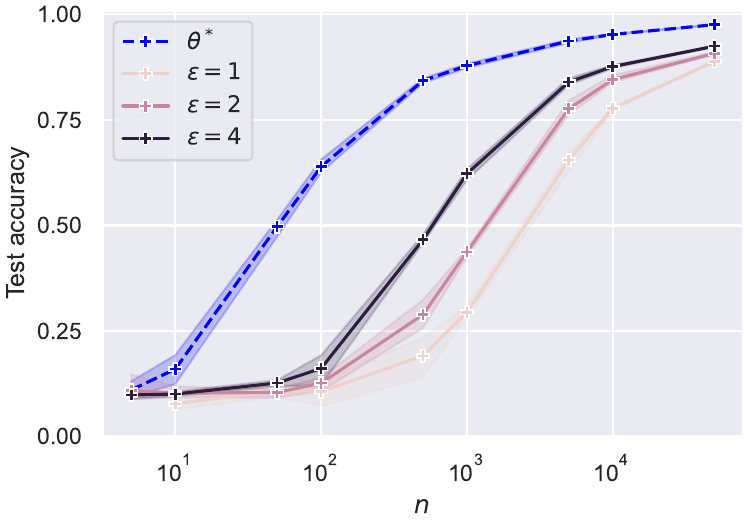}
  \end{center}
  \caption{Test accuracy of DP-GD on MNIST for a 2-layer, fully connected ReLU network, as a function of the number of training samples $n$ with fixed hidden layer width $ = 1000$. Further details on the experimental setting can be found in Section \ref{sec:experiments}.}
  \label{fig:MNISTintro2}
\end{wrapfigure}

Our main result proves that, in the RF model, this heuristic picture holds for any sufficiently large degree of over-parameterization. The key technical hurdle is to formalize these intuitions.
Specifically, while \cite{mmm2022,hu2024asymptotics} establish the asymptotic test error of $\theta^*$ at convergence, we need a \emph{non-asymptotic} control (in terms of $n, d, p$) on the \emph{whole DP-GD trajectory} to understand the impact of clipping and early stopping. This in turn leads to a completely different proof strategy: we first frame DP-GD as the Euler-Maruyama discretization scheme of a stochastic differential equation (SDE); then, exploiting a number of concentration arguments (Dudley, Sudakov-Fernique, and Borell-TIS inequalities), we establish the size of a set containing the full trajectory of the SDE (with high probability), which allows to set the clipping constant; finally, we study the DP-trained solution $\theta^p$ via an Ornstein-Uhlenbeck process, which allows to set the early stopping. Remarkably, our approach provides insights on the problem of hyper-parameter optimization. In fact, as a byproduct of our analysis, we establish the scaling of key hyper-parameters (clipping constant, noise magnitude and number of training iterations) in terms of the learning setup (input dimension $d$, number of parameters $p$) for random features.

The remainder of the paper is organized as follows. In Section \ref{sec:related}, we discuss relevant related work. In Section \ref{sec:dpgd}, we introduce the DP-GD algorithm, and we formally define the DP-trained solution $\theta^p$ as well as the non-private baseline $\theta^*$ that is considered in the excess population risk \eqref{eq:excess}. In Section \ref{sec:mainresult}, we 
formally state our main result (Theorem \ref{thm:mainthm}). In Section \ref{sec:experiments}, we provide supporting numerical evidence, including a discussion on the role of hyper-parameters. 
In Section \ref{sec:proof}, we provide the proof of Theorem \ref{thm:mainthm}, deferring some technical parts to the appendix, and we conclude with a discussion in Section \ref{sec:discussion}.

\section{Related work}\label{sec:related}

Differentially private empirical risk minimization is the problem of minimizing the empirical risk $\mathcal L (\theta) = \sum_{i =1}^n \ell(x_i, y_i, \theta) / n$ while also ensuring privacy guarantees, where $(x_i, y_i)$ represents an input-label pair in the training set 
and $\ell$ is a loss function. Different algorithms have been proposed to tackle this problem, and the performance is 
measured through the excess risks of their solution $\theta^p$ 
with respect to the 
baseline 
$\theta^*$. In particular, we can define the \emph{excess empirical risk} as the difference in the training losses $\mathcal R_{\textup{E}} = \mathcal L (\theta^p) - \mathcal L (\theta^*)$, and the \emph{excess population risk} $\mathcal R_{\textup{P}}$ as the difference in the test losses, see \eqref{eq:excess}. These quantities capture the privacy-utility trade-off of the DP algorithm, as excess risks 
worsen with an increase in the privacy guarantees (\emph{i.e.}, with smaller values of the privacy parameter $\varepsilon$). The seminal works by \citet{chauduri2008, chaudhuri11a} quantitatively investigate this trade-off for strongly convex losses in the context of $\varepsilon$-DP (obtained by setting $\delta=0$), both for output and objective perturbation methods.

\paragraph{Constrained optimization.} \citet{kifer12} extend the analysis of the objective perturbation by \citet{chaudhuri11a}, providing a bound on the empirical excess risk of $\mathcal R_{\textup{E}} = \tilde{\mathcal O}(\sqrt p / (n \varepsilon))$ for constrained, strongly convex $(\varepsilon, \delta)$-DP optimization (see their Theorem 4). The subsequent work by \citet{bassily2014differentially} focuses on a variant of DP-GD previously studied by \cite{williams2010}, and it bounds the excess empirical risk as $\mathcal R_{\textup{E}} = \beta L \, \tilde {\mathcal O} ( \sqrt p / (n \varepsilon ))$ (see their Theorem 2.4.2).\footnote{The definition of empirical excess risk in \cite{bassily2014differentially} is not re-normalized by $n$. 
This makes the bounds of Theorem 2.4.2 therein larger by a factor $n$, which is removed here for the sake of comparison with other works.} Here, 
$\beta$ is the diameter of the (bounded) optimization domain $\mathcal B$, $L$ is the Lipschitz constant of the loss with respect to the parameters of the model, and the baseline is defined as the training loss minimizer $\theta^* = \textup{arg min}_{\theta \in \mathcal B} \mathcal L (\theta)$. We note that, in this setting,
clipping is not necessary as the loss is Lipschitz. 
By relying on the algorithmic stability analysis in \citet{2009shalevstochastic},
\citet{bassily2014differentially} also provide an excess population risk bound of the form $\mathcal R_{\textup{P}} = \tilde {\mathcal O} ( p^{1/4} / \sqrt{n \varepsilon} )$ (see their Theorem F.2.2), where the baseline is the Bayes optimal solution $\theta^* = \textup{arg min}_{\theta \in \mathcal B} \E_{(x, y) \sim \mathcal P_{XY}} \left[ \ell (x, y, \theta) \right]$. Using again algorithmic stability \citep{hardt2016train}, \citet{bassily19} later improve this bound to $\E [ \mathcal R_{\textup{P}} ] = \beta L \, \tilde {\mathcal O} ( 1 / \sqrt n + \sqrt p  / (n \varepsilon))$, where the expectation is taken with respect to the randomness of the training dataset and algorithm.

\paragraph{Unconstrained optimization.}

In the setting of (convex) unconstrained optimization, 
\citet{jain14} aim to remove the dependence of $p$ in excess risk bounds for $(\varepsilon, \delta)$-DP, which was present in earlier work by \citet{chaudhuri11a}. They focus on generalized linear models (GLMs, \emph{i.e.}, 
$\ell(x, y, \theta) = \ell(\varphi(x)^\top \theta, y)$ where $\varphi$ is the feature map), assume $L$-Lipschitz loss with a strongly convex regularization term, and bound the excess population risk with respect to the Bayes optimal solution $\theta^*$ as $\mathcal R_{\textup{P}} = \tilde {\mathcal O} ( L \| \theta^* \|_2 / \sqrt n \varepsilon)$. The more recent work by \cite{song2021evading} also considers unconstrained convex Lipschitz GLMs, without the assumption on the strong convexity.
Using the same non-private baseline $\theta^*$ and denoting by $M$ the projector on the column space of $\E_{x \sim \mathcal P_X} [ \varphi(x) \varphi(x)^\top ]$, the previous bound is improved by $\E [ \mathcal R_{\textup{P}} ] = L \norm{ M \theta^* }_2 \, \tilde {\mathcal O}  \left( 1 / \sqrt{n} +  \sqrt{ \min \left(\rank (M), n \right)} / (\varepsilon n) \right)$ (see their Theorem 3.2).
\citet{li2022when} consider Lipschitz losses with $\ell_2$ regularization and introduce the notion of restricted Lipschitz continuity, which gives dimension-free bounds. However, in the setting of GLMs, this approach leads to the same bounds as in \citep{song2021evading}. A similar approach is taken by \cite{ma2022dimension} that remove the dependence on $p$ at the cost of an additional factor $\tr(\tilde H)$, where $\tilde H \succeq \sup_{\theta} \nabla^2_{\theta} \, \E_{(x, y) \sim \mathcal P_{XY}} [ \ell (x, y, \theta) ]$. The work by \citet{arora2022differentially} considers GLMs with a loss that is not necessarily Lipschitz, and it also recovers the result of \cite{song2021evading} in the Lipschitz setting with $\varepsilon = \Omega(1)$.

Importantly, even in the prototypical case of an RF model, existing bounds do not cover the over-parameterized regime. We now provide a detailed discussion explaining why this is the case. First, note that the RF model in \eqref{eq:rfmodelintro} has a non-Lipschitz (quadratic) loss, which does not allow a direct application of most previous bounds. Thus, to ensure a fair comparison, one can estimate $\norm{\theta^*}_2$ and evaluate the Lipschitz constant of the training loss restricted to a bounded set $\mathcal B$ with radius $\beta = \norm{\theta^*}_2$. This provides the scaling of the effective Lipschitz constant of the model, as if the optimization was bounded to the set $\mathcal B$. From our results in the later sections, it can be shown that $\norm{\theta^*}_2 = \Theta(\sqrt{n / p})$, which in turn gives $\norm{\theta^*}_2 \sup_{\theta \in \mathcal B} \norm{ \nabla_\theta \ell (x_i, y_i, \theta) }_2 = \Theta(n)$. This makes the bound by \citet{jain14} read $\mathcal R_{\textup{P}} = \tilde {\mathcal O} ( \sqrt n / \varepsilon)$, which is vacuous for $\varepsilon = \Theta(1)$. A similar discussion holds for the bound by \cite{bassily19}. Furthermore, even by assuming that the loss function in \eqref{eq:rfmodelintro} is Lipschitz, the result improves only by a factor $\sqrt n$, which gives $\mathcal R_{\textup{P}} = \tilde {\mathcal O} ( 1 / \varepsilon)$ and, therefore, a trivial bound when $\varepsilon = \Theta(1)$. Similar considerations apply to the bound by \cite{song2021evading} (and, thus, to the ones by \cite{li2022when, arora2022differentially}), since in this setting we have that $\norm{M \theta^*}_2 = \norm{\theta^*}_2$ and $\rank(M) \geq n$. As concerns the result by \cite{ma2022dimension}, it can be verified that $\tr (\tilde H) \geq \E_x \left [\norm{\varphi(x)}_2^2 \right] = \Theta(p)$, which reintroduces the dependence on the number of parameters $p$, preventing an improvement upon \citet{jain14}. Finally, while \cite{arora2022differentially} give bounds for quadratic losses, the reasoning resembles the one above on the effective Lipschitz constant in $\mathcal B$ and it does not lead to an improvement with respect to the Lipschitz case. The detailed calculation of the quantities mentioned in this paragraph is deferred to Appendix \ref{app:related}.

\paragraph{Linear regression.} This model reads $\ell(x, y, \theta) = (x^\top \theta - y)^2$, where $y = x^\top \theta^* + \sigma$, $\theta^*$ is the ground-truth solution and $\sigma$ independent label noise. \citet{varshney22nearly} assume that each point is sampled from a standard Gaussian distribution and obtain a sample complexity  of $d / \varepsilon$ via gradient descent with adaptive clipping \citep{andrew2021differentially}. This improves over \citep{wang2018revisiting, cai21, milionis22a}, and it matches the result of earlier work by \citet{liu22b}, which however proposes a computationally inefficient method. \cite{liu2023near} relax the assumption on the data (sub-Weibull concentration) and improve the dependence of the sample complexity on the condition number of the covariance matrix.
\citet{brown2024private} remove the need for adaptive clipping, focusing on a setting where the clipping constant is large enough so that clipping will most likely never happen.

\section{Differentially Private Gradient Descent}\label{sec:dpgd}

The formal definition of DP builds on the notion of \emph{adjacent datasets}. In our supervised learning setting, a dataset $D'$ is said to be adjacent to a dataset $D$ if they differ by only one sample.
\begin{definition}[$(\varepsilon, \delta)$-DP \citep{dwork2006}]\label{def:dp}
A randomized algorithm 
satisfies $(\varepsilon, \delta)$-differential privacy if for any pair of adjacent datasets $D, D'$, 
and for any subset of the parameters space $S \subseteq \R^p$, we have
\begin{equation}
    \P \left( \mathcal A (D) \in S \right) \leq e^{\varepsilon} \P \left( \mathcal A (D') \in S \right) + \delta.
\end{equation}
\end{definition}
Here, the probability is with respect to the randomness induced by the algorithm, and the inequality has to hold uniformly on all the adjacent datasets $D$ and $D'$.

Different methods have been proposed to train machine learning models while enforcing $(\varepsilon, \delta)$-DP. An approach is to rely on the scheme of \emph{output perturbation}, which adds noise to the output of empirical risk minimization \citep{dwork2006, chaudhuri11a, kifer12}. In deep learning, however, it is not easy to characterize the final parameters of the model, and therefore to estimate the minimal amount of noise required to guarantee protection. For this reason, a popular choice relies on DP-GD algorithms, which involve perturbing the individual updates during training time, thus providing privacy guarantees based on the size of the perturbation and the number of iterations \citep{song13dpsgd, bassily2014differentially, Abadi2016}. In this work, we consider Algorithm \ref{alg:dp-gd}, which can be regarded as a variant of the well-established method by \cite{Abadi2016} without stochastic batching. Its privacy guarantees are reported below.
{
\begin{algorithm}[t]
\caption{DP-GD}
\label{alg:dp-gd}
\begin{algorithmic}
\REQUIRE Number of iterations $T$, learning rate $\eta$, clipping constant $\clip$, noise $\sigma$, initialization $\theta_0$.
\FOR{$t \in [T]$}
    \STATE Compute the gradient $g(x_i, y_i, \theta_{t-1}) \leftarrow \nabla_{\theta} \ell(x_i, y_i, \theta_{t-1})$
    \STATE Clip the gradients $g_{\clip}(x_i, y_i, \theta_{t-1}) \leftarrow g(x_i, y_i, \theta_{t-1}) / \max\left(1, \frac{\|g_t(x_i, y_i, \theta_{t-1})\|_2}{\clip}\right)$
    \STATE Aggregate the gradients $g_{\theta_{t-1}} \leftarrow \frac{1}{n} \sum_{i} g_{\clip}(x_i, y_i, \theta_{t-1})$
    \STATE Update the model parameters adding noise $\theta_t = \theta_{t-1} - \eta g_{\theta_{t-1}} + \sqrt{\eta} \, \frac{2\clip}{n} \sigma \mathcal{N}(0, I_p)$
\ENDFOR
\ENSURE Model parameters $\theta_T$.
\end{algorithmic}
\end{algorithm}
}
\begin{proposition}\label{prop:alg1dp}
    For any $\delta \in (0, 1)$, $\varepsilon \in (0, 8 \log (1 / \delta))$, if we set
    \begin{equation}
        \sigma \geq \sqrt {\eta T} \frac{\sqrt{8 \log(1 / \delta)}}{\varepsilon},
    \end{equation}
    then Algorithm \ref{alg:dp-gd} is $(\varepsilon, \delta)$-differentially private.
\end{proposition}
The proof follows the strategy of \cite{Abadi2016} and is based on the moment accountant method. As our formulation is slightly different, we provide the complete argument in Appendix \ref{app:dp}.

While existing work tackles the problem of understanding utility through stability arguments \citep{bassily19, song2021evading}, we propose a novel approach, which 
focuses on the implicit bias of DP-GD. 
The idea is to resort to a continuous process defined by Algorithm \ref{alg:dp-gd} in the limit of small learning rates $\eta \to 0$.
This has 
proven effective in the non-private setting, where the \emph{gradient flow} equation helped understanding the properties of GD in over-parameterized models \citep{mei2022generalization, mmm2022, hu2024asymptotics}.
However, two challenges arise in the analysis of DP-GD: \emph{(i)} the \emph{clipping} of the per-sample gradients, and \emph{(ii)} the injection of \emph{noise}. To overcome the former, we define an auxiliary \emph{clipped loss} $\mathcal L_{\clip}(\theta)$ that incorporates clipping. As for the latter, we consider the stochastic differential equation (SDE) obtained by adding a Wiener process to the gradient flow. This motivates the scaling of the standard deviation of the injected noise as $\sqrt{\eta}$, as e.g.\ considered by \cite{wang19nonconvex}.

More specifically, similarly to \cite{song2021evading}, we note that the clipping step in Algorithm \ref{alg:dp-gd} can be reformulated as the optimization of the surrogate loss $\ell_{i, \clip}(\varphi(x_i)^\top \theta - y_i)$, whose derivative for the $i$-th training sample reads
\begin{equation}\label{eq:lclip}
    \ell_{i, \clip}'(z) = \ell' (z) \min \left( 1, \frac{\clip}{\left| \ell'(z) \right| {\norm{\varphi(x_i)}_2}} \right). 
\end{equation}
This is formalized by the result below, whose proof is deferred to Appendix \ref{app:dp}. 
\begin{proposition}\label{prop:clipandcliploss}
For any $\theta \in \R^p$ and any clipping constant $\clip > 0$, we have that
$g_{\clip}(x_i, y_i, \theta) = \nabla_{\theta} \ell_{i, \clip} \left( \varphi(x_i)^\top \theta_{t-1} - y_i \right)$, 
where $g_{\clip}(x_i, y_i, \theta)$ is defined in the clipping step of Algorithm \ref{alg:dp-gd}.
\end{proposition}
In other words, running Algorithm \ref{alg:dp-gd} is equivalent to running the same algorithm, without the clipping step, on  the \emph{clipped loss} $\mathcal L_{\clip}(\theta) = \frac{1}{n} \sum_{i = 1}^n \ell_{i, \clip}(\varphi(x_i)^\top \theta - y_i)$.
Hence, we can write the $t$-th iteration of Algorithm \ref{alg:dp-gd} as a noisy GD update on $\mathcal L_{\clip}(\theta)$, \emph{i.e.},
\begin{equation}\label{eq:discretization}
    \theta_t - \theta_{t-1} = - \eta \nabla \mathcal L_{\clip}(\theta_{t-1}) + \sqrt \eta \frac{2 \clip}{n} \sigma \mathcal N(0, I_p).
\end{equation}
This update rule corresponds to the Euler-Maruyama discretization scheme of the SDE
\begin{equation}\label{eq:sde1}
    \diff \Theta(t) = - \nabla \mathcal L_{\clip}(\Theta(t)) \diff t + \frac{2 \clip}{n} \sigma \, \diff B(t) = - \nabla \mathcal L_{\clip}(\Theta(t)) \diff t + \Sigma \, \diff B(t),
\end{equation}
with discretization $\eta$ (see Section 10.2 of \cite{kloeden2011numerical}). Here, $B(t)$ represents the standard $p$-dimensional Wiener process, we introduce the shorthand $\Sigma = 2 \clip \sigma / n$, and we set the initial condition of \eqref{eq:sde1} to correspond to the initialization of Algorithm \ref{alg:dp-gd}, \emph{i.e.} $\Theta(0) = \theta_0$.
The strong convergence of the Euler-Maruyama method guarantees that, for any $\tau = \eta T$, the solution $\theta_T$ of Algorithm \ref{alg:dp-gd} approaches the solution $\Theta(\tau)$ of the SDE in \eqref{eq:sde1} as the learning rate $\eta$ gets smaller. 
\simonenew{We note that previous work \citep{paquette2024homogenization} has considered a similar SDE to analyze the effects of stochastic batching, separating the dynamics in a gradient flow plus a Wiener process. Thus, the approach developed here could prove useful also to study DP-SGD, after incorporating an additional independent Wiener process in \eqref{eq:sde1}.}
Importantly, privacy guarantees analogous to those of Proposition \ref{prop:alg1dp} hold for $\Theta(\tau)$, as formalized in the result below (the proof is deferred to Appendix \ref{app:dp}). 
\begin{proposition}\label{prop:dpsde}
    Let $\ell: \R \to \R$ be a differentiable function with Lipschitz-continuous derivative. Then, for any $\delta \in (0, 1)$, $\varepsilon \in (0, 8 \log (1 / \delta))$, if we set
    \begin{equation}\label{eq:Sigmahp}
        \Sigma \geq \frac{2 \clip \sqrt \tau}{n} \frac{\sqrt{8 \log(1 / \delta)}}{\varepsilon},
    \end{equation}
    the solution $\Theta(\tau)$ of the SDE \eqref{eq:sde1} at time $\tau$ is $(\varepsilon, \delta)$-differentially private.
\end{proposition}
Armed with Proposition \ref{prop:dpsde}, the rest of the paper focuses on the analysis of the DP-trained solution $\Theta(\tau)$. In fact, by strong convergence, the privacy-utility trade-offs of $\theta_T$ approach the ones of $\Theta(\tau)$, for sufficiently small learning rates. As a non-private baseline, we consider the solution of the gradient flow equation
\begin{equation}\label{eq:gradientflow}
     \diff \hat \theta(t) = - \nabla \mathcal L(\hat \theta(t)) \diff t, \qquad \theta^* = \lim_{t \to +\infty} \hat \theta(t),
\end{equation}
where $\mathcal L(\theta) = \frac{1}{n} \sum_{i = 1}^n \ell(\varphi(x_i)^\top \theta - y_i)$ is the original training loss (see Section 5.1 of \cite{bartlett2021deep} for more details on the convergence). Note that \eqref{eq:gradientflow} resembles \eqref{eq:sde1}, without the effects of clipping and the introduction of noise, and $\theta^*$ is defined at convergence, as it does not require early stopping.

\section{Main result}\label{sec:mainresult}

For simplicity, we set the initialization $\theta_0 = 0$, although similar results can be obtained for a broad class of initializations. The random features matrix $V \in \R^{p \times d}$ is a normalized Gaussian matrix, \emph{i.e.}, $V_{i,j} \distas{}_{\rm i.i.d.}\mathcal{N}(0, 1 / d)$, and the data distribution is scaled as follows.  
\begin{assumption}[Data distribution]\label{ass:data}
    The training samples $\{ (x_1, y_1), \ldots, (x_n, y_n) \}$ are $n$ i.i.d.\ samples from the joint distribution $\mathcal P_{XY}$, such that the marginal distribution $\mathcal P_X$ satisfies the following properties:
    \begin{enumerate}
    \item $x \sim \mathcal P_X$ is sub-Gaussian, with $\subGnorm{x} = \bigO{1}$.
    \item 
    The data $x \sim \mathcal P_X$ are normalized such that $\norm{x}_2 = \sqrt d$.
    \item $\evmin {\E_{x \sim \mathcal P_X} \left[ xx^\top \right]} = \Omega(1)$, \emph{i.e.}, the second-moment matrix of the data is well-conditioned.
    \end{enumerate}
    Furthermore, we assume all the labels $(y_1, \ldots, y_n)$ to be bounded.
\end{assumption}
The first assumption 
requires the data to have well-behaved tails. It is commonly used in the literature \citep{liu2023near}, and it covers 
a number of important cases such as data respecting Lipschitz concentration \citep{bubeck2021a} (\emph{e.g.}, standard Gaussian \citep{varshney22nearly, brown2024private} or uniform on the sphere \citep{mei2022generalization}). The second assumption (and the boundedness of the labels) can be achieved via data pre-processing. The third assumption is also required in related work to achieve the best rates in linear regression \citep{varshney22nearly, liu2023near}.
The scaling of input data and random features $V$ guarantees that the pre-activations of the model (\emph{i.e.}, the entries of $Vx$) are of constant order. We then process the entries of $Vx$ via the following family of activation functions.
\begin{assumption}[Activation function]\label{ass:activation}
    The activation function $\phi: \R \to \R$ is a non-linear, Lipschitz continuous function such that $\mu_0 = \mu_2 = 0$ and $\mu_1 \neq 0$, where $\mu_k$ denotes the $k$-th Hermite coefficient of $\phi$.
\end{assumption}
This choice is motivated by theoretical convenience, and it covers a wide family of activations, including all odd ones (\emph{e.g.}, $\tanh$). We believe that our result can be extended to a more general setting, as the one in \cite{mmm2022}, at the cost of a more involved analysis.
\begin{assumption}[Parameter scaling]\label{ass:scalings}
\begin{equation}\label{eq:scalings}
    n = \bigO{\sqrt p}, \qquad \log n = \Theta \left( \log p \right), \qquad n = \omega \left( d \log^2 d \right), \qquad  n = o \left( \frac{d^{3/2}}{\log^3 d} \right).
\end{equation}
\end{assumption}
We consider an over-parameterized setting in which the number of parameters $p$ is at least of order $n^2$. To guarantee that the RF model interpolates the data, it suffices that $p\gg n$ (see e.g. \citep{mmm2022,wang2023overparameterized}), and we expect our result to hold under this milder assumption. This would however need a different argument, specifically for Lemma \ref{lemma:stability}. 
The requirement $\log n = \Theta \left( \log p \right)$ is mild and it could be relaxed at the expenses of a poly-logarithmic dependence on $p$ in our final result in Theorem \ref{thm:mainthm}. Finally, we focus on the regime $d\ll n\ll d^{3/2}$, which corresponds to standard datasets, such as CIFAR-10 ($n = 5 \cdot 10^4$, $d \approx 3 \cdot 10^3$), or ImageNet as considered in \citep{zhang2017understanding} ($n \approx 1.3 \cdot 10^6$, $d \approx 9 \cdot 10^4$). 
\begin{assumption}[Privacy budget]\label{ass:privacy}
\begin{equation}
    \delta \in (0, 1), \qquad \varepsilon \in (0, 8 \log (1 / \delta)), \qquad \frac{\varepsilon}{\sqrt{\log(1 / \delta)}} = \omega \left( \frac{d \log^5 n}{n} \right).
\end{equation}
\end{assumption}
The first two conditions on $\delta, \varepsilon$ are standard in differential privacy. The lower bound on $\varepsilon$ coming from the third condition still allows for strong privacy regimes with $\varepsilon = o(1)$, as $n\gg d$ from Assumption \ref{ass:scalings}. To state our main result, we set the hyper-parameters in the SDE \eqref{eq:sde1} as
\begin{equation}\label{eq:tauSigmaapp}
    \tau = \frac{d \log^2 n}{p}, \qquad \clip = \sqrt p \log^2 n, \qquad \Sigma = \frac{2 \clip \sqrt \tau}{n} \frac{\sqrt{8 \log(1 / \delta)}}{\varepsilon}.
\end{equation}
\vspace{1cm}
\begin{theorem}\label{thm:mainthm}
    Consider the RF model in \eqref{eq:rfmodelintro} with input dimension $d$ and number of features $p$. Let $n$ be the number of training samples and $\mathcal R_{\textup{P}}$ be defined in \eqref{eq:excess}, where $\theta^*$ is given by  \eqref{eq:gradientflow} and $\theta^p$ is 
    the solution $\Theta(\tau)$ of the SDE \eqref{eq:sde1} at time $\tau$, with hyper-parameters set as in \eqref{eq:tauSigmaapp}. Let Assumptions \ref{ass:data}, \ref{ass:activation}, \ref{ass:scalings}, \ref{ass:privacy} hold. Then, we have that $\theta^p$ is $(\varepsilon, \delta)$-differentially private, and that
    \begin{equation}\label{eq:eqmainstatement}
    \begin{aligned}
        \left| \mathcal R_P \right| &= \bigO{\frac{d}{n \varepsilon} \, \log^{5} n \sqrt{\log(1 / \delta)} + \sqrt{\frac{d}{n}} + \sqrt{\frac{n \log^3 d}{d^{3/2}}}} \\
        &= \tilde{\mathcal O} \left(\frac{d}{n \varepsilon} + \sqrt{\frac{d}{n}} + \sqrt{\frac{n }{d^{3/2}}} \right) = o(1),
    \end{aligned}
    \end{equation}
    with probability at least $1 - 2 \exp\left( -c \log^2 n \right)$, 
    where $c$ is an absolute constant.
\end{theorem}
The strong convergence of the Euler-Maruyama discretization scheme guarantees that, for sufficiently small learning rates $\eta$, the solution $\theta_T$ of Algorithm \ref{alg:dp-gd} also satisfies \eqref{eq:eqmainstatement}, which establishes the utility of DP-GD.

\section{Numerical experiments}\label{sec:experiments}

Theorem \ref{thm:mainthm} proves that, in the RF model, over-parameterization is not inherently detrimental in private learning. To verify this numerically, we plot in the first panel of Figure \ref{fig:RFfirst} the test loss of a family of RF models trained on a synthetic dataset via DP-GD, as the number of parameters $p$ increases. As a comparison, we also report the performance achieved by (non-private) GD, which provides the baseline $\theta^*$. 
While the test loss of $\theta^*$ displays the typical double-descent curve \citep{Belkin19,hastie2022surprises}, with the expected peak at the interpolation threshold ($p = n$), the performance of $\theta^p$ steadily improves and, as $p$ increases, it plateaus to a value close to the corresponding test loss of $\theta^*$. This is in agreement with Theorem \ref{thm:mainthm}, which predicts that, for a sufficiently over-parameterized model, there is a small performance gap between 
$\theta^p$ and $\theta^*$.
\simonenew{Furthermore, the lack of an interpolation peak in the test loss of DP-GD points to the regularization offered by this algorithm and it resembles the effect of a ridge penalty, see \citep{raskutti14a} for a connection between ridge and early stopping and also the discussion after Lemma \ref{lemma:earlystopping} in Section \ref{sec:proof}.}

In the first panel of Figure \ref{fig:RFfirst}, the hyper-parameters in DP-GD are chosen to maximize performance, and the scaling of such optimal hyper-parameters is considered in the other panels. Specifically, in the second plot, we take 
$\clip = 0.5 \sqrt p$ and report the test error as a function of the number of iterations $T$, setting the noise according to Proposition \ref{prop:alg1dp} to guarantee the desired privacy budget. 
As suggested by \eqref{eq:tauSigmaapp}, the optimal $T$ minimizing the test error 
decreases with $p$. 
Then, as $\eta T = \tau$ scales proportionally to $d / p$ in \eqref{eq:tauSigmaapp}, we rescale the plot putting $\eta T p / d$ on the $x$-axis of the third panel. The curves now collapse onto each other, confirming the accuracy of the proposed scaling. This enables us to transfer the hyper-parameters from small RF models (where brute-force optimization is affordable) to larger models (where training is more time and resources intensive).
Finally, the heat-map in the rightmost panel displays the results of a full hyper-parameter grid search over $(\clip, T)$ for a fixed $p$ and $\varepsilon$.
We distinguish 4 regions in hyper-parameters space: in the \emph{(1) top-right} we have very low utility due to the large noise required by the high number of training iterations $T$ and large clipping constants $\clip$. In the \emph{(2) bottom-left} the test loss is close to 1 as at initialization, since the model does not have the time to learn due to low $T$ and small $\clip$. In the \emph{(3) bottom-right} we have larger $\clip$ which 
could allow for faster convergence. However, as $\clip$ becomes 
larger than the typical per-sample gradient size ($\clip \gg \sqrt p$), the overly-pessimistic injection of noise ultimately undermines utility. In the \emph{(4) top-left} we have \emph{high} $T$ and \emph{small} $\clip$, which lead to low test loss. At the center of the panel, we have $\clip \sim \sqrt p$ and $\eta T \sim d / p$ as in \eqref{eq:tauSigmaapp}, which also lead to low test loss. Moving towards the upper-left part of the plot there is no decrease in performance, as the slower rate of convergence (smaller $\clip$ and thus smaller gradients) is compensated by larger $T$, keeping fixed the overall amount of injected noise. Thus, the model has enough time to learn (differently from the bottom-left), while the added noise does not catastrophically degrade utility (differently from the top-right).
This behaviour is in line with earlier empirical work \citep{Kurakin2022, li2022large, de2022} noting that wide ranges of $\clip$ result in optimal performance. 
\begin{figure}
  \begin{center}
    \includegraphics[width=\textwidth]{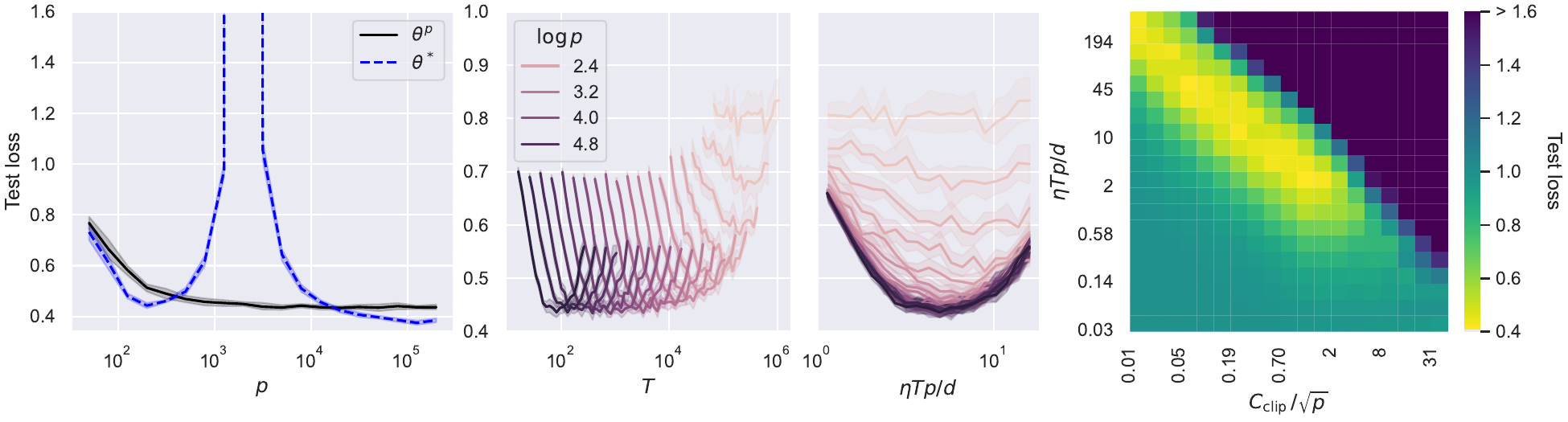}
  \end{center}
  \caption{Experiments on RF models with $\tanh$ activation, and synthetic data sampled from a standard Gaussian distribution with $d = 100$. The learning task is given by $y = \textup{sign}(u^\top x)$, where $u \in \R^d$ is a fixed vector sampled from the unit sphere, and we consider a fixed number of training samples $n = 2000$. $\theta^p$ is the solution of Algorithm \ref{alg:dp-gd} with $\varepsilon = 4$, $\delta = 1/n$, and $\theta^*$ is the solution of GD, both with small enough learning rate $\eta$. \emph{First panel:} test losses of $\theta^p$ and $\theta^*$ for different number of parameters $p$. \emph{Second panel:} test loss for $\theta^p$ 
  as a function of the number of training iterations $T$. \emph{Third panel:} Same plot as in the second panel, with the $x$-axis set to be $\eta T p / d$. \emph{Fourth panel:} test loss of $\theta^p$ for a fixed $p= 40000$, as a function of the hyper-parameters $(\clip, T)$.}
  \label{fig:RFfirst}
\end{figure}

\begin{wrapfigure}{r}{0.6\textwidth}
  \begin{center}
    \vspace{-1em}\includegraphics[width=0.6\textwidth]{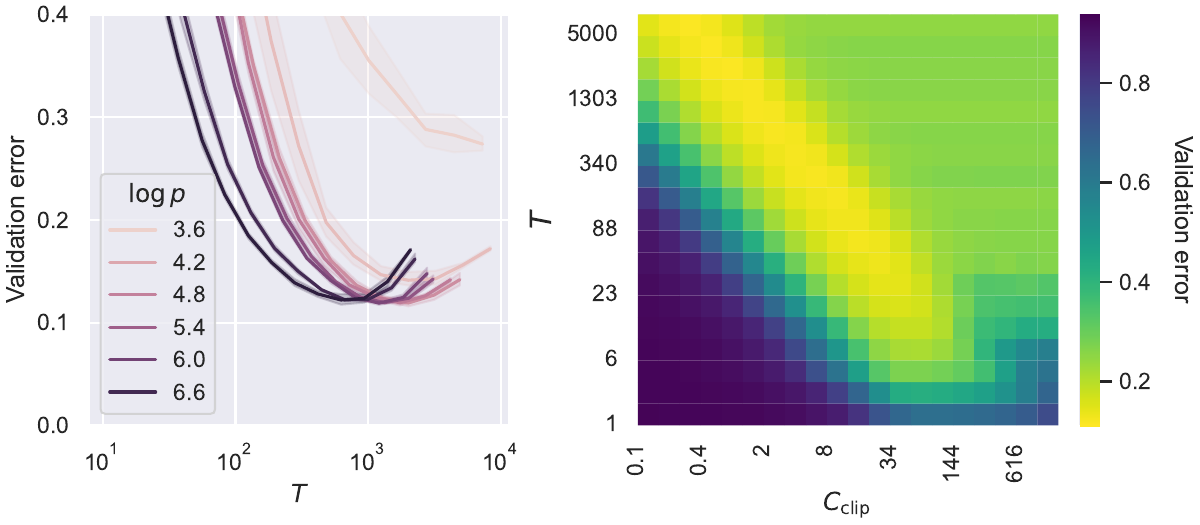}
  \end{center}
  \caption{Experiments on a family of 2-layer fully-connected ReLU networks trained with cross-entropy loss on the MNIST classification task ($d=768$ and we fix $n = 50000$), with privacy budget $\varepsilon = 1$, $\delta = 1/n$. \emph{First panel:} validation error as a function of the number of training iterations $T$. \emph{Second panel:} validation error for a hidden-layer width of 1000 ($p \sim 10^6$), as a function of the hyper-parameters $(\clip, T)$.}
  \vspace{-1em}
  \label{fig:NNcomplete}
\end{wrapfigure}

A similar picture emerges from the training of 2-layer neural networks with DP-GD and cross-entropy loss on MNIST. Figure \ref{fig:MNISTintro} already showed that the performance of $\theta^p$ steadily improves and then plateaus as $p$ increases, thus demonstrating  that DP-GD benefits from larger models, contrary to the view that more parameters 
hinder privacy. Figure \ref{fig:NNcomplete} then considers the scaling of the hyper-parameters (which are chosen to maximize performance in Figure \ref{fig:MNISTintro}). Specifically, in the first panel, we set $\clip = 1$ and report the validation error as a function of  $T$; in the second panel, we present the full hyper-parameter grid search over $(\clip, T)$. 
The results are qualitatively similar to those presented in Figure \ref{fig:RFfirst} for random features: on the left, the optimal $T$ decreases with $p$; on the right, the grid exhibits the 4 regions discussed above. The implementation of the experiments is publicly available at the GitHub repository \href{https://github.com/simone-bombari/privacy-for-free}{\texttt{https://github.com/simone-bombari/privacy-for-free}}.

\section{Proof of Theorem \ref{thm:mainthm}}\label{sec:proof}

\paragraph{Notation.} All logarithms are in the natural basis. Random variables are defined on different probability spaces, which describe the randomness in the data, the random features and the private mechanisms. 
All complexity notations 
are understood for sufficiently large data size $n$, input dimension $d$ and number of parameters $p$. 
We indicate with $C$ and $c$ absolute, strictly positive, numerical constants, that do not depend on the scalings of the problem and whose value may change from line to line.
Given a positive number $s$, $[s]$ denotes the set of positive numbers from $1$ to $s$. 
Given a vector $v$, $\norm{v}_2$ denotes its Euclidean norm. Given a matrix $A$, 
$\opnorm{A}$ denotes its operator norm, 
$\norm{A}_F$ its Frobenius norm and 
$A^+$ its Moore-Penrose inverse.
We use 
$X \in \R^{n \times d}$ to denote the data matrix (containing the $i$-th element of the training set $x_i \in \R^d$ in its $i$-th row), 
$\Phi \in \R^{n \times p}$ to denote the feature matrix (containing $\varphi (x_i):= \phi (Vx) \in \R^p$ in its $i$-th row) and $K := \Phi \Phi^\top \in \R^{n \times n}$ to denote the kernel associated with the feature map.

\subsection{Technical outline}

To circumvent the difficulty in explicitly solving the SDE in \eqref{eq:sde1} (\emph{i.e.}, in characterizing the probability density function of $\Theta(\tau)$), we instead consider the SDE
\begin{equation}\label{eq:OU1}
    \diff \hat \Theta(t) = - \nabla \mathcal L(\hat \Theta(t)) \diff t + \Sigma \, \diff B(t) = - \frac{2\Phi^\top}{n} \left( \Phi \hat \Theta(t) - Y \right) \diff t + \Sigma \, \diff B(t),
\end{equation}
where $\mathcal L(\theta)$ is the original (quadratic) training loss and $B(t)$ is the same standard Wiener process as in \eqref{eq:sde1}. The solution of the SDE in \eqref{eq:OU1} is a multi-dimensional Ornstein–Uhlenbeck (OU) process which admits a closed form (see, e.g., Section 4.4.4 in \cite{gardiner1985handbook}).
Let us then define
\begin{equation}\label{eq:definitionsofC0}
    \mathcal C := \left \{  \theta \quad \textup{s.t.} \quad \|\nabla_{\theta} \ell(\varphi(x_i)^\top \theta - y_i)\|_2 < \clip  \quad \textup{ for all }  i \in [n] \right \},
\end{equation}
which corresponds to the subset of the parameters space where \emph{clipping does not happen}, \emph{i.e.}, where $\mathcal L_{\clip}(\theta) = \mathcal L(\theta)$. If the full path of the process $\Theta(t)$ happens in this region (\emph{i.e.} $\Theta(t) \in \mathcal C$ for all  $t \in [0, \tau]$), then $\Theta(\tau) = \hat \Theta(\tau)$. This corresponds to the event
\begin{equation}\label{eq:hatthetainC}
    \hat \Theta(t) \in \mathcal C, \textup{ for all } t \in [0, \tau],
\end{equation}
which is easier to control, as $\hat \Theta(t)$ is an OU process. The first part of our proof consists in showing that, for our choice of the hyper-parameters in \eqref{eq:tauSigmaapp}, this event happens with high probability. To do so, we consider the decomposition
\begin{equation}\label{eq:tildeTheta}
    \hat \Theta (t) = \E_B \left [ \hat \Theta(t) \right] + \tilde \Theta(t) = \hat \theta(t) + \tilde \Theta (t),
\end{equation}
where we introduce the notation $\tilde \Theta(t) := \hat \Theta (t) - \E_B [ \hat \Theta(t) ]$ and use that the expectation of an OU process corresponds to the gradient flow $\E_B [ \hat \Theta(t)] = \hat \theta(t)$ defined in \eqref{eq:gradientflow}.
Then, $\mathcal C$ can be characterized as
\begin{equation}\label{eq:definitionsofC}
\begin{aligned}
    \mathcal C &:= \left \{  \theta \quad \textup{s.t.} \quad \norm{\nabla_{\theta} \ell(\varphi(x_i)^\top\theta - y_i)}_2 < \clip  \quad \textup{ for all }  i \in [n] \right \} \\
    &= \left \{  \theta \quad \textup{s.t.} \quad \left| \varphi(x_i)^\top \theta - y_i \right| < \frac{\clip}{2 \norm{\varphi(x_i)}_2}  \quad \textup{ for all }  i \in [n] \right \},
\end{aligned}
\end{equation}
which implies that the probability of the event in \eqref{eq:hatthetainC} is lower bounded by the probability of the event
\begin{equation}\label{eq:actualthesisofclipping}
    \left| \varphi(x_i)^\top \tilde \Theta(t) \right| + \left| \varphi(x_i)^\top \hat \theta(t) - y_i \right| \leq \frac{\clip}{2 \norm{\varphi(x_i)}_2},  \quad  \textup{ for all }  i \in [n], \quad \textup{for all } t \in [0, \tau].
\end{equation}
As we set $\clip = \sqrt p \log^2 n$ (see \eqref{eq:tauSigmaapp}) and $\norm{\varphi(x_i)}_2 = \Theta(\sqrt p)$ with high probability (see the argument in \eqref{eq:concentrationfeaturevector1new} and \eqref{eq:concentrationfeaturevector2new}), 
we therefore want to show that, for all $i \in [n]$ and for all $t \in [0, \tau]$, $| \varphi(x_i)^\top \tilde \Theta(t) | = o (\log^2 n )$ and $| \varphi(x_i)^\top \hat \theta(t) - y_i | = o ( \log^2 n )$ with high probability. The term $| \varphi(x_i)^\top \hat \theta(t) - y_i |$ is handled by the lemma below.
\begin{lemma}\label{lemma:lastcliptraj}
    Let Assumptions \ref{ass:data}, \ref{ass:activation} and \ref{ass:scalings} hold. Then, we have that, jointly for all $i \in [n]$,
    \begin{equation}\label{eq:res1}
         \sup_{t \in [0, \tau]} \left| \varphi(x_i)^\top \hat \theta(t) - y_i \right| 
         = \bigO{\log n},
    \end{equation}
     with probability at least $1 - 2 \exp \left( -c \log^2 n\right)$ over $V$ and $X$, where $c$ is an absolute constant.
\end{lemma}

We now briefly sketch the argument to obtain \eqref{eq:res1}, deferring the full proof to Section \ref{sec:proof12}. 
As the initialization $\theta_0 = 0$, the solution of the gradient flow in \eqref{eq:gradientflow} takes the form $\hat \theta(t) = ( 1 - e^{- \frac{2 \Phi^\top \Phi}{n}  t  }) \Phi^+ Y$. Since $K$ is invertible with high probability (due to Lemma \ref{lemma:hyperK}), the quantity of interest can be expressed as $$y_i - \varphi(x_i)^\top \hat \theta(t) = \varphi(x_i)^\top  e^{- \frac{2 \Phi^\top \Phi}{n}  t  } \Phi^+ Y.$$ Our approach consists in iteratively controlling this quantity introducing auxiliary \emph{leave-one-out} variables, \emph{i.e.}, $\Phi_{-i} \in \R^{(n-1) \times p}$ and $Y_{-i} \in \R^{n-1}$, defined without the $i$-th sample.
The first step involves showing that $$\norm{ \Phi^+ Y - \Phi_{-i}^+ Y_{-i} }_2 = \tilde{\mathcal O} \left( p^{-1/2} \right),$$ which implies that we can focus the analysis on $\sup_{t \in [0, \tau]} | \varphi(x_i)^\top e ^{-\frac{2 \Phi^\top \Phi}{n} t} \Phi_{-i}^+ Y_{-i}|$. This result comes from the \emph{stability} of GD, as shown in Lemma \ref{lemma:stability}, and is a consequence of lower bounding the smallest eigenvalue of the kernel $\evmin{K}$.

The second step aims to introduce a leave-one-out variable in the exponent, as we show that $$\sup_{t \in [0, \tau]} \left| \varphi(x_i)^\top e ^{-\frac{2 \Phi^\top \Phi}{n} t} \Phi_{-i}^+ Y_{-i} \right| \leq 2 \sup_{t \in [0, \tau]} \left| \varphi(x_i)^\top e ^{-\frac{2 \Phi_{-i}^\top \Phi_{-i}}{n} t} \Phi_{-i}^+ Y_{-i} \right|.$$ This result is achieved via an explicit computation based on \emph{Lie's product formula} for the matrix exponential, which is carried through in Lemma \ref{lemma:usingLie}.

Finally, since the only dependence on $x_i$ is via the term $\varphi(x_i)^\top$, we can conclude the argument with \emph{Dudley's (chaining tail) inequality}, upper bounding the covering number of the set described by the curve $\gamma(t) = V^\top e ^{-\frac{2 \Phi_{-i}^\top \Phi_{-i}}{n} t} \Phi_{-i}^+ Y_{-i}$  in Lemma \ref{lemma:Tisgood}.

Next, we handle the term $| \varphi(x_i)^\top \tilde \Theta(t)|$ via the lemma below, also proven in Section \ref{sec:proof12}. 
\begin{lemma}\label{lemma:lastclipnoise}
Let Assumptions \ref{ass:data}, \ref{ass:activation}, and \ref{ass:privacy} hold. 
Then, we have that, jointly for all $i \in [n]$,
\begin{equation}\label{eq:res2}
    \sup_{t \in [0, \tau] } \left| \varphi(x_i)^\top \tilde \Theta(t) \right| = \bigO {\log n},
\end{equation}
with probability at least $1 - 2 \exp\left(-c \log ^2 n\right)- 2n \exp\left(-c p \right)$ over $B$ and $V$, where $c$ is an absolute constant and $B$ refers to the probability space of the private mechanism, \emph{i.e.}, the noise in \eqref{eq:OU1}.
\end{lemma}
To prove \eqref{eq:res2}, we start by noticing that $\varphi(x_i)^\top \tilde \Theta(t)$ evolves as a Gaussian random variable with time-dependent variance. The idea is to upper bound this variance with that of the auxiliary process
$$
\diff z_i(t) = \varphi(x_i)^\top \Sigma \, \diff B(t),
$$
which is obtained by removing the \emph{attractive drift} $[- 2\Phi^\top ( \Phi \hat \Theta(t) - Y ) / n ]$ from \eqref{eq:OU1}. Intuitively, the removal of the attractive drift increases the variance, as the process becomes less concentrated around the mean. This is formalized in Lemma \ref{lemma:lessthanB}, where an application of the \emph{Sudakov-Fernique inequality} gives that 
\begin{equation}\label{eq:comp}    
\E_B \left[ \sup_{t \in [0, \tau] } \left| \varphi(x_i)^\top \tilde \Theta(t) \right| \right] \leq \E_{z_i} \left[ \sup_{t \in [0, \tau]} \left| z_i(t) \right| \right].
\end{equation}
We note that the RHS of \eqref{eq:comp} is easier to control than the LHS, as $z_i(t)$ is a Wiener process and, therefore, its variance is bounded by  $\bigO{\Sigma^2 \tau \norm{\varphi(x_i)}_2^2}$ via the \emph{reflection principle}. Then, given our hyper-parameter choice in \eqref{eq:tauSigmaapp} and the lower bound on $\varepsilon$ in Assumption \ref{ass:privacy}, we have that $\Sigma^2 \tau \norm{\varphi(x_i)}_2^2 = \bigO{1}$. To conclude, we exploit the \emph{Borell-TIS inequality} to show that, with high probability, $$\sup_{t \in [0, \tau] } |  \varphi(x_i)^\top \tilde \Theta(t) | \leq \E_B [ \sup_{t \in [0, \tau] } | \varphi(x_i)^\top \tilde \Theta(t) | ] + \log n,$$ which concludes the argument.  

The combination of Lemma \ref{lemma:lastcliptraj} and \ref{lemma:lastclipnoise} gives that
the event in \eqref{eq:hatthetainC} happens with high probability, which allows us to study the utility of $\Theta(\tau)$ through the closed form of the OU process $\hat \Theta(\tau)$, as $\Theta(\tau) = \hat \Theta(\tau)$. This, in turn, boils down to controlling the effects of the noise and of the early stopping that are decoupled via the decomposition 
$\hat \Theta(\tau) = \hat \theta(\tau) + \tilde \Theta (\tau)$. In fact, on the one hand, $\tilde \Theta (\tau)$ is a mean-0 random variable in the probability space of $B$ and it captures the effect of the noise; on the other hand, $\hat \theta(\tau)$ 
is the deterministic component (with respect to $B$) describing the flow and it captures the effect of the early stopping. 
We show that both the noise and the early stopping provide negligible damage to utility in the following two lemmas, whose proofs are contained in Section \ref{sec:proof34}.
\begin{lemma}\label{lemma:noise}
    Let Assumptions \ref{ass:data}, \ref{ass:activation} and \ref{ass:privacy} hold, and let $d = o(p)$. Then, we have
    \begin{equation}\label{eq:res3}
        \E_{x \sim \mathcal P_X} \left[ \left( \varphi(x)^\top \tilde \Theta(\tau) \right)^2 \right] = \bigO{ \frac{d^2 \log^{10} n}{n^2} \, \frac{\log(1 / \delta)}{\varepsilon^2}} = \tilde{\mathcal O} \left( \frac{d^2}{\varepsilon^2 n^2} \right) = o(1),
    \end{equation}
    with probability at least $1 - 2 \exp \left(-c \log^2 n \right)$ over $V$ and $B$, where $c$ is an absolute constant.
\end{lemma}
 Similarly to Lemma \ref{lemma:lastclipnoise}, Lemma \ref{lemma:noise} also uses that
    $\varphi(x_i)^\top \tilde \Theta(\tau)$ is a Gaussian random variable with variance increasing linearly in $\norm{\varphi(x_i)}_2^2$, $\tau$, and $\Sigma^2$. Then, the claim in \eqref{eq:res3} is a consequence of the choice of the hyper-parameters in \eqref{eq:tauSigmaapp} and Assumption \ref{ass:privacy}.
\begin{lemma}\label{lemma:earlystopping}
    Let Assumptions \ref{ass:data}, \ref{ass:activation} and \ref{ass:scalings} hold. Then, we have
    \begin{equation}\label{eq:res4}
        \E_{x \sim \mathcal P_X} \left[ \left( \varphi(x)^\top \left( \hat \theta(\tau) - \theta^* \right)\right)^2 \right] = \bigO{\frac{d}{n} + \frac{n \log^3 d}{d^{3/2}}} = \tilde {\mathcal O} \left( \frac{d}{n} + \frac{n}{d^{3/2}} \right) = o(1),
    \end{equation}
    with probability at least $1 - 2 \exp \left( -c \log^2 n \right)$ over $X$ and $V$, where $c$ is an absolute constant.
\end{lemma}
The idea of the argument is to decompose the LHS of \eqref{eq:res4} through two disjoint subspaces, \emph{i.e.}, $\varphi(x)^\top (P_\Lambda + P_\Lambda^\perp) ( \hat \theta(\tau) - \theta^* )$, where $P_\Lambda \in \R^{p \times p}$ is the projector on the space spanned by the eigenvectors associated with the $d$ largest eigenvalues of $\Phi^\top \Phi$. The rationale is that there is a \emph{spectral gap} between the $d$-th and the $(d + 1)$-th eigenvalue of the kernel $K$, when sorting them in non-increasing order, see Lemma \ref{lemma:hyperK}. \simonenew{We note that this result also uses the well conditioning of the data covariance ($\evmin {\E_{x \sim \mathcal P_X} \left[ xx^\top \right]} = \Omega(1)$), see the discussion right after the proof of Lemma \ref{lemma:earlystopping} in Section \ref{sec:proof}.}
As a consequence of the spectral gap, the term $\| P_\Lambda ( \hat \theta(\tau) - \theta^*)\|_2$ is negligible, since in this subspace $\hat \theta(\tau)$ is \emph{already close to convergence}, despite the early stopping. To control the other subspace, we consider the decomposition $\phi(Vx) = \mu_1 Vx + \tilde \phi(Vx)$, where $\mu_1$ is the first Hermite coefficient of $\phi$. Then, we show that \emph{(i)} $\E_{x} [ ( \tilde \phi(Vx)^\top P_\Lambda^\perp ( \hat \theta(\tau) - \theta^* ))^2 ] = \mathcal O (d / n +  n \log^3 d / d^{3/2})$, exploiting the upper bound on $\| \E_{x} [ \tilde \phi(Vx) \tilde \phi(Vx)^\top ] \|_{\textup{op}}$ in Lemma \ref{lemma:Eopnormsmall}, and \emph{(ii)} $\E_{x} [ ( x^\top V^\top  P_\Lambda^\perp ( \hat \theta(\tau) - \theta^* ))^2] = \bigO{d/n + n \log^3 d / d^{3/2}}$ exploiting the bound on $\| V^\top  P_\Lambda^\perp \Phi^+ \|_{\textup{op}}$ in \simonesolved{Lemma \ref{lemma:earlystoppingtermnew}.} 

Finally, denoting by $\hat{\mathcal R}$ and $\mathcal R^*$ the generalization error of $\hat \Theta(\tau)$ and $\theta^*$ respectively, Lemmas \ref{lemma:noise}-\ref{lemma:earlystopping} (together with some additional manipulations) guarantee that $$| \hat{\mathcal R} - \mathcal R^* | = \tilde{\mathcal O} \left( \frac{d}{n \varepsilon} + \sqrt{\frac{d}{n}} + \sqrt{\frac{n}{d^{3/2}}}\right).$$ As $\Theta(\tau) = \hat \Theta(\tau)$ (due to Lemmas \ref{lemma:lastcliptraj}-\ref{lemma:lastclipnoise}), the result of Theorem \ref{thm:mainthm} follows. The details are deferred to the end of Section \ref{sec:proof34}.

\subsection{Analysis of clipping}\label{sec:proof12}

This section contains the proofs of Lemmas \ref{lemma:lastcliptraj} and \ref{lemma:lastclipnoise}, including a number of auxiliary results required by the argument.
We make use of the notation introduced in the previous section, including the leave-one-out variables ($\Phi_{-1} \in \R^{(n-1) \times p}$, $Y_{-1} \in \R^{n-1}$, $X_{-1} \in \R^{(n-1) \times d}$, $K_{-1} \in \R^{(n-1) \times (n-1)}$) and 
the OU process $\hat \Theta(\tau) = \hat \theta(\tau) + \tilde \Theta (\tau)$. We also consider the hyper-parameter choice in \eqref{eq:tauSigmaapp}. When stating that a random variable $Z$ is sub-Gaussian (sub-exponential), we implicitly mean $\subGnorm{Z} = \bigO{1}$ ($\subEnorm{Z} = \bigO{1}$), \emph{i.e.}, its sub-Gaussian (sub-exponential) norm 
does not increase with the scalings of the problem.
Given a p.s.d.\ matrix $A$, $\lambda_j(A)$ denotes its $j$-th eigenvalue sorted in non-increasing order ($\evmax{A} = \lambda_1(A) \geq \lambda_2(A) \geq \ldots \geq \lambda_s(A) = \evmin{A}$).
We indicate with $P_\Lambda \in \R^{p \times p}$ the projector on the space spanned by the eigenvectors associated with the $d$ largest eigenvalues of $\Phi^\top \Phi$. This quantity is well defined due to Lemma \ref{lemma:hyperK} which proves a spectral gap, and we will condition on the result of this lemma throughout the paper. 
It is convenient to define the function $\tilde \phi(z) := \phi(z) - \mu_1 z$, where $\mu_1$ is the first Hermite coefficient of $\phi$. Note that this function is Lipschitz continuous and has the first 3 Hermite coefficients equal to 0, due to Assumption \ref{ass:activation}. We then define the shorthands $\tilde \varphi(x) = \tilde \phi (Vx) \in \R^p$, $\tilde \Phi \in \R^{n \times p}$ as the matrix containing $\tilde \varphi (x_i) \in \R^p$ in its $i$-th row, and $\tilde K = \tilde \Phi \tilde \Phi^\top \in \R^{n \times n}$.

\begin{lemma}\label{lemma:stability}
Let Assumptions \ref{ass:data} and \ref{ass:activation} hold, and let $n = \bigO{\sqrt p}$, $n \log^3 n =\bigO{d^{3/2}}$ and $n = \omega(d)$. Then, we have
\begin{equation}
    \norm{\Phi^+ Y - \Phi_{-1}^+ Y_{-1}}_2 = \bigO{\frac{\log n}{\sqrt p}},
\end{equation}
with probability at least $1 - 2 \exp \left( -c \log^2 n\right)$ over $V$ and $X$, where $c$ is an absolute constant.
\end{lemma}
\begin{proof}
    \simonesolved{From Lemma \ref{lemma:hyperK}, we have that $K$ is invertible with probability at least $1 - 2 \exp \left( -c_1 \log^2 n\right)$ over $V$ and $X$. Conditioning on such high probability event, we have that also $K_{-1}$ is invertible and, therefore, we can write $\Phi^+ = \Phi^\top K^{-1}$ and $\Phi^+_{-1} = \Phi_{-1}^\top K^{-1}_{-1}$.} Thus, from the proof of Lemma 4.1 in \cite{bombari2023stability} (see their Equation (44)), we have
    \begin{equation}\label{eq:fromstability}
        \Phi^+ Y - \Phi_{-1}^+ Y_{-1} = \frac{\Ppm^\perp \varphi(x_1)}{\norm{\Ppm^\perp \varphi(x_1)}_2^2}  \left(y_1 - \varphi(x_1)^\top \Phi_{-1}^+ Y_{-1} \right),
    \end{equation}
    where $\Ppm$ is the projector over the span of the rows of $\Phi_{-1}$. 

    To bound the term $\varphi(x_1)^\top \Phi_{-1}^+ Y_{-1}$, we decompose $\varphi(x_1) = \mu_1 V x_1 + \tilde \varphi(x_1)$. 
    Thus, an application of the triangle inequality gives
    \begin{equation}\label{eq:badlemma1}
        \left| \varphi(x_1)^\top \Phi_{-1}^+ Y_{-1} \right| \leq  \left| \mu_1 x_1^\top V^\top \Phi_{-1}^+ Y_{-1} \right| +  \left| \tilde \varphi(x_1)^\top \Phi_{-1}^+ Y_{-1} \right|.
    \end{equation}
    We bound the two terms in \eqref{eq:badlemma1} separately. As for the first term, we have that $$\norm{V^\top \Phi_{-1}^+ Y_{-1}}_2 \leq \opnorm{V^\top \Phi_{-1}^+} \norm{Y_{-1}}_2 = \bigO{1},$$ 
    where the second step is a consequence of Lemma \ref{lemma:facts}, and holds with probability at least $1 - 2 \exp \left( -c_2 \log^2 n\right)$ over $V$ and $X_{-1}$. In fact, notice that considering $\Phi$ or $\Phi_{-1}$ in Lemma \ref{lemma:facts} does not change the argument, and therefore the final result. Then, conditioning on this high probability event, since $x_1$ is sub-Gaussian, we have that the first term of \eqref{eq:badlemma1} reads
    \begin{equation}\label{eq:badlemma11}
        \left| \mu_1 x_1^\top V^\top \Phi_{-1}^+ Y_{-1} \right| = \bigO{\log n},
    \end{equation}
    with probability at least $1- 2 \exp \left( -c_3 \log^2 n\right)$ over $V$ and $X$. 
    
    As for the second term in \eqref{eq:badlemma1}, by triangle inequality, we have that
    \begin{equation}\label{eq:RFminusK}
        \norm{\left( \Phi_{-1}^+ \right)^\top \tilde \varphi(x_1)}_2 \leq \norm{K_{-1}^{-1} \E_V \left[ \Phi_{-1} \tilde \varphi(x_1) \right] }_2 + \norm{K_{-1}^{-1} \left(  \Phi_{-1} \tilde \varphi(x_1) - \E_V \left[ \Phi_{-1} \tilde \varphi(x_1) \right] \right) }_2.
    \end{equation}
    We bound the two terms in \eqref{eq:RFminusK} separately. As for the first term, using indices $i \in [n-1]$, we have
    \begin{equation}\label{eq:expectednormtilde}
        \E_V \left[ \Phi_{-1} \tilde \varphi(x_1) \right]_i = p \, \E_v \left[ \phi (x_{i + 1}^\top v) \tilde \phi(x_1^\top v) \right] = p \sum_{l = 3}^{+\infty} \mu_l^2 \left( \frac{x_{i + 1}^\top x_1}{d} \right)^l,
    \end{equation}
    where we use the Hermite decomposition of the functions $\phi$ and $\tilde \phi$. Recall that the first 3 Hermite coefficients of $\tilde \phi$ are 0, and the others correspond to the ones of $\phi$, which we denote as $\mu_l$.
    Since $x_1$ is sub-Gaussian and independent from the $x_{i+1}$-s, with $\norm{x_{i+1}}_2 = \sqrt d$, we have that
    \begin{equation}
        \max_{i \in [n -1]} \left| x_{i+1}^\top x_1 \right| \leq \sqrt d \log n,
    \end{equation}
    with probability at least $1 - 2 \exp \left( -c_4 \log^2 n \right)$ over $x_1$. Thus, conditioning on this high probability event, \eqref{eq:expectednormtilde} gives
    \begin{equation}
        \left| \E_V \left[ \Phi_{-1} \tilde \varphi(x_1) \right]_i \right| \leq p \left( \frac{\left| x_{i+1}^\top x_1 \right|}{d} \right)^3 \sum_{l = 3}^{+\infty} \mu_l^2 = \bigO{p \frac{\log^3 n}{d^{3/2}}},
    \end{equation}
    which implies
    \begin{equation}\label{eq:expectedVd32}
         \norm{K_{-1}^{-1} \E_V \left[ \Phi_{-1} \tilde \varphi(x_1) \right] }_2 \leq \opnorm{K_{-1}^{-1}} \norm{\E_V \left[ \Phi_{-1} \tilde \varphi(x_1) \right]}_2 = \bigO{\frac{\sqrt n \log^3 n}{d^{3/2}}},
    \end{equation}
    with probability at least $1 - 2 \exp \left( -c_5 \log^2 n \right)$ over $X$ and $V$. Here, the last passage also uses Lemma \ref{lemma:hyperK}, which provides a lower bound on $\evmin{K_{-1}} \geq \evmin{K} = \Omega(p)$.

As for the second term of \eqref{eq:RFminusK}, we have
    \begin{equation}\label{eq:newforrevisionpweird}
        \left[ \Phi_{-1} \tilde \varphi(x_1) \right]_i - \E_V \left[ \Phi_{-1} \tilde \varphi(x_1) \right]_i = \sum_{k = 1}^p \phi(x_{i + 1}^\top v_k) \tilde \phi(x_{1}^\top v_k) - \E_{v_k} \left[ \phi(x_{i + 1}^\top v_k) \tilde \phi(x_{1}^\top v_k) \right].
    \end{equation}
    All terms in the previous sum are independent, sub-exponential, mean-0 random variables, since both $\phi$ and $\tilde \phi$ are Lipschitz. Thus, Bernstein inequality (see Theorem 2.8.1 of \cite{vershynin2018high}) gives
    \begin{equation}
        \left| \left[ \Phi_{-1} \tilde \varphi(x_1) \right]_i - \E_V \left[ \Phi_{-1} \tilde \varphi(x_1) \right]_i \right| = \bigO{\log n \sqrt p},
    \end{equation}
    with probability at least $1 - 2 \exp \left( -c_6 \log^2 n\right)$ over $V$. Then, performing a union bound over all the indices $i \in [n-1]$, we obtain
    \begin{equation}\label{eq:overparamweird}
        \norm{K_{-1}^{-1} \left(  \Phi_{-1} \tilde \varphi(x_1) - \E_V \left[ \Phi_{-1} \tilde \varphi(x_1) \right] \right) }_2 = \bigO{\frac{1}{p} \sqrt{np} \log n} = \bigO{\sqrt{\frac{n}{p}} \log n},
    \end{equation}
    with probability at least $1 - 2 \exp \left( -c_7 \log^2 n\right)$ over $V$ and $X$. Then, \eqref{eq:overparamweird} and \eqref{eq:expectedVd32} make \eqref{eq:RFminusK} read
    \begin{equation}\label{eq:firstpartthesisoldbadlemma}
        \norm{\left( \Phi_{-1}^+ \right)^\top \tilde \varphi(x_1)}_2 = \bigO{\frac{\sqrt n \log^3 n}{d^{3/2}} + \sqrt{\frac{n}{p}} \log n} = \bigO{\frac{\log n}{\sqrt n}},
    \end{equation}
    with probability at least $1 - 2 \exp \left( -c_8 \log^2 n\right)$ over $V$ and $X$, where the last step is a consequence of $p = \Omega(n^2)$. 
    Plugging \eqref{eq:firstpartthesisoldbadlemma} and \eqref{eq:badlemma11} 
    in \eqref{eq:badlemma1} provides the upper bound
    \begin{equation}\label{eq:secondpartthesisoldbadlemma}
        \left| \varphi(x_1)^\top \Phi_{-1}^+ Y_{-1} \right| = \bigO{\log n},
    \end{equation}
    which holds with probability at least $1 - 2 \exp \left( -c_9 \log^2 n\right)$ over $V$ and $X$.
    

    To bound the term $\norm{\Ppm^\perp \varphi(x_1)}_2$ in \eqref{eq:fromstability}, we can use Lemma B.1. in \cite{bombari2023stability}, which gives
    \begin{equation}\label{eq:projnormfin}
        \norm{\Ppm^\perp \varphi(x_1)}_2 \geq \sqrt{\evmin{K}} = \Omega(\sqrt p),
    \end{equation}
    where the last step is a consequence of Lemma \ref{lemma:hyperK}, and holds with probability at least $1 - 2 \exp \left( -c_{10} \log^2 n\right)$ over $V$ and $X$. The combination of \eqref{eq:fromstability}, \eqref{eq:secondpartthesisoldbadlemma} and \eqref{eq:projnormfin} readily gives the desired result. 
\end{proof}

\simonenew{After performing a union bound on all $i \in [n]$, the result of Lemma \ref{lemma:stability} guarantees that, with high probability, 
\begin{equation*}
    \norm{\Phi^+ Y - \Phi_{-i}^+ Y_{-i}}_2 = \bigO{\frac{\log n}{\sqrt p}}, \qquad \textup{for all } i \in [n].
\end{equation*}
This will be key in our later proof of Lemma \ref{lemma:lastcliptraj}, where we will use Lemma \ref{lemma:stability} and the fact that it holds with probability at least $1 - 2 \exp \left( -c \log^2 n\right)$. 
We note that the proof strategy of Lemma \ref{lemma:stability} is specifically designed to obtain such a probability guarantee, and it differs from previous work, especially in the argument to show
\begin{equation*}
    \| (\Phi_{-i}^+)^\top \tilde \varphi(x_i)\|_2 = \bigO{ \frac{\log n}{\sqrt n} }, \qquad \textup{for all } i \in [n],
\end{equation*}
see \eqref{eq:firstpartthesisoldbadlemma}.
Note that, for a fixed $i$, this can be proved via Markov inequality through an upper bound on the second moment of the LHS of the equation above, which in turn reduces to an upper bound on $\| \mathbb E_{x} [ \tilde \varphi(x) \tilde \varphi(x)^\top ] \|_{\textup{op}}$. This resembles the strategy followed by \cite{mmm2022} (see their Proposition 7.(b)) to prove that this term is negligible when estimating the test loss of the RF model, which is their object of interest. However, this approach would lead to vacuous probability guarantees after performing a union bound on all the training samples, due to the polynomial tail bound given by Markov inequality. To solve the issue, 
we pursue our desired result via the argument between \eqref{eq:RFminusK} and \eqref{eq:expectedVd32}, which leads to
\begin{equation*}
    \frac{\sqrt n \log^3 n}{d^{3/2}} = O \left( \frac{\log n}{\sqrt n} \right),
\end{equation*}
and therefore to our assumption $d^{3/2} \gg n$, see \eqref{eq:scalings}. Similarly, in the argument between \eqref{eq:newforrevisionpweird} and \eqref{eq:overparamweird}, we upper bound $\|(\Phi_{-i}^+)^\top \tilde \varphi(x_i) - \mathbb E_V \left[ (\Phi_{-i}^+)^\top \tilde \varphi(x_i) \right] \|_2$ via Bernstein inequality on the separate indices of the argument vector, which at the end relies on the bound
\begin{equation*}
    \sqrt{\frac{n}{p} \log n} = O \left( \frac{\log n}{\sqrt n}\right),
\end{equation*}
 leading to our assumption $p = \Omega(n^2)$, see \eqref{eq:scalings}.
}

\begin{lemma}\label{lemma:fakeinequality}
Let Assumptions \ref{ass:data} and \ref{ass:activation} hold, and let $n = o \left( p / \log^4 p \right)$, $n = \omega \left( d \log^2 d \right)$, and $n = \bigO{d^{3/2} / \log^3 d}$. Then, we have that, with probability at least $1 - 2 \exp \left(-c \log^2 n \right)$ over $X$ and $V$,
\begin{equation}\label{eq:decclaim}
    \frac{\varphi(x_1)^\top}{\norm{\varphi(x_1)}_2} e^{-\frac{2 \Phi^\top \Phi}{n} t} \frac{\varphi(x_1)}{\norm{\varphi(x_1)}_2} < \frac{\varphi(x_1)^\top }{\norm{\varphi(x_1)}_2} e ^{-\frac{2 \varphi(x_1) \varphi(x_1)^\top}{n} t} \frac{\varphi(x_1)}{\norm{\varphi(x_1)}_2},
\end{equation}
holds uniformly for all $t \in (0, \tau]$, where $c$ is an absolute constant.
\end{lemma}
\begin{proof}
    Let $P_\Lambda$ be the projector on the space spanned by the eigenvectors associated with the $d$ largest eigenvalues of $\Phi^\top \Phi$, $P_0$ be the projector on the kernel of $\Phi^\top \Phi$, and $P_\lambda = I - P_\Lambda - P_0$. 
    By definition, we have that $P_0 \varphi(x_1) = 0$. Thus, since both $P_\Lambda$ and $P_\lambda$ are projectors on eigenspaces of $\Phi^\top \Phi$, we can write
    \begin{equation}\label{eq:fakeineq1}
    \begin{aligned}
        & \frac{\varphi(x_1)^\top}{\norm{\varphi(x_1)}_2} e^{-\frac{2 \Phi^\top \Phi}{n} t} \frac{\varphi(x_1)}{\norm{\varphi(x_1)}_2} \\
        &= \frac{\varphi(x_1)^\top P_\lambda}{\norm{\varphi(x_1)}_2} e^{-\frac{2 P_\lambda \Phi^\top \Phi P_\lambda}{n} t} \frac{P_\lambda \varphi(x_1)}{\norm{\varphi(x_1)}_2} + \frac{\varphi(x_1)^\top P_\Lambda}{\norm{\varphi(x_1)}_2} e^{-\frac{2 P_\Lambda \Phi^\top \Phi P_\Lambda}{n} t} \frac{P_\Lambda \varphi(x_1)}{\norm{\varphi(x_1)}_2} \\
        &\leq  e^{-\lambda t} \frac{\norm{P_\lambda \varphi(x_1)}_2^2}{\norm{\varphi(x_1)}_2^2} +  e^{-\Lambda t} \frac{\norm{P_\Lambda \varphi(x_1)}_2^2}{\norm{\varphi(x_1)}_2^2},
    \end{aligned}
    \end{equation}
    where we have defined the shorthands
    \begin{equation}
        \Lambda = \frac{2 \lambda_d(K)}{n}, \qquad \lambda = \frac{2 \lambda_{\min}(K)}{n}.
    \end{equation}
    Note that the last step of \eqref{eq:fakeineq1} holds as $e^{-\lambda t}$ and $e^{-\Lambda t}$ are the largest eigenvalues of $e^{-\frac{2 \Phi^\top \Phi}{n} t}$ in the subspaces $P_\lambda$ and $P_\Lambda$ are respectively projecting on. Furthermore, the RHS of \eqref{eq:decclaim} reads    \begin{equation}\label{eq:rank1expmatfake}
        \frac{\varphi(x_1)^\top }{\norm{\varphi(x_1)}_2} e ^{-\frac{2 \varphi(x_1) \varphi(x_1)^\top}{n} t} \frac{\varphi(x_1)}{\norm{\varphi(x_1)}_2} = e ^{-\frac{2 \norm{\varphi(x_1)}_2^2}{n} t}.
    \end{equation}

    In remaining part of the proof, we use the following inequality
    \begin{equation}\label{eq:ineqexponint}
        1 - \frac{cz}{e} \geq e^{-cz},
    \end{equation}
    which holds for any $c > 0$, and $0 < z \leq 1 / c$. We also use that, for all $z > 0$,
    \begin{equation}\label{eq:ineqexpon}
        1 - cz < e^{-cz}.
    \end{equation}

    We prove the inequality in \eqref{eq:decclaim} in two disjoint intervals.

    \begin{enumerate}
        \item \underline{$0 < t \leq 1 / \Lambda$}:
        In this interval, we can use \eqref{eq:fakeineq1} and apply \eqref{eq:ineqexponint} twice (note that $1/\Lambda \leq 1/\lambda$), obtaining
        \begin{equation}\label{eq:time1fakeL}
        \begin{aligned}
            \frac{\varphi(x_1)^\top}{\norm{\varphi(x_1)}_2} e^{-\frac{2 \Phi^\top \Phi}{n} t} \frac{\varphi(x_1)}{\norm{\varphi(x_1)}_2} &\leq e^{-\lambda t} \frac{\norm{P_\lambda \varphi(x_1)}_2^2}{\norm{\varphi(x_1)}_2^2} +  e^{-\Lambda t} \frac{\norm{P_\Lambda \varphi(x_1)}_2^2}{\norm{\varphi(x_1)}_2^2} \\
            &\leq \left( 1 - \frac{\Lambda t}{e} \right) \frac{\norm{P_\Lambda \varphi(x_1)}_2^2}{\norm{\varphi(x_1)}_2^2} + \left( 1 - \frac{\lambda t}{e} \right) \frac{\norm{P_\lambda \varphi(x_1)}_2^2}{\norm{\varphi(x_1)}_2^2} \\
            &= 1 - \frac{\Lambda t}{e} \frac{\norm{P_\Lambda \varphi(x_1)}_2^2}{\norm{\varphi(x_1)}_2^2} - \frac{\lambda t}{e}  \frac{\norm{P_\lambda \varphi(x_1)}_2^2}{\norm{\varphi(x_1)}_2^2}.
        \end{aligned}
        \end{equation}

        Applying \eqref{eq:ineqexpon} to \eqref{eq:rank1expmatfake} we obtain
        \begin{equation}\label{eq:time1fakeR}
            \frac{\varphi(x_1)^\top }{\norm{\varphi(x_1)}_2} e ^{-\frac{2 \varphi(x_1) \varphi(x_1)^\top}{n} t} \frac{\varphi(x_1)}{\norm{\varphi(x_1)}_2} > 1 -\frac{2 \norm{\varphi(x_1)}_2^2}{n} t.
        \end{equation}

        Then, on this interval, \eqref{eq:time1fakeL} and \eqref{eq:time1fakeR} imply that proving the following
        \begin{equation}
            1 - \frac{\Lambda t}{e} \frac{\norm{P_\Lambda \varphi(x_1)}_2^2}{\norm{\varphi(x_1)}_2^2} - \frac{\lambda t}{e}  \frac{\norm{P_\lambda \varphi(x_1)}_2^2}{\norm{\varphi(x_1)}_2^2} \stackrel{?}{\leq} 1 - \frac{2 \norm{\varphi(x_1)}_2^2}{n} t,
        \end{equation}
        is enough to prove the thesis. This, in turn, can be shown proving that
        \begin{equation}\label{eq:questionmark1}
            \Lambda = \frac{2 \lambda_d(K)}{n} \stackrel{?}{\geq} \frac{2 e \norm{\varphi(x_1)}_2^4}{n \norm{P_\Lambda \varphi(x_1)}_2^2},
        \end{equation}
        \simonenew{where the ? remarks that this inequality still has to be proved.}
        
        Now, by Lemmas \ref{lemma:hyperK} and \ref{lemma:Pvarphi}, 
        we jointly have
        \begin{equation}
            \lambda_d(K) = \Omega \left( \frac{pn}{d} \right) = \Omega(p \log n), \qquad \norm{P_\Lambda \varphi(x_1)}_2^2 = \Omega(p),
        \end{equation}
        with probability at least $1 - 2 \exp \left(-c_2 \log^2 n \right)$ over $X$ and $V$. Then, conditioning on this high probability event, the LHS of \eqref{eq:questionmark1} is $\Omega(p \log n \, / n)$, while its RHS is $\bigO{p / n}$ (recall that $\norm{\varphi(x_1)}_2^2 = \bigO{p}$). Thus, \eqref{eq:questionmark1} holds for $n$ sufficiently large, which gives that the desired result holds in the interval $0<t\le 1/\Lambda$, with probability at least $1 - 2 \exp \left(-c_3 \log^2 n \right)$ over $X$ and $V$ (where $c_3$ is eventually smaller than $c_2$, to make this probability being 0 for the $n$-s that are not large enough).

        \item \underline{$1 / \Lambda < t \leq \tau$}:
        We have
        \begin{equation}\label{eq:time2fakeL}
        \begin{aligned}
            \frac{\varphi(x_1)^\top}{\norm{\varphi(x_1)}_2} e^{-\frac{2 \Phi^\top \Phi}{n} t} \frac{\varphi(x_1)}{\norm{\varphi(x_1)}_2} &\leq e^{-\lambda t} \frac{\norm{P_\lambda \varphi(x_1)}_2^2}{\norm{\varphi(x_1)}_2^2} +  e^{-\Lambda t} \frac{\norm{P_\Lambda \varphi(x_1)}_2^2}{\norm{\varphi(x_1)}_2^2}\\
            &\leq \frac{\norm{P_\lambda \varphi(x_1)}_2^2}{\norm{\varphi(x_1)}_2^2} +  e^{- 1 } \frac{\norm{P_\Lambda \varphi(x_1)}_2^2}{\norm{\varphi(x_1)}_2^2 } \\
            &= 1 - \left(1 - e^{-1} \right) \frac{\norm{P_\Lambda \varphi(x_1)}_2^2}{\norm{\varphi(x_1)}_2^2 }.
        \end{aligned}
        \end{equation}
        Applying \eqref{eq:ineqexpon} to \eqref{eq:rank1expmatfake} we obtain
        \begin{equation}\label{eq:time2fakeR}
            \frac{\varphi(x_1)^\top }{\norm{\varphi(x_1)}_2} e ^{-\frac{2 \varphi(x_1) \varphi(x_1)^\top}{n} t} \frac{\varphi(x_1)}{\norm{\varphi(x_1)}_2} > 1 -\frac{2 \norm{\varphi(x_1)}_2^2}{n} t \geq 1 -\frac{2 \norm{\varphi(x_1)}_2^2}{n} \tau.
        \end{equation}

        Then, on this interval, \eqref{eq:time2fakeL} and \eqref{eq:time2fakeR} imply that proving the following
        \begin{equation}\label{eq:fake2thesis}
            \left(1 - e^{-1} \right) \frac{\norm{P_\Lambda \varphi(x_1)}_2^2}{\norm{\varphi(x_1)}_2^2 } \stackrel{?}{\geq} \frac{2 \norm{\varphi(x_1)}_2^2}{n} \tau,
        \end{equation}
        is enough to prove the thesis. By Lemma \ref{lemma:Pvarphi}, we have that the LHS of the previous equation is $\Theta(1)$ with probability at least $1 - 2 \exp \left(-c_4 \log^2 n \right)$ over $X$ and $V$. For the RHS, since $\norm{\varphi(x_1)}_2^2 = \bigO{p}$ \simonesolved{with probability at least $1 - 2 \exp\left( -c_5 p \right)$ over $V$ (see the argument carried out in \eqref{eq:concentrationfeaturevector1new} and \eqref{eq:concentrationfeaturevector2new}), we have that}
        \begin{equation}
            \frac{2 \norm{\varphi(x_1)}_2^2}{n} \tau = \bigO{\frac{p}{n} \frac{d \log^2 n}{p}} = o(1).
        \end{equation}
        Thus, \eqref{eq:fake2thesis} holds for $n$ sufficiently large, which gives that the desired result holds in the interval $1/\Lambda<t\le \tau$, with probability at least $1 - 2 \exp \left(-c_6 \log^2 n \right)$ over $X$ and $V$ (where $c_6$ is set to make this probability being 0 for the $n$-s that are not large enough).
    \end{enumerate}
\end{proof}


\begin{lemma}\label{lemma:usingLie}
    Let Assumptions \ref{ass:data} and \ref{ass:activation} hold, and let $n = o \left( p / \log^4 p \right)$, $n = \omega \left( d \log^2 d \right)$, and $n = \bigO{d^{3/2} / \log^3 d}$. Then, we have
    \begin{equation}
            \sup_{t \in [0, \tau]} \left| \varphi(x_1)^\top e ^{-\frac{2 \Phi^\top \Phi}{n} t} \Phi_{-1}^+ Y_{-1}\right| \leq 2 \sup_{t \in [0, \tau]} \left| \varphi(x_1)^\top e ^{-\frac{2 \Phi_{-1}^\top \Phi_{-1}}{n} t} \Phi_{-1}^+ Y_{-1}\right|,
    \end{equation}
    with probability at least $1 - 2 \exp \left(-c \log^2 n \right)$ over $X$ and $V$, where $c$ is an absolute constant.
\end{lemma}
\begin{proof}
Note that, for any positive number $s$,
    \begin{equation}\label{eq:identitylielemma}
        e^{- \frac{2 \varphi(x_1) \varphi(x_1)^\top}{ns} t } = e^{- \frac{\norm{2 \varphi(x_1)}^2_2}{ns} t} \, \frac{ \varphi(x_1) \varphi(x_1)^\top}{\norm{\varphi(x_1)}_2^2} + P^\perp_{\varphi(x_1)},
    \end{equation}
    where $P^\perp_{\varphi(x_1)} \in \R^{p \times p}$ is the orthogonal projector to the space spanned by $\varphi(x_1)$, \emph{i.e.}, $P^\perp_{\varphi(x_1)} = I - \varphi(x_1) \varphi(x_1)^\top / \norm{\varphi(x_1)}_2^2$. 
    Thus, we can write
    \begin{equation}\label{eq:diff1}
        e^{- \frac{2 \varphi(x_1) \varphi(x_1)^\top}{ns} t } = I + \alpha(s) \frac{ \varphi(x_1) \varphi(x_1)^\top}{\norm{\varphi(x_1)}_2^2},
    \end{equation}
    where we introduced the shorthand
    \begin{equation}\label{eq:alphas}
        -1 < \alpha(s) = - \left( 1 - e^{- \frac{2 \norm{\varphi(x_1)}^2_2}{ns} t} \right) < 0.
    \end{equation}
    Let us also introduce
    \begin{equation}
        \Pi(s) = \left( e^{ - \frac{2 \varphi(x_1) \varphi(x_1)^\top}{ns} t} e^{ - \frac{2 \Phi_{-1}^\top \Phi_{-1}}{ns} t } \right)^s \in \R^{p \times p}, \qquad \chi(s) = \left| \varphi(x_1)^\top \Pi(s) \Phi_{-1}^+ Y_{-1}\right| \in \R,
    \end{equation}
    defined for any $s$ being a positive natural number.
    As $\varphi(x_1) \varphi(x_1)^\top + \Phi_{-1}^\top \Phi_{-1} = \Phi^\top \Phi$, an application of Lie's product formula gives
    \begin{equation}\label{eq:limPis}
        \lim_{s \to \infty} \Pi(s) = e ^{-\frac{2 \Phi^\top \Phi}{n} t},
    \end{equation}
    and therefore
    \begin{equation}\label{eq:limchi}
        \lim_{s \to \infty} \chi(s) = \left| \varphi(x_1)^\top e ^{-\frac{2 \Phi^\top \Phi}{n} t} \Phi_{-1}^+ Y_{-1}\right|.
    \end{equation}

    Plugging \eqref{eq:diff1} in the expression of $\Pi(s)$ gives
    \begin{equation}\label{eq:expandingPi}
    \begin{aligned}
        \Pi(s) &= \left( \left( I + \alpha(s) \frac{ \varphi(x_1) \varphi(x_1)^\top}{\norm{\varphi(x_1)}_2^2} \right) A(s) \right)^s, 
    \end{aligned}
    \end{equation}
    where we define
    \begin{equation}\label{eq:defAs}
        A(s) =  e^{ - \frac{2 \Phi_{-1}^\top \Phi_{-1}}{ns} t } \in \R^{p \times p}.
    \end{equation}

    Note that $\Pi(s)$ can be expanded as 
    \begin{equation}
        \Pi(s) =  \left( I + \alpha(s) \frac{ \varphi(x_1) \varphi(x_1)^\top}{\norm{\varphi(x_1)}_2^2} \right) \sum_{l = 1}^{s-1} \Pi_l(s) A(s)^{s-l} +  \left( I + \alpha(s) \frac{ \varphi(x_1) \varphi(x_1)^\top}{\norm{\varphi(x_1)}_2^2} \right) A(s)^s,
    \end{equation}
    where we define
    \begin{equation}
        \Pi_l(s) = \left( A(s)  \left( I + \alpha(s) \frac{ \varphi(x_1) \varphi(x_1)^\top}{\norm{\varphi(x_1)}_2^2} \right) \right)^{l-1} A(s) \alpha(s) \frac{ \varphi(x_1) \varphi(x_1)^\top}{\norm{\varphi(x_1)}_2^2}.
    \end{equation}
    In words, $\Pi_l(s)$ includes all the terms where the last term containing $\alpha(s)$ is taken at the $(l+1)$-th factor of \eqref{eq:expandingPi}.
    This gives
    \begin{equation}\label{eq:bigchiPipi}
    \begin{aligned}
        \chi(s) &= \left| \varphi(x_1)^\top \Pi(s) \Phi_{-1}^+ Y_{-1}\right|  \\
        & = \left| \left( 1 + \alpha(s) \right) \varphi(x_1)^\top \sum_{l = 1}^{s-1} \Pi_l(s) A(s)^{s-l}  \Phi_{-1}^+ Y_{-1} + \left( 1 + \alpha(s) \right) \varphi(x_1)^\top A(s)^s \Phi_{-1}^+ Y_{-1} \right|  \\
        & =\left| \left( 1 + \alpha(s) \right) \alpha(s) \sum_{l = 1}^{s-1} \pi_l(s) \varphi(x_1)^\top A(s)^{s-l}  \Phi_{-1}^+ Y_{-1} + \left( 1 + \alpha(s) \right) \varphi(x_1)^\top  A(s)^s \Phi_{-1}^+ Y_{-1} \right| \\
        &\leq \left| \left( 1 + \alpha(s) \right) \alpha(s) \sum_{l = 1}^{s-1} \pi_l(s) \right| \max_{l \in [s-1]} \left| \varphi(x_1)^\top A(s)^{s-l}  \Phi_{-1}^+ Y_{-1} \right| +\\
        &\qquad + \left| \left( 1 + \alpha(s) \right) \varphi(x_1)^\top  A(s)^s \Phi_{-1}^+ Y_{-1} \right|,
    \end{aligned}
    \end{equation}
    where we introduce the shorthand
    \begin{equation}\label{eq:defpi}
    \begin{aligned}
        \pi_l(s) &= \frac{\varphi(x_1)^\top}{\norm{\varphi(x_1)}_2} \left( A(s)  \left( I + \alpha(s) \frac{ \varphi(x_1) \varphi(x_1)^\top}{\norm{\varphi(x_1)}_2^2} \right) \right)^{l-1} A(s) \frac{\varphi(x_1)}{\norm{\varphi(x_1)}_2} \\
        &= \frac{\left(A(s)^{1/2} \varphi(x_1)\right)^\top}{\norm{\varphi(x_1)}_2} M_l(s) \frac{A(s)^{1/2} \varphi(x_1)}{\norm{\varphi(x_1)}_2},
    \end{aligned}
    \end{equation}
    and $M_l(s)$ is the p.s.d.\ matrix defined as
    \begin{equation}\label{eq:defMl}
        M_l(s) = \left(A(s)^{1/2}  \left( I + \alpha(s) \frac{ \varphi(x_1) \varphi(x_1)^\top}{\norm{\varphi(x_1)}_2^2} \right) A(s)^{1/2}\right)^{l-1}.
    \end{equation}

    Then, note that \eqref{eq:limchi} and \eqref{eq:bigchiPipi} give
    \begin{equation}\label{eq:limsupappears}
    \begin{aligned}
        & \left| \varphi(x_1)^\top e ^{-\frac{2 \Phi^\top \Phi}{n} t} \Phi_{-1}^+ Y_{-1}\right|  \\
        & \leq \limsup_{s \to \infty} \left( \left| \left( 1+ \alpha(s) \right) \alpha(s) \sum_{l = 1}^{s-1} \pi_l(s) \right|  \max_{l \in [s-1]} \left| \varphi(x_1)^\top e ^{-\frac{2 \Phi_{-1}^\top \Phi_{-1}}{ns} (s-l) t} \Phi_{-1}^+ Y_{-1}\right|\right)  \\
        &\qquad + \limsup_{s \to \infty} \left( \left| 1+ \alpha(s) \right| \left| \varphi(x_1)^\top e ^{-\frac{2 \Phi_{-1}^\top \Phi_{-1}}{ns} s t} \Phi_{-1}^+ Y_{-1}\right|\right)  \\
        & \leq \limsup_{s \to \infty} \left| \left( 1+ \alpha(s) \right) \alpha(s) \sum_{l = 1}^{s-1} \pi_l(s) \right|  \sup_{t' \in (0, t)} \left| \varphi(x_1)^\top e ^{-\frac{2 \Phi_{-1}^\top \Phi_{-1}}{n} t'} \Phi_{-1}^+ Y_{-1}\right| \\
        &\qquad + \limsup_{s \to \infty} \left| 1+ \alpha(s) \right|  \left| \varphi(x_1)^\top e ^{-\frac{2 \Phi_{-1}^\top \Phi_{-1}}{n} t} \Phi_{-1}^+ Y_{-1}\right|  \\
        & \le  \left( 1 + \limsup_{s \to \infty} \left| \alpha(s) \sum_{l = 1}^{s-1} \pi_l(s) \right| \right) \sup_{t' \in (0, t)} \left| \varphi(x_1)^\top e ^{-\frac{2 \Phi_{-1}^\top \Phi_{-1}}{n} t'} \Phi_{-1}^+ Y_{-1}\right|,
    \end{aligned}
    \end{equation}
    where in the last line we used $\lim_{s \to \infty} \alpha(s) = 0$, which follows from \eqref{eq:alphas}. We will now upper bound the first factor of this last expression, for all $t \in (0, \tau]$, as we will treat the case $t=0$ separately. Note that, following \eqref{eq:defpi}, we have
    \begin{equation}\label{eq:rhsgeometric}
    \begin{aligned}
        \pi_l(s) &= \frac{\left(A(s)^{1/2} \varphi(x_1)\right)^\top}{\norm{\varphi(x_1)}_2} M_l(s) \frac{A(s)^{1/2} \varphi(x_1)}{\norm{\varphi(x_1)}_2} \\
        &= \frac{\norm{A(s)^{1/2} \varphi(x_1)}_2^2}{\norm{\varphi(x_1)}_2^2} \frac{\left(A(s)^{1/2} \varphi(x_1)\right)^\top}{\norm{A(s)^{1/2} \varphi(x_1)}_2} M_s(s)^{(l-1) / (s-1)} \frac{A(s)^{1/2} \varphi(x_1)}{\norm{ A(s)^{1/2} \varphi(x_1)}_2} \\
        &\leq \frac{\norm{A(s)^{1/2} \varphi(x_1)}_2^2}{\norm{\varphi(x_1)}_2^2} \left( \frac{\left(A(s)^{1/2} \varphi(x_1)\right)^\top}{\norm{A(s)^{1/2} \varphi(x_1)}_2} M_s(s) \frac{A(s)^{1/2} \varphi(x_1)}{\norm{A(s)^{1/2} \varphi(x_1)}_2} \right)^{(l-1) / (s-1)} \\
        & =: \frac{\norm{A(s)^{1/2} \varphi(x_1)}_2^2}{\norm{\varphi(x_1)}_2^2} \mu(s)^{(l-1) / (s-1)},
    \end{aligned}
    \end{equation}
    where the second line follows directly from \eqref{eq:defMl}, and the third line is a consequence of Jensen inequality, since $M_s(s)$ is p.s.d.\ and $l \leq s$. 

    From \eqref{eq:defAs}, we have that $\lim_{s\to\infty} A(s) = I$. Thus, 
    \begin{equation}
    \begin{aligned}
        \lim_{s\to\infty} M_s(s) &= \lim_{s\to\infty} \left(A(s)^{1/2}  \left( I + \alpha(s) \frac{ \varphi(x_1) \varphi(x_1)^\top}{\norm{\varphi(x_1)}_2^2} \right) A(s)^{1/2}\right)^{s-1} \\
        &= \lim_{s\to\infty} A(s)^{1/2}  \Pi(s) \left(\left( I + \alpha(s) \frac{ \varphi(x_1) \varphi(x_1)^\top}{\norm{\varphi(x_1)}_2^2} \right) A(s) \right)^{-1} A(s)^{1/2} \\
        &= \lim_{s\to\infty} \Pi(s),
    \end{aligned}
    \end{equation}
    and
    \begin{equation}
    \begin{aligned}
        \lim_{s\to\infty} \mu(s) &= \lim_{s\to\infty} \frac{\left(A(s)^{1/2} \varphi(x_1)\right)^\top}{\norm{A(s)^{1/2} \varphi(x_1)}_2} \lim_{s\to\infty}  M_s(s) \lim_{s\to\infty} \frac{A(s)^{1/2} \varphi(x_1)}{\norm{A(s)^{1/2} \varphi(x_1)}_2} \\
        &= \frac{\varphi(x_1)^\top }{\norm{\varphi(x_1)}_2} \left( \lim_{s\to\infty} \Pi(s) \right) \frac{\varphi(x_1)}{\norm{\varphi(x_1)}_2} \\
        &= \frac{\varphi(x_1)^\top }{\norm{\varphi(x_1)}_2} e ^{-\frac{2 \Phi^\top \Phi}{n} t} \frac{\varphi(x_1)}{\norm{\varphi(x_1)}_2} \\
        &< \frac{\varphi(x_1)^\top }{\norm{\varphi(x_1)}_2} e ^{-\frac{2 \varphi(x_1) \varphi(x_1)^\top}{n} t} \frac{\varphi(x_1)}{\norm{\varphi(x_1)}_2} \\
        &= e^{- \frac{2 \norm{\varphi(x_1)}^2_2}{n} t},
    \end{aligned}
    \end{equation}
    where the third line follows from \eqref{eq:limPis} and Lemma \ref{lemma:fakeinequality} guarantees that, with probability at least $1 - 2 \exp \left(-c \log^2 n \right)$ over $X$ and $V$, the fourth line uniformly holds for all $t \in (0, \tau)$. We will condition on this event until the end of the proof.

    The previous limit implies that there exists $s^*$ such that, for all $s > s^*$, we have $\mu(s) < e^{- \frac{2 \norm{\varphi(x_1)}^2_2}{n} t}$. Thus, for such $s$, we also have
    \begin{equation}\label{eq:mualphasinequality}
        0< \mu(s)^{1/(s-1)} < \mu(s)^{1/s} \leq e^{- \frac{2 \norm{\varphi(x_1)}^2_2}{ns} t} = 1 + \alpha(s) < 1.
    \end{equation}
    
    Then, for such $s$, \eqref{eq:rhsgeometric} leads to
    \begin{equation}
    \begin{aligned}
        \sum_{l = 1}^{s-1} \pi_l(s) &\leq \frac{\norm{A(s)^{1/2} \varphi(x_1)}_2^2}{\norm{\varphi(x_1)}_2^2}  \sum_{l = 1}^{s-1} \mu(s)^{(l-1) / (s-1)} \\
        &= \frac{\norm{A(s)^{1/2} \varphi(x_1)}_2^2}{\norm{\varphi(x_1)}_2^2} \frac{1 - \mu(s)}{1 - \mu(s)^{1/(s-1)}} \\
        &< \frac{\norm{A(s)^{1/2} \varphi(x_1)}_2^2}{\norm{\varphi(x_1)}_2^2} \frac{1 - \mu(s)}{-\alpha(s)},
    \end{aligned}
    \end{equation}
    where we solve the geometric series in the second line and use \eqref{eq:mualphasinequality} in the third one.
    This gives
    \begin{equation}
    \begin{aligned}
        &\limsup_{s \to \infty} \left| \alpha(s) \sum_{l = 1}^{s-1} \pi_l(s) \right|  \\
        &\leq \lim_{s \to \infty} \left| \frac{\alpha(s) }{-\alpha(s)} \right| \lim_{s \to \infty} \frac{\norm{A(s)^{1/2} \varphi(x_1)}_2^2}{\norm{\varphi(x_1)}_2^2} \lim_{s \to \infty} \left| 1 - \mu(s) \right| \leq 1.
    \end{aligned}
    \end{equation}

    Plugging this last result in \eqref{eq:limsupappears}, we get
    \begin{equation}
        \left| \varphi(x_1)^\top e ^{-\frac{2 \Phi^\top \Phi}{n} t} \Phi_{-1}^+ Y_{-1}\right| \leq  2  \sup_{t' \in (0, t)} \left| \varphi(x_1)^\top e ^{-\frac{2 \Phi_{-1}^\top \Phi_{-1}}{n} t'} \Phi_{-1}^+ Y_{-1}\right|,
    \end{equation}
    which, taking the supremum of $t\in(0, \tau)$, and extending by continuity to $t = 0$, leads to the desired result.
\end{proof}

\simonesolved{For the next lemma, it is convenient to define the $\epsilon$-covering number of a separable set $T$ as follows
\begin{equation}
    \mathcal N (T, \epsilon) := \inf \left \{ \left| T_0 \right| \, \textup{  such that  }  \, T_0 \subseteq T, \; \textup{ and } \; T \subseteq \bigcup_{t_0 \in T_0} \bar S(t_0, \epsilon) \right \},
\end{equation}
where $\left| T_0 \right|$ denotes the cardinality of the set $T_0$, and $\bar S(t_0, \epsilon)$ denotes the closed Euclidean ball with center $t_0$ and radius $\epsilon$. We denote with $\textup{diam}(T) = \sup_{u, u' \in T} \norm{u - u'}_2$ the diameter of $T$.}

\begin{lemma}\label{lemma:Tisgood}
    Let Assumptions \ref{ass:data} and \ref{ass:activation} hold, and let $n = o \left( p / \log^4 p \right)$, $n \log^3 n =\bigO{d^{3/2}}$ and $n = \omega(d)$. Let $T \subseteq \mathbb R^d$ be the set described by the curve $\gamma: \R \to \R^d$ defined as
    \begin{equation}
        \gamma(t) = V^\top e ^{-\frac{2 \Phi_{-1}^\top \Phi_{-1}}{n} t} \Phi_{-1}^+ Y_{-1},
    \end{equation}
    for $t \in [0, \tau]$. Then, $T$ is separable with respect to the Euclidean norm and, denoting with $\mathcal N(T, \epsilon)$ its $\epsilon$-covering number and with $\textup{diam}(T)$ its diameter, we jointly have that
    \begin{equation}
        \int_0^\infty \sqrt{\log \mathcal N(T, \epsilon)} \diff \epsilon = \bigO{\log n}, \qquad \textup{diam}(T) = \bigO{1},
    \end{equation}
    with probability at least $1 - 2 \exp \left(-c \log^2 n \right)$ over $X_{-1}$ and $V$, where $c$ is an absolute constant.
\end{lemma}
\begin{proof}
    $T$ is described by a continuous function $\gamma$ applied to $t \in [0, \tau]$. Since the interval $[0, \tau]$ is separable, we also have that $T$ is separable.   
    Furthermore, note that $e ^{-\frac{2 \Phi_{-1}^\top \Phi_{-1}}{n} t} \Phi_{-1}^+ = \Phi_{-1}^+ e ^{-\frac{2 K_{-1}}{n} t}$. Then, for all $t, t' \in [0, \tau]$, we have
    \begin{equation}
        \norm{\gamma(t) - \gamma(t')}_2 = \norm{V^\top \Phi_{-1}^+ e ^{-\frac{2 K_{-1}}{n} t} \left(I - e ^{-\frac{2 K_{-1}}{n} (t'-t)}\right) Y_{-1}}_2.
    \end{equation}
    Assuming, without loss of generality, $t' \geq t$, we have that
    \begin{equation}
        \norm{\gamma(t) - \gamma(t')}_2 \leq \opnorm{V^\top \Phi_{-1}^+}  \norm{Y_{-1}}_2 = \bigO{1},
    \end{equation}
    where the last step is a consequence of Lemma \ref{lemma:facts} and it holds with probability at least $1 - 2 \exp \left(-c_1 \log^2 n \right)$ over $X_{-1}$ and $V$ (note that the argument of Lemma \ref{lemma:facts} goes through equivalently when considering $\Phi_{-1}$ instead of $\Phi$). This proves the desired result on the diameter.

    Following a similar strategy, we can prove that, for all $t'>t$ s.t.\ $\left| t - t' \right| \leq \Delta$, 
    \begin{equation}
    \begin{aligned}
        \norm{\gamma(t) - \gamma(t')}_2 &\leq \opnorm{V^\top \Phi_{-1}^+} \opnorm{I - e ^{-\frac{2 K_{-1}}{n} (t'-t)}} \norm{Y_{-1}}_2  \\
        &\leq \opnorm{V^\top \Phi_{-1}^+} \left(1 - e ^{-\frac{2 \evmax{K_{-1}}}{n} (t'-t)}\right) \norm{Y_{-1}}_2  \\
        &\leq \opnorm{V^\top \Phi_{-1}^+} \frac{2 \evmax{K_{-1}}}{n} (t'-t)\norm{Y_{-1}}_2  \\
        &\leq \opnorm{V^\top \Phi_{-1}^+} \norm{Y_{-1}}_2 \frac{2 \evmax{K_{-1}}}{n} \Delta,
    \end{aligned}
    \end{equation}
    where the second step is a consequence of the fact that $1 - e ^{-\frac{2 K_{-1}}{n} (t'-t)}$ is a p.s.d.\ matrix with eigenvalues given by $1 - e ^{-\frac{2 \lambda_i\left(K_{-1}\right)}{n} (t'-t)}$, and the third step follows from the inequality $z \geq 1 - e^{-z}$, for $z \geq 0$.
    Since $\norm{\varphi(x_i)}_2^2 = \bigO{p}$ jointly for all $i \in [n]$ \simonesolved{with probability at least $1 - 2 \exp\left( -c_2 p \right)$ over $V$ (see the argument carried out in \eqref{eq:concentrationfeaturevector1new} and \eqref{eq:concentrationfeaturevector2new}), we have that} $\evmax{K_{-1}} = \opnorm{\Phi_{-1}}^2 \leq \norm{\Phi_{-1}}_F^2 = \bigO{np}$. Thus, using again Lemma \ref{lemma:facts}, we get that, for every $\Delta > 0$ and for all $t'>t$ s.t.\ $\left| t - t' \right| \leq \Delta$,
    \begin{equation}
        \norm{\gamma(t) - \gamma(t')}_2 \leq C_1 p \Delta,
    \end{equation}
    with probability at least $1 - 2 \exp \left(-c_3 \log^2 n \right)$ over $X_{-1}$ and $V$, where $C_1$ is an absolute constant independent from the scalings of the problems and from $\Delta$.
    This means that we can cover $T$ using $\left\lceil  \tau / \Delta \right\rceil  = \left\lceil d \log^2 n / (p \Delta) \right\rceil $ balls with radius $C_1 p \Delta$.
    Then, setting $\epsilon = C_1 p \Delta$, this implies
    \begin{equation}
        \mathcal N(T, \epsilon) \leq \left\lceil \frac{d \log^2 n}{p} \frac{C_1 p}{\epsilon} \right\rceil   = \left\lceil \frac{C_1 d \log^2 n}{\epsilon} \right\rceil.
    \end{equation}
    Notice that, for $\epsilon \geq \textup{diam}(T)$, we simply have $\mathcal N(T, \epsilon) = 1$. For smaller values of $\epsilon \leq \textup{diam}(T) \leq C_2$, the previous expression reads $\mathcal N(T, \epsilon) \leq \frac{C_3 d \log^2 n}{\epsilon}$, where $C_3$ is another absolute constant set large enough such that $\frac{C_3 d \log^2 n}{\epsilon} \geq 2$, for all $\epsilon \leq C_2$ (this allows to remove the ceiling function), and $C_3 d \log^2 n \geq e C_2$ (which will be used in the next inequality). Then, we have
    \begin{equation}
    \begin{aligned}
        \int_0^\infty \sqrt{\log \mathcal N(T, \epsilon)} \diff \epsilon &\leq \int_0^{C_2} \sqrt{\log \mathcal N(T, \epsilon)} \diff \epsilon \\
        &\leq \int_0^{C_2} \sqrt{\log \left( \frac{C_3 d \log^2 n}{\epsilon} \right)} \diff \epsilon \\
        &\leq \int_0^{C_2} \sqrt{\log \left(\frac{e C_2}{\epsilon} \right) + \log \left(\frac{C_3 d \log^2 n}{e C_2} \right) } \diff \epsilon \\
        &\leq \int_0^{C_2} \sqrt{2 \log \left(\frac{e C_2}{\epsilon} \right) } \diff \epsilon + \int_0^{C_2} \sqrt{2 \log \left(\frac{C_3 d \log^2 n}{e C_2} \right) } \diff \epsilon \\
        &\leq \sqrt{2} \int_0^{C_2}  \log \left(\frac{e C_2}{\epsilon} \right) \diff \epsilon + C_2 \sqrt{2 \log \left(\frac{C_3 d \log^2 n}{e C_2} \right) } \\
        &= \sqrt{2}C_2 + C_4 \log n,
    \end{aligned}
    \end{equation}
    which concludes the argument.
\end{proof}

\paragraph{Proof of Lemma \ref{lemma:lastcliptraj}.}
    \simonesolved{Throughout the proof, we condition on the event occurring with probability at least $1 - 2 \exp \left( -c_1 \log^2 n\right)$ over $V$ and $X$ given by Lemma \ref{lemma:hyperK}. This guarantees that $K$ is invertible and, therefore, that $\Phi^+ = \Phi^\top K^{-1}$, which implies that the LHS of the statement reads $ y_i - \varphi(x_i)^\top \hat \theta(t) = \varphi(x_i)^\top  e^{- \frac{2 \Phi^\top \Phi}{n}  t  } \Phi^+ Y$.}
    
    Due to symmetry, we prove the statement for $i=1$, and the thesis will then be true for all $i \in [n]$ by performing a union bound on the $n$ events. A simple application of the triangle inequality gives
    \begin{equation}\label{eq:clippingfirsttriangle}
    \begin{aligned}
        & \sup_{t \in [0, \tau]}  \left| \varphi(x_1)^\top e^{-  \frac{2 \Phi^\top \Phi}{n}  t  } \Phi^+ Y \right| \\
        & \leq \sup_{t \in [0, \tau]} \left| \varphi(x_1)^\top e^{-  \frac{2 \Phi^\top \Phi}{n} t  } \left( \Phi^+ Y - \Phi^+_{-1} Y_{-1} \right)\right| + \sup_{t \in [0, \tau]}  \left| \varphi(x_1)^\top e^{-  \frac{2 \Phi^\top \Phi}{n}  t  } \Phi^+_{-1} Y_{-1} \right|.
    \end{aligned}
    \end{equation}
    The first term on the RHS can be bounded as follows
    \begin{equation}\label{eq:clippingfirsttriangle1}
        \sup_{t \in [0, \tau]} \left| \varphi(x_1)^\top e^{-  \frac{2 \Phi^\top \Phi}{n}  t  } \left( \Phi^+ Y - \Phi^+_{-1} Y_{-1} \right)\right| \leq  \norm{\varphi(x_1) }_2 \norm{\Phi^+ Y - \Phi^+_{-1} Y_{-1}}_2  = \bigO{\log n},
    \end{equation}
    where the last step is a consequence of Lemma \ref{lemma:stability} and $\norm{\varphi(x_1)}_2 = \bigO{\sqrt p}$ \simonesolved{(see \eqref{eq:concentrationfeaturevector1new} and \eqref{eq:concentrationfeaturevector2new})}, and it holds with probability at least $1 - 2 \exp \left( -c_2 \log^2 n\right)$ over $V$ and $X$.
    To bound the second term on the RHS of \eqref{eq:clippingfirsttriangle}, we first consider the term
    \begin{equation}\label{eq:suptwoterms}
    \begin{aligned}
        \sup_{t \in [0, \tau]} \left| \varphi(x_1)^\top e ^{-\frac{2 \Phi_{-1}^\top \Phi_{-1}}{n} t} \Phi_{-1}^+ Y_{-1}\right|  \leq \sup_{t \in [0, \tau]} \mu_1 \left| x_1^\top \gamma(t) \right| + \sup_{t \in [0, \tau]} \left| \tilde \varphi(x_1)^\top \Phi_{-1}^+ e ^{-\frac{2 K_{-1}}{n} t} Y_{-1}\right|,
    \end{aligned} 
    \end{equation}
    where we define the shorthand $\gamma(t) = V^\top e ^{-\frac{2 \Phi_{-1}^\top \Phi_{-1}}{n} t} \Phi_{-1}^+ Y_{-1}$. Using \eqref{eq:firstpartthesisoldbadlemma}, the second term 
    reads
    \begin{equation}\label{eq:trisup1}
        \sup_{t \in [0, \tau]} \left| \tilde \varphi(x_1)^\top \Phi_{-1}^+ e ^{-\frac{2 K_{-1}}{n} t} Y_{-1}\right| \leq \norm{\left(\Phi_{-1}^+\right)^\top \tilde \varphi(x_1)}_2^2 \norm{Y_{-1}}_2^2 = \bigO{\log n},
    \end{equation}
    with probability at least $1 - 2 \exp \left(-c_3 \log^2 n \right)$ over $X$ and $V$. For the first term of \eqref{eq:suptwoterms}, we apply the following decomposition
    \begin{equation}\label{eq:suptwotermsnew}
        \sup_{t \in [0, \tau]} \left| x_1^\top \gamma(t) \right| \leq \sup_{t \in [0, \tau]} \left| \left(x_1 -  \E_{X} \left[ x_1 \right] \right)^\top \gamma(t) \right|+ \sup_{t \in [0, \tau]}  \left| \E_{X} \left[ x_1 \right]^\top \gamma(t) \right|.
    \end{equation}
    The second term can be easily bounded as follows
    \begin{equation}\label{eq:suptwotermsnew1}
        \sup_{t \in [0, \tau]}  \left| \E_{X} \left[ x_1 \right]^\top \gamma(t) \right| \leq \norm{\E_{X} \left[ x_1 \right]}_2 \sup_{t \in [0, \tau]} \norm{\gamma(t)}_2 \leq \norm{\E_{X} \left[ x_1 \right]}_2 \opnorm{V^\top \Phi_{-1}^+}  \norm{Y_{-1}}_2 = \bigO{1},
    \end{equation}
    where the last step is a consequence of Lemmas \ref{lemma:normofx} and \ref{lemma:facts}, and it holds with probability at least $1 - 2 \exp \left(-c_4 \log^2 n \right)$ over $X_{-1}$ and $V$. 
    For the first term of \eqref{eq:suptwotermsnew}, note that the stochastic process $\left(x_1 -  \E_{X} \left[ x_1 \right] \right)^\top \gamma(t)$ is separable, mean-0, and sub-Gaussian, \emph{i.e.},
    \begin{equation}
        \subGnorm{\left(x_1 - \E[x_1]\right)^\top (u- u')} \leq \subGnorm{x^\top (u- u')} + \norm{\E[x]}_2 \norm{u - u'}_2 \leq K \norm{u - u'}_2,
    \end{equation}
    for all $u, u' \in T$ (where we use Lemma \ref{lemma:normofx} and the fact that $x$ is sub-Gaussian by Assumption \ref{ass:data}). Furthermore, by Lemma \ref{lemma:Tisgood}, with probability at least $1 - 2 \exp \left(-c_5 \log^2 n \right)$ over $X_{-1}$ and $V$, we have that, for $t \in [0, \tau]$, $\gamma(t)$ describes a separable set $T \subset \R^d$ such that $\int_0^\infty \sqrt{\log \mathcal N(T, \epsilon)} \, \diff \epsilon = \bigO{\log n}$ and $\textup{diam}(T) = \bigO{1}$. Then, conditioning on such event, 
    we can use Dudley's (or chaining tail) inequality (see Theorem 5.29 in \cite{vanhandel}), which gives
    \begin{equation}\label{eq:trisup2}
    \begin{split}
        \sup_{t \in [0, \tau]}& \left|   \left(x_1 -  \E_{X} \left[ x_1 \right] \right)^\top \gamma(t) -  \left(x_1 -  \E_{X} \left[ x_1 \right] \right)^\top \gamma(0) \right| \\
        &\leq C_1 \left( \int_0^\infty \sqrt{\log \mathcal N(T, \epsilon)} \, \diff \epsilon +  \textup{diam}(T)  \log n \right) = \bigO{\log n},
    \end{split}
    \end{equation}
    with probability at least $1 - 2 \exp \left( -c_6 \log^2 n \right)$ over $x_1$.
    Furthermore, we have that
    \begin{equation}\label{eq:trisup3}
    \begin{split}
       & \left| \left(x_1 -  \E_{X} \left[ x_1 \right] \right)^\top \gamma(0) \right| \leq \left| x_1^\top \gamma(0) \right|  + \left|  \E_{X} \left[ x_1 \right]^\top \gamma(0) \right|  \\
        &\hspace{2em}\leq C_2 \log n \norm{\gamma(0)}_2 + \norm{\E_X \left[ x_1 \right]}_2  \norm{\gamma(0)}_2 \leq C_3 \norm{\gamma(0)}_2 \log n = \bigO{\log n},
    \end{split}
    \end{equation}
    where the second step holds with probability $1 - 2 \exp \left( -c_7 \log^2 n \right)$ over $x_1$, since it is sub-Gaussian, the third step follows from Lemma \ref{lemma:normofx}, and the last step follows from the same argument used in \eqref{eq:suptwotermsnew1}, due to Lemma \ref{lemma:facts}, and holds with probability $1 - 2 \exp \left( -c_8 \log^2 n \right)$ over $X_{-1}$ and $V$.

    Plugging \eqref{eq:trisup2}, \eqref{eq:trisup3}, and \eqref{eq:suptwotermsnew1} in \eqref{eq:suptwotermsnew} gives $\sup_{t \in [0, \tau]} \left| x_1^\top \gamma(t) \right| = \bigO{\log n}$ with probability at least $1 - 2 \exp \left( -c_9 \log^2 n \right)$ over $X$ and $V$. Then, this last result, together with \eqref{eq:trisup1} and \eqref{eq:suptwoterms}, leads to \begin{equation}\label{eq:clippingfirsttriangle2}
        \sup_{t \in [0, \tau]} \left| \varphi(x_1)^\top e^{-  \frac{2 \Phi^\top \Phi}{n}  t  } \Phi^+_{-1} Y_{-1} \right| \leq 2 \sup_{t \in [0, \tau]} \left| \varphi(x_1)^\top e ^{-\frac{2 \Phi_{-1}^\top \Phi_{-1}}{n} t} \Phi_{-1}^+ Y_{-1}\right| = \bigO{\log n},
    \end{equation}
    with probability at least $1 - 2 \exp \left( -c_{10} \log^2 n \right)$, where the first step holds because of Lemma \ref{lemma:usingLie}.


    
    Plugging \eqref{eq:clippingfirsttriangle1} and \eqref{eq:clippingfirsttriangle2} in \eqref{eq:clippingfirsttriangle}, 
    and performing a union bound over all the indices $i \in [n]$, leads to the desired result. \simonesolved{We finally remark that this result does not require the condition $\log n = \Theta(\log p)$ in Assumption \ref{ass:scalings}.}
    \qed


\begin{lemma}\label{lemma:lessthanB}
For all $i\in[n]$, we have that
  \begin{equation}
      \E_B \left[ \left(\varphi(x_i)^\top \tilde \Theta(t) - \varphi(x_i)^\top \tilde \Theta(s)\right)^2 \right] \leq \Sigma^2 \left| t - s \right| \norm{\varphi(x_i)}_2^2.
  \end{equation}  
\end{lemma}
\begin{proof}
Let us define the matrix $O \in \R^{p \times p}$ as the orthogonal matrix such that $\Phi^\top \Phi = O D O^\top$, with $D \in \R^{p \times p}$ being a p.s.d.\ diagonal matrix with eigenvalues sorted in non-increasing way (which we indicate as $D_k$, $k \in [p]$, for simplicity). Then, we have
\begin{equation}\label{eq:justputtingOOtop}
    \E_B \left[ \left(\varphi(x_i)^\top \tilde \Theta(t) - \varphi(x_i)^\top \tilde \Theta(s) \right)^2 \right] = \E_B \left[ \left(\varphi(x_i)^\top OO^\top \tilde \Theta(t) - \varphi(x_i)^\top  OO^\top \tilde \Theta(s) \right)^2 \right].
\end{equation}

Note that, multiplying both sides of \eqref{eq:OU1} by $O^\top$ we get
\begin{equation}\label{eq:OUD}
    \diff \left( O^\top \hat \Theta(t) \right) = - \frac{2D}{n}  \left( O^\top \hat \Theta(t) \right) \diff t +  \frac{2 O^\top \Phi^\top Y}{n}  \diff t + \Sigma \, \diff \left(O^\top B(t)\right).
\end{equation}
Since $D$ is diagonal, we have that each coordinate of $O^\top \hat \Theta(t)$ evolves as an independent, one dimensional, OU process (without drift, eventually), since each coordinate of $O^\top B(t)$ evolves as an independent Wiener processes, for rotational invariance of the Gaussian measure.
This implies that, for all $s, t \in [0, \tau]$, and for all $k \neq k'$, we have
\begin{equation}
    \E_B \left[  [O^\top \tilde \Theta(t)]_k [O^\top \tilde \Theta(s)]_{k'}\right] = 
    0.
\end{equation}
This, together with \eqref{eq:justputtingOOtop}, gives
\begin{equation}\label{eq:varz1}
\begin{aligned}
    \E_B \left[ \left(\varphi(x_i)^\top \tilde \Theta(t) - \varphi(x_i)^\top \tilde \Theta(s) \right)^2 \right] &= \E_B \left[ \left(\sum_{k=1}^p [O^\top \varphi(x_i)]_k [O^\top \tilde \Theta(t)]_k - [O^\top \varphi(x_i)]_k [O^\top \tilde \Theta(s)]_k \right)^2 \right] \\
    &=  \sum_{k=1}^p \E_B \left[ \left( [O^\top \varphi(x_i)]_k [O^\top \tilde \Theta(t)]_k - [O^\top \varphi(x_i)]_k [O^\top \tilde \Theta(s)]_k \right)^2 \right] \\
    &=  \sum_{k=1}^p [O^\top \varphi(x_i)]_k^2 \, \E_B \left[ \left( [O^\top \tilde \Theta(t)]_k - [O^\top \tilde \Theta(s)]_k \right)^2 \right].
\end{aligned}
\end{equation}

Since $[O^\top \hat \Theta(t)]_k$ describes an OU process with drift $\Delta_k = \frac{2D_k}{n}$, standard results on the variance and the auto-correlation of OU processes (see Section 4.4.4 in \cite{gardiner1985handbook}) give (supposing $\Delta_k \neq 0$)
\begin{equation}
    \E_B \left[  [O^\top \tilde \Theta(t)]_k^2 \right] = \frac{\Sigma^2}{2 \Delta_k} \left( 1 - e^{-2 \Delta_k t}\right),
\end{equation}
and
\begin{equation}
    \E_B \left[  [O^\top \tilde \Theta(t)]_k  [O^\top \tilde \Theta(s)]_k \right] = \frac{\Sigma^2}{2 \Delta_k}  \left( e^{-\Delta_k | t - s |} - e^{-\Delta_k (t + s )}\right),
\end{equation}
which implies that
\begin{equation}
    \E_B \left[\left( [O^\top \tilde \Theta(t)]_k - [O^\top \tilde \Theta(s)]_k \right)^2 \right] = \frac{\Sigma^2}{\Delta_k} \left( 1 - \frac{e^{-2 \Delta_k t} + e^{-2 \Delta_k s}}{2} - e^{- \Delta_k | t - s |} + e^{- \Delta_k ( t + s )} \right).
\end{equation}
Then, assuming $t \geq s$ without loss of generality, and introducing the shorthand $\delta = t - s$, we have
\begin{equation}\label{eq:varz2}
\begin{aligned}
    \E_B \left[\left( [O^\top \tilde \Theta(t)]_k - [O^\top \tilde \Theta(s)]_k \right)^2 \right] &= \frac{\Sigma^2}{\Delta_k} \left( 1 - e^{-2 \Delta_k s} \left( \frac{1 + e^{-2 \Delta_k \delta} - 2 e^{- \Delta_k \delta }}{2} \right) - e^{- \Delta_k \delta} \right) \\
    &= \frac{\Sigma^2 \left( 1  - e^{- \Delta_k \delta} \right)}{\Delta_k} \left( 1 - e^{-2 \Delta_k s} \left(  \frac{1  - e^{- \Delta_k \delta} }{2} \right)\right) \\
    &\leq \frac{\Sigma^2 \left( 1  - e^{- \Delta_k \delta} \right)}{\Delta_k} \\
    &\leq \Sigma^2 \delta = \Sigma^2 |t - s|,
\end{aligned}
\end{equation}
where the inequality in the fourth line holds as $1 - e^{-x} \leq x$ for all $x \geq 0$. Note that, when $[O^\top \hat \Theta(t)]_k$ does not have a drift term ($\Delta_k = 0$), it is a Wiener process, which implies that $\E_B \left[\left( [O^\top \tilde \Theta(t)]_k - [O^\top \tilde \Theta(s)]_k \right)^2 \right] = \Sigma^2 |t - s|$ (see Section 3.8.1 in \cite{gardiner1985handbook}).

Thus, plugging \eqref{eq:varz2} in \eqref{eq:varz1} we get
\begin{equation}
\begin{aligned}
    \E_B \left[ \left(\varphi(x_i)^\top \tilde \Theta(t) - \varphi(x_i)^\top \tilde \Theta(s) \right)^2 \right] &=   \sum_{k=1}^p [O^\top \varphi(x_i)]_k^2 \, \E_B \left[ \left( [O^\top \tilde \Theta(t)]_k - [O^\top \tilde \Theta(s)]_k \right)^2 \right] \\
    &\leq \sum_{k=1}^p [O^\top \varphi(x_i)]_k^2 \,\Sigma^2 \left| t - s \right| \\
    &=  \Sigma^2 \left| t - s \right| \norm{\varphi(x_i)}_2^2,
\end{aligned}
\end{equation}
which gives the desired result.
\end{proof}

\paragraph{Proof of Lemma \ref{lemma:lastclipnoise}.}
Consider the auxiliary Wiener process $z_i(t) \in \R$ such that (see Section 3.8.1 of \cite{gardiner1985handbook})
\begin{equation}\label{eq:variancewiener}
    \E_{z_i} \left[ \left(z_i(t) - z_i(s)\right)^2 \right] = \Sigma^2 \left| t - s \right| \norm{\varphi(x_i)}_2^2.
\end{equation}
By the reflection principle in standard Wiener processes (see Proposition 3.7 in \cite{revuzyor}), we have
\begin{equation}\label{eq:Pbhata}
    \P_{z_i} \left( \sup_{t \in [0, \tau] } z_i(t) \geq a \right) = 2 \, \P_{z_i} \left( z_i(\tau) \geq a \right) \leq 4 \exp \left( - \frac{a^2}{\Sigma^2 \tau \norm{\varphi(x_i)}_2^2} \right).
\end{equation}

\simonesolved{By the same argument carried out in \eqref{eq:concentrationfeaturevector1new} and \eqref{eq:concentrationfeaturevector2new}, we have that $\norm{\varphi(x_i)}_2^2 = \bigO{p}$ with probability at least $1 - 2 \exp\left( -c_1 p \right)$ over $V$.
Thus, there exists a constant $C$ for which we can write}
\begin{equation}\label{eq:sigma2tau}
    \Sigma^2 \tau \norm{\varphi(x_i)}_2^2 \leq 4 C \frac{d \log^6 n}{n^2} \, \frac{8 \log(1 / \delta)}{\varepsilon^2} \, \frac{d \log^2 n}{p} \, p =  32 C \frac{d^2 \log^8 n}{n^2} \, \frac{\log(1 / \delta)}{\varepsilon^2} = \bigO{1},
\end{equation}
where the last step is a consequence of Assumption \ref{ass:privacy}. Then, \eqref{eq:Pbhata} reads $\P_{z_i} \left( \sup_{t \in [0, \tau] } z_i(t) \geq a \right) \leq 4 \exp \left( - c_2 a^2 \right)$, which implies
\begin{equation}\label{eq:BhatlessC1}
    \E_{z_i} \left[ \sup_{t \in [0, \tau] } \left| z_i(t) \right| \right] = \int_{0}^{+\infty} \P_{z_i} \left( \sup_{t \in [0, \tau] } \left|  z_i(t) \right| \geq a \right) \diff a \leq \int_{0}^{+\infty} 2 \P_{z_i} \left( \sup_{t \in [0, \tau] } z_i(t) \geq a \right) \diff a \leq C_1,
\end{equation}
where $C_1$ is an absolute constant, and the statement holds with probability at least $1 - 2 \exp\left( -c_1 p \right)$ over $V$.
Note that, by Lemma \ref{lemma:lessthanB}, for all $s, t \in [0, \tau]$, we have
\begin{equation}\label{eq:inequalitymomentauxiliary}
    \E_B \left[ \sup_{t \in [0, \tau] } \left| \varphi(x_i)^\top \tilde \Theta(t) \right| \right] \leq \Sigma^2 \left| t - s \right| \norm{\varphi(x_i)}_2^2 = \E_{z_i} \left[ \left(z_i(t) - z_i(s)\right)^2 \right].
\end{equation}
Then, since $\varphi(x_i)^\top \tilde \Theta(t)$ and $z_i(t)$ are two mean-0 Gaussian processes, due to Sudakov-Fernique inequality (see Theorem 7.2.11 of \cite{vershynin2018high}), we have that, for all finite subsets $T_0 \subset [0 , \tau]$,
\begin{equation}\label{eq:sudferfinite}
    \E_B \left[ \sup_{t \in T_0 } \left| \varphi(x_i)^\top \tilde \Theta(t) \right| \right] \leq \E_{z_i} \left[ \sup_{t \in T_0} \left| z_i(t) \right| \right].
\end{equation}
Differently from Theorem 7.2.11 of \cite{vershynin2018high}, we consider the supremum of the absolute value of the processes. The result can in fact be extended to this case by simply looking at both $\varphi(x_i)^\top \tilde \Theta(t)$ and $- \varphi(x_i)^\top \tilde \Theta(t)$. To extend to the previous equation to the full interval $[0, \tau]$, we have that the Kolmogorov continuity theorem (see Theorem 1.8 in \cite{revuzyor}) 
guarantees that $\varphi(x_i)^\top \tilde \Theta(t)$ is almost surely 
continuous (as its $\alpha$-moment is bounded by the $\alpha$-moment of $z_i(t)$ via \eqref{eq:inequalitymomentauxiliary} since they are both Gaussian processes). 
Then, in the compact time interval $[0, \tau]$, we have uniform continuity, \emph{i.e.}, for every $h > 0$, there exists $r$, such that, for every $t \in [0, \tau]$, $\left| \varphi(x_i)^\top \tilde \Theta(t + r) - \varphi(x_i)^\top \tilde \Theta(t) \right| < h$, with probability 1 over $B$. Then, \eqref{eq:sudferfinite} can be extended by continuity to
\begin{equation}\label{eq:sudferinfinite}
    \E_B \left[ \sup_{t \in [0, \tau] } \left| \varphi(x_i)^\top \tilde \Theta(t) \right| \right] \leq \E_{z_i} \left[ \sup_{t \in [0, \tau]} \left| z_i(t) \right| \right] \leq C_1,
\end{equation}
where the last step is a consequence of \eqref{eq:BhatlessC1}.

Note that, by Lemma \ref{lemma:lessthanB}, 
since we have $\varphi(x_i)^\top \tilde \Theta(0) = 0$,
\begin{equation}\label{eq:btis2}
    \sigma^2_\tau := \sup_{t \in [0, \tau]} \E_{B} \left[ \left(\varphi(x_i)^\top \tilde \Theta(t)\right)^2 \right] \leq \Sigma^2  \tau  \norm{\varphi(x_i)}_2^2 = \bigO{1},
\end{equation}
where in the last step we used \eqref{eq:sigma2tau}. 
Then, the Borell-TIS inequality (see Theorem 2.1.1 of \cite{adler2007random}) allows us to write, for all $a > 0$,
\begin{equation}\label{eq:borellTis}
    \P_B \left( \sup_{t \in [0, \tau] } \left| \varphi(x_i)^\top \tilde \Theta(t) \right| > \E_B \left[ \sup_{t \in [0, \tau] } \left|  \varphi(x_i)^\top \tilde \Theta(t) \right| \right] + a \right) < \exp \left( - \frac{a^2}{2 \sigma^2_\tau} \right) \leq 2 \exp \left( - c_3 a^2 \right).
\end{equation}
Thus, setting $a = \log n$, we get
\begin{equation}
    \sup_{t \in [0, \tau] } \left|  \varphi(x_i)^\top \tilde \Theta(t) \right| \leq \E_B \left[ \sup_{t \in [0, \tau] } \left| \varphi(x_i)^\top \tilde \Theta(t) \right| \right] + \log n \leq C_1 + \log n = \bigO{\log n},
\end{equation}
with probability at least $1 - 2 \exp\left(-c_3 \log ^2 n\right) - 2 \exp\left(-c_1 p\right)$ over $V$ and $B$, where we use \eqref{eq:sudferinfinite} in the second step.
Our argument holds for any choice of $i$. Then, the thesis holds uniformly on every $i \in [n]$ with probability at least $1 - 2 \exp\left(-c_4 \log ^2 n \right) - 2 n \exp\left(-c_4 p\right)$ over $B$ and $V$, which provides the desired result. \qed

\subsection{Analysis of noise and early stopping}\label{sec:proof34}


This section contains the proofs of Lemmas \ref{lemma:noise} and \ref{lemma:earlystopping}, as well as of our main Theorem \ref{thm:mainthm}, including a number of auxiliary results required by the argument.
We make use of the notation introduced in the previous sections and consider the hyper-parameter choice in \eqref{eq:tauSigmaapp}.

\paragraph{Proof of Lemma \ref{lemma:noise}.}
    We denote with $A \in \R^{p \times p}$ the p.s.d.\ matrix such that $A^2 = A^\top A = \E_{x} \left[ \varphi(x) \varphi(x)^\top \right]$ (the definition is well posed as $\E_{x} \left[ \varphi(x) \varphi(x)^\top \right]$ is p.s.d.). Thus,
    \begin{equation}
        \norm{A}_F^2 = \tr \left( A^\top A \right) = \tr \left( \E_{x} \left[ \varphi(x) \varphi(x)^\top \right] \right) = \E_{x} \left[ \tr \left( \varphi(x) \varphi(x)^\top \right)  \right] = \E_{x} \left[ \norm{ \varphi(x) }_2^2 \right],
    \end{equation}
    where we use the linearity of the trace in the third step, and its cyclic property in the fourth step. \simonesolved{Furthermore, we can write
    \begin{equation}\label{eq:FnormA}
        \norm{A}_F^2 = \E_{x} \left[ \norm{ \varphi(x) }_2^2 \right] \leq 2 \norm{ \phi(\mathbf 0) }_2^2  + 2 L^2 \E_{x} \left[ \norm{ Vx }_2^2 \right] \leq 2 \norm{ \phi(\mathbf 0) }_2^2  + 2 L^2 \opnorm{V}^2 \E_{x} \left[ \norm{x }_2^2 \right] = \bigO{p},
    \end{equation}
    where we use that $\phi$ is Lipschitz in the second step, we denote with $\mathbf 0 \in \R^p$ a vector of zeros, and the last step follows from the bound on $\opnorm{V}$ given by Lemma \ref{lemma:evminX}, which holds with probability at least $1 - 2 \exp \left( -c_1 d \right)$ over $V$ (high probability event over which we will condition until the end of the proof).}
    The introduction of $A$ allows us to rewrite the LHS of \eqref{eq:res3} as $\E_{x} \left[ \left( \varphi(x)^\top \tilde \Theta(\tau) \right)^2 \right] = \tilde \Theta(\tau)^\top \E_{x} \left[ \varphi(x) \varphi(x)^\top \right] \tilde \Theta(\tau)  = \norm{A \tilde \Theta(\tau)}_2^2$, which gives
    \begin{equation}\label{eq:rewritethesisA1}
        \E_{x} \left[ \left( \varphi(x)^\top \tilde \Theta(\tau) \right)^2 \right] = \norm{A \tilde \Theta(\tau)}_2^2 = \norm{A O O^\top \tilde \Theta(\tau)}_2^2 = \norm{A O \tilde \Delta \rho}_2^2,
    \end{equation}
    where $\tilde \Delta \in \R^{p \times p}$ is a p.s.d.\ diagonal matrix s.t.\ $\tilde \Delta^2$ contains $\E_B \left[\left[O^\top \tilde \Theta(\tau)\right]_k^2\right]$ in its $k$-th entry ($O$ is the same matrix considered in Lemma \ref{lemma:lessthanB}), and
    \begin{equation}
        \rho := \tilde \Delta^{-1} O^\top \tilde \Theta(\tau) \in \R^p
    \end{equation}
    is a standard Gaussian vector in the probability space of $B$ (since each coordinate of $O^\top \hat \Theta(t)$ evolves as an independent, one dimensional, OU process, by the argument following \eqref{eq:OUD}).
    Then, by Theorem 6.3.2 in \cite{vershynin2018high}, we have
    \begin{equation}
        \subGnorm{\norm{A O \tilde \Delta \rho}_2 - \norm{A O \tilde \Delta}_F} \leq C_1 \opnorm{A O \tilde \Delta} \leq C_1 \norm{A O \tilde \Delta}_F,
    \end{equation}
    where the sub-Gaussian norm is intended in the probability space of $\rho$. This can be turned into the tail inequality (see Proposition 2.5.2 in \cite{vershynin2018high})
    \begin{equation}\label{eq:normAtildetheta}
        \P_B \left( \norm{A O \tilde \Delta \rho}_2 \geq \norm{A O \tilde \Delta}_F + \sqrt{p \Sigma^2 \tau} \log n \right) \leq 2 \exp \left(-c_2 \frac{p \Sigma^2 \tau \log^2 n}{\norm{A O \tilde \Delta}_F^2} \right).
    \end{equation}

    Since we have $\norm{A O \tilde \Delta}_F^2 \leq \norm{A}_F^2 \opnorm{O \tilde \Delta}^2 \leq \norm{A}_F^2 \opnorm{\tilde \Delta}^2 \leq C_2 p \Sigma^2 \tau$, 
    where the last step follows from \eqref{eq:FnormA} ($C_2$ is an absolute constant) and the argument in \eqref{eq:varz2}, 
    \eqref{eq:normAtildetheta} gives
    \begin{equation}
        \norm{A O \tilde \Delta \rho}_2^2 = \bigO{p \Sigma^2 \tau \log^2 n} = \bigO{ \frac{d^2 \log^{10} n}{n^2} \, \frac{\log(1 / \delta)}{\varepsilon^2}} = \tilde{\mathcal O} \left( \frac{d^2}{\varepsilon^2 n^2} \right) = o(1),
    \end{equation}
    with probability at least $1 - 2 \exp \left(-c_2 \log^2 n \right)$ over $B$, where the last step is a consequence of Assumption \ref{ass:privacy}.
    Plugging the last equation in \eqref{eq:rewritethesisA1} provides the desired result. \qed

\paragraph{Proof of Lemma \ref{lemma:earlystopping}.}
\simonesolved{Consider the projector $P_\Lambda$  on the space spanned by the eigenvectors associated with the $d$ largest eigenvalues of $\Phi^\top \Phi$.} 
Then, let us decompose the LHS of \eqref{eq:res4} as
\begin{equation}\label{eq:earlystop}
\begin{aligned}
    & \E_{x} \left[ \left( \varphi(x)^\top \left( \hat \theta(\tau) - \theta^* \right)\right)^2 \right]  \\
    & \leq 2 \E_{x} \left[ \left( \varphi(x)^\top P_\Lambda \left( \hat \theta(\tau) - \theta^* \right)\right)^2 \right] + 2 \E_{x} \left[ \left( \varphi(x)^\top P_\Lambda^\perp \left( \hat \theta(\tau) - \theta^* \right)\right)^2 \right] \\
    & \leq 2 \E_{x} \left[ \left( \varphi(x)^\top P_\Lambda \left( \hat \theta(\tau) - \theta^* \right)\right)^2 \right] + 4 \E_{x} \left[ \left( \tilde \varphi(x)^\top P_\Lambda^\perp \left( \hat \theta(\tau) - \theta^* \right)\right)^2 \right] \\
    & \qquad + 4 \mu_1^2 \E_{x} \left[ \left( x^\top V^\top  P_\Lambda^\perp \left( \hat \theta(\tau) - \theta^* \right)\right)^2 \right].
\end{aligned}
\end{equation}
We now bound the last three terms separately. \simonesolved{As usual, we condition on the high probability event given by Lemma \ref{lemma:hyperK}, which guarantees that $K$ is invertible and, therefore,  $\hat \theta(t) = ( 1 - e^{- \frac{2 \Phi^\top \Phi}{n}  t  } ) \Phi^+ Y$, with probability at least $1 - 2 \exp \left( -c_2 \log^2 n\right)$ over $V$ and $X$, where $\Phi^+ = \Phi^\top K^{-1}$.}
For the first term of \eqref{eq:earlystop}, exploiting the fact that $P_\Lambda$ is a projector over an eigenspace of $\Phi^\top \Phi$, we can write
\begin{equation}\label{eq:rp1}
    P_\Lambda \left( \theta^*-\hat \theta(\tau)   \right) = P_\Lambda e^{-  \frac{2 \Phi^\top \Phi }{n} \tau } \Phi ^+ Y = P_\Lambda e^{ - P_\Lambda \frac{2 \Phi^\top \Phi }{n} P_\Lambda  \tau } P_\Lambda \Phi ^+ Y,
\end{equation}
where $\opnorm{P_\Lambda e^{ - P_\Lambda \frac{2 \Phi^\top \Phi }{n} P_\Lambda  \tau } P_\Lambda} = e^{ - \frac{2 \lambda_d(K) }{n}  \tau } = \bigO{ e^{- c_3 \log^2 n }}$, where $c_3$ is a small enough absolute constant, and the last step is a consequence of Lemma \ref{lemma:hyperK} and $\tau = d \log^2 n / p$, and it holds with probability at least $1 - 2 \exp \left( -c_4 \log^2 n \right)$ over $V$ and $X$. Then, using $n  e^{- c_3 \log^2 n } = \bigO{1}$ we have
\begin{equation}\label{eq:fordiscussionlaterrevision}
    \norm{P_\Lambda \left( \hat \theta(\tau) - \theta^* \right)}_2 \leq \opnorm{P_\Lambda e^{ - P_\Lambda \frac{2 \Phi^\top \Phi }{n} P_\Lambda  \tau } P_\Lambda} \opnorm{\Phi ^+} \norm{Y}_2 = \bigO{\frac{1}{\sqrt{pn}}},
\end{equation}
where the second step holds because of Lemma \ref{lemma:hyperK}. Then, \simonesolved{using \eqref{eq:FnormA}}, the first term in \eqref{eq:earlystop} reads
\begin{equation}\label{eq:earlystop1}
    \E_{x} \left[ \left( \varphi(x)^\top P_\Lambda \left( \hat \theta(\tau) - \theta^* \right)\right)^2 \right] \leq \E_{x} \left[ \norm{\varphi(x)}_2^2 \right] \norm{P_\Lambda \left( \hat \theta(\tau) - \theta^* \right)}_2^2 = \bigO{p \, \frac{1}{pn}} = \bigO{\frac{1}{n}},
\end{equation}
with probability at least $1 - 2 \exp \left( -c_5 \log^2 n \right)$ over $V$ and $X$.

For the second term of \eqref{eq:earlystop}, we have
\begin{equation}\label{eq:earlystop2}
\begin{aligned}
    \E_{x} \left[ \left( \tilde \varphi(x)^\top P_\Lambda^\perp \left( \hat \theta(\tau) - \theta^* \right)\right)^2 \right] 
    &\leq \opnorm{\E_{x} \left[ \tilde \varphi(x) \tilde \varphi(x)^\top \right]} 
    \opnorm{\Phi ^+}^2 \norm{Y}_2^2 \\
    &= \bigO{\left( \log^4 n + \frac{p \log^3 d}{d^{3/2}} \right) \, 1 \, \frac{1}{p} \, n} \\
    &= \bigO{\frac{n}{ \sqrt p} \, \frac{\log^4 n}{\sqrt p} + \frac{n \log^3 d}{d^{3/2}} } \\
    &= \bigO{\frac{d}{n} + \frac{n \log^3 d}{d^{3/2}}} = o(1).
\end{aligned}
\end{equation}
Here, the second line follows from Lemmas \ref{lemma:Eopnormsmall} and \ref{lemma:hyperK}, and it holds with probability at least $1 - 2 \exp \left( -c_6 \log^2 n \right)$ over $X$ and $V$. 

For the third term of \eqref{eq:earlystop}, since $x$ is distributed according to $\mathcal P_X$, it is sub-Gaussian and $\subGnorm{x} = \bigO{1}$. Then, we can bound its second moment (see \cite{vershynin2018high}, Proposition 2.5.2) as follows
\begin{equation}\label{eq:earlystop3}
\begin{aligned}
    \E_{x} \left[ \left( x^\top V^\top  P_\Lambda^\perp \left( \hat \theta(\tau) - \theta^* \right)\right)^2 \right] &\leq C_1 \norm{V^\top  P_\Lambda^\perp \left( \hat \theta(\tau) - \theta^* \right)}_2^2 \\
    &= C_1 \norm{V^\top  P_\Lambda^\perp \Phi^+ e^{- \frac{2K}{n} \tau } Y}_2^2 \\
    &\leq C_1 \opnorm{V^\top  P_\Lambda^\perp \Phi^+}^2  \norm{Y}_2^2 \\
    & = \bigO{\frac{d}{n} +  \frac{n^2 \log^6 d}{d^3}} \\
    &= \bigO{\frac{d}{n} + \frac{n \log^3 d}{d^{3/2}}} = o(1),
\end{aligned}
\end{equation}
where $C_1$ is an absolute constant and the fourth line holds with probability $1 - 2 \exp \left( -c_7 \log^2 n \right)$ over $X$ and $V$, \simonesolved{because of Lemma \ref{lemma:earlystoppingtermnew}}. Plugging \eqref{eq:earlystop1}, \eqref{eq:earlystop2}, and \eqref{eq:earlystop3} in \eqref{eq:earlystop} provides the desired result. \qed

\simonenew{We would like to remark that a key aspect of the previous argument is the existence of a \emph{spectral gap} between $\lambda_d(K)$ and $\lambda_{d+1}(K)$. This result, proved in Lemma \ref{lemma:hyperK}, critically uses that $\evmin{X^\top X} = \Omega(n)$, which in turn relies on our assumption $\evmin{\E \left[ x^\top x \right]} = \Omega(1)$, \emph{i.e.}, that the data covariance is well-conditioned. The spectral gap is what allows us to define a proper early stopping time, which implicitly connects to the spectrum of $K$ via the previous argument. We note that, if $\evmin{\E \left[ x^\top x \right]} = o(1)$, then $K$ might lose its spectral gap, a new argument would be required to set the early stopping time, and we eventually expect the final bound on the excess population risk to depend on the condition number of the covariance (as it does also in prior related work \citep{varshney22nearly, liu2023near}).}

\simonenew{When discussing Figure \ref{fig:RFfirst}, we have commented on the regularizing effect of DP-GD which can be connected to the introduction of a ridge penalty. 
In fact, adding Gaussian noise and early stopping does not allow DP-GD to interpolate the training samples in a similar way to $\ell_2$ regularization. 
The connection between $\ell_2$ regularization and early stopping
can be made quantitative by inspecting the argument of the previous Lemma \ref{lemma:earlystopping}: there, we show that the gradient flow at early stopping is close to the gradient flow at convergence when looking only at the sub-space $S_\Lambda$ spanned by the eigenvectors of $\Phi^\top \Phi$ associated to the top-$d$ eigenvalues (see \eqref{eq:fordiscussionlaterrevision}), while the same is not true in the orthogonal sub-space $S_\Lambda^\perp$. When we look at the ridge solution
\begin{equation}
    \theta^*_\lambda = \left( \Phi^\top \Phi + \lambda I \right)^{-1} \Phi^\top Y,
\end{equation}
if $\lambda$ is chosen such that $\lambda_{d+1}(\Phi^\top \Phi) \ll \lambda \ll \lambda_{d}(\Phi^\top \Phi)$, we have a similar effect: in the sub-space $S_\Lambda$, $\theta^*_\lambda$ is approximately the same as the un-regularized $\theta^*$, while we still avoid over-fitting due to the differences in the orthogonal sub-space $S_\Lambda^\perp$.}

\paragraph{Proof of Theorem \ref{thm:mainthm}.} First, note that since our model has a quadratic loss, following the first two equations in Assumption \ref{ass:privacy} we can apply Proposition \ref{prop:dpsde}. Then, the hyper-parameter $\Sigma$ in \eqref{eq:tauSigmaapp} is sufficiently large to guarantee that the solution $\Theta(\tau)$ of the SDE \eqref{eq:sde1} at time $\tau$ is $(\varepsilon, \delta)$-differentially private.
Let us introduce the shorthand $R(x) = \varphi(x)^\top \left( \hat \Theta(\tau) - \theta^* \right)$, and denote the generalization error of $\hat \Theta(\tau)$ and $\theta^*$ by $\hat{\mathcal R}$ and $\mathcal R^*$,  respectively. Then, we have that
\begin{equation}\label{eq:excessriskbound}
\begin{aligned}
    \left(\hat{\mathcal R} - \mathcal R^*\right)^2 &= \left(\E_{(x, y) \sim \mathcal P_{XY}} \left[ \left( \varphi(x)^\top \hat \Theta(\tau) - y \right)^2 \right] - \E_{(x, y) \sim \mathcal P_{XY}} \left[ \left( \varphi(x)^\top \theta^* - y \right)^2 \right]\right)^2 \\
    &= \left(\E_{(x, y) \sim \mathcal P_{XY}} \left[ R(x)  \left( \left( \varphi(x)^\top \hat \Theta(\tau) - y \right) + \left( \varphi(x)^\top \theta^* - y \right) \right)\right]\right)^2 \\
    &\leq \E_{x} \left[ R(x)^2 \right] \E_{(x, y) \sim \mathcal P_{XY}} \left[  \left( \varphi(x)^\top \hat \Theta(\tau) +  \varphi(x)^\top \theta^* - 2 y \right)^2  \right] \\
    &= \E_{x} \left[ R(x)^2 \right] \E_{(x, y) \sim \mathcal P_{XY}} \left[  \left( R(x) +  2 \varphi(x)^\top \theta^* - 2 y \right)^2  \right]\\
    & \leq 2 \E_{x} \left[ R(x)^2 \right]^2 + 2 \E_{x} \left[ R(x)^2 \right] \E_{(x, y) \sim \mathcal P_{XY}} \left[  \left(  2\varphi(x)^\top \theta^* - 2 y \right)^2  \right]\\
    & \leq 2 \E_{x} \left[ R(x)^2 \right]^2 + 16 \E_{x} \left[ R(x)^2 \right] \E_{x} \left[  \left(  \varphi(x)^\top \theta^* \right)^2  \right] + 16 \E_{x} \left[ R(x)^2 \right] \E_{y} \left[  y^2  \right],
\end{aligned}
\end{equation}
where we factorize the difference of two squares in the second line and use Cauchy-Schwartz inequality in the third line. By Lemmas \ref{lemma:noise} and \ref{lemma:earlystopping}, we have that
\begin{equation}\label{eq:excessriskbound1}
\begin{aligned}
    \E_{x} \left[ R(x)^2 \right] &= \E_{x} \left[ \left( \varphi(x)^\top \left( \hat \theta(\tau) + \tilde \Theta(\tau) - \theta^* \right)\right)^2 \right] \\
    &\leq 2 \E_{x} \left[ \left( \varphi(x)^\top \tilde \Theta(\tau) \right)^2 \right] + 2 \E_{x} \left[ \left( \varphi(x)^\top \left( \hat \theta(\tau) - \theta^* \right)\right)^2 \right] \\
    &= \bigO{ \frac{d^2 \log^{10} n}{n^2} \, \frac{\log(1 / \delta)}{\varepsilon^2}  + \frac{d}{n} + \frac{n \log^3 d}{d^{3/2}}} = o \left(1\right),
\end{aligned}
\end{equation}
with probability at least $1 - 2 \exp \left( -c_1 \log^2 n \right)$ over $X$, $V$ and $B$. 
Furthermore, by Lemma \ref{lemma:outputorder1}, we also have that $\E_{x} \left[  \left(  \varphi(x)^\top \theta^* \right)^2  \right] = \bigO{1}$ with probability at least $1 - 2 \exp \left( -c_2 \log^2 n \right)$ over $X$ and $V$. Since also $\E_{y} \left[  y^2  \right] = \bigO{1}$ by Assumption \ref{ass:data}, plugging \eqref{eq:excessriskbound1} in \eqref{eq:excessriskbound} gives
\begin{equation}
    \left(\hat{\mathcal R} - \mathcal R^*\right)^2 
    = \bigO{\frac{d^2 \log^{10} n}{n^2} \, \frac{\log(1 / \delta)}{\varepsilon^2} + \frac{d}{n} + \frac{n \log^3 d}{d^{3/2}}} = o(1),
\end{equation}
with probability at least $1 - 2 \exp \left( -c_3 \log^2 n \right)$ over $X$, $V$ and $B$. Then, since $\sqrt{a + b + c} \leq \sqrt{a} + \sqrt{b} + \sqrt{c}$, we have that, with this same probability, 
\begin{equation}\label{eq:gengaphat}
\begin{aligned}
    \left| \hat{\mathcal R} - \mathcal R^* \right| &= \bigO{\frac{d}{n \varepsilon} \, \log^{5} n \sqrt{\log(1 / \delta)} + \sqrt{\frac{d}{n}} + \sqrt{\frac{n \log^3 d}{d^{3/2}}}} \\
    &= \tilde{\mathcal O} \left(\frac{d}{n \varepsilon} + \sqrt{\frac{d}{n}} + \sqrt{\frac{n }{d^{3/2}}} \right) = o(1).
\end{aligned}
\end{equation}


To conclude the argument, we can now use Lemmas \ref{lemma:lastcliptraj} and \ref{lemma:lastclipnoise}, which guarantee that, jointly for all $i \in [n]$, 
\begin{equation}
    \sup_{t \in [0, \tau]}  \left( \left|  \varphi(x_i)^\top \hat \Theta(t) \right| +  \left|  \varphi(x_i)^\top \hat \theta(t) - y_i  \right| \right)
    \leq C_1 \log n,
\end{equation}
with probability at least $1 - 2 \exp (-c_4 \log^2 n)$ over $X$, $V$, and $B$, for some absolute constant $C_1$. 
\simonesolved{Furthermore, with probability at least $1 - 2 \exp\left( -c_5 p \right)$ over $V$, $\norm{\varphi(x_i)}_2 \leq C_2 \sqrt p$ for some absolute constant $C_2$, jointly for every $i \in [n]$ (see the argument carried out in \eqref{eq:concentrationfeaturevector1new} and \eqref{eq:concentrationfeaturevector2new}). Thus, the previous equation gives
\begin{equation}
\begin{aligned}
    \sup_{t \in [0, \tau]}  \left( \left|  \varphi(x_i)^\top \hat \Theta(t) \right| +  \left|  \varphi(x_i)^\top \hat \theta(t) - y_i  \right| \right) &\leq C_1 \log n \leq \frac{C_2 \sqrt p}{\norm{\varphi(x_i)}_2} C_1 \log n \\
    &\leq \frac{\sqrt p \log^2 n}{2 \norm{\varphi(x_i)}_2} = \frac{\clip}{2 \norm{\varphi(x_i)}_2},
\end{aligned}
\end{equation}
with probability at least $1 - 2 \exp (-c_6 \log^2 n)$ over $X$, $V$, and $B$, where $c_6$ may depend on $C_1$ and $C_2$ due to the step in the second line.}
Hence, due to the argument between \eqref{eq:definitionsofC0} and \eqref{eq:actualthesisofclipping}, we have that, with the same probability,
\begin{equation}
    \Theta(\tau) = \hat \Theta(\tau),
\end{equation}
which in turn implies $\hat {\mathcal R} = \mathcal R$, where $\mathcal R$ is the generalization error of $\Theta(\tau)$. This, together with \eqref{eq:gengaphat}, concludes the proof. \qed

\section{Concluding remarks}\label{sec:discussion}

In this work, we show that privacy can be obtained for free, i.e., the excess population risk is negligible even when $\varepsilon=o(1)$, in sufficiently over-parameterized random feature models. This proves that over-parameterization is not at odds with private learning.

We focus on the regime where the number of samples $n$ scales with the input dimension $d$ as $d \ll n \ll d^{3 / 2}$, but we believe that our approach can be applied in much wider generality. In fact, we crucially leverage the fact that the test loss of the GD solution $\theta^*$ does not decrease with respect the case $n = \Theta(d)$. This in turn means that there is a surplus of samples that can be used to learn privately.
A similar plateau in the test loss has been demonstrated for kernel ridge regression by a recent line of work \citep{ghorbani2021linearized, mmm2022, hu2024asymptotics} tackling increasingly general settings, and showing privacy for free here is an exciting direction for future research.

\section*{Acknowledgments}

This research was funded in whole or in part by the Austrian Science Fund (FWF) 10.55776/COE12. For the purpose of open access, the author has applied a CC BY public copyright license to any Author Accepted Manuscript version arising from this submission. The authors were also supported by the 2019 Lopez-Loreta prize, and Simone Bombari was supported by a Google PhD fellowship. The authors would like to thank (in alphabetical order) Diyuan Wu, Edwige Cyffers, Francesco Pedrotti, Inbar Seroussi, Nikita P. Kalinin, Pietro Pelliconi, Roodabeh Safavi, Yizhe Zhu, and Zhichao Wang for helpful discussions.

{
\small

\bibliographystyle{plainnat}
\bibliography{bibliography_new.bib}

}

\newpage

\appendix

\section{Further discussion on related work}\label{app:related}

In Section \ref{sec:related}, we have discussed how existing bounds are unable to 
tackle an over-parameterized RF model when $\varepsilon = \Theta(1)$.  
This requires the characterization of different quantities, \emph{i.e.}, the Lipschitz constant of the loss $L$, the norm of the baseline solution $\norm{\theta^*}_2$, the rank of the matrix $M$, defined as the projector on the column space of $ \E_{x} [ \varphi(x) \varphi(x)^\top ]$, and the quantity $\norm{M \theta^*}_2$. We provide the desired characterization in the following paragraphs.

\paragraph{Estimate of $\norm{\theta^*}_2$.} Recall that $\theta^* = \Phi^+ Y$, which readily implies that $\norm{\theta^*}_2 \leq \opnorm{\Phi^+} \norm{Y}_2$. An application of Lemma \ref{lemma:hyperK} gives that, with high probability, $\opnorm{\Phi^+} = \bigO{p^{-1/2}}$. Then, as the labels are bounded, we have $\norm{\theta^*}_2 = \bigO{\sqrt{n/p}}$. This upper bound is tight in different settings, \emph{e.g.}, when the labels contain random noise independent from the data. In this case, an argument similar to the one in Theorem 2 of \cite{bombari2023universal}, together with Lemma \ref{lemma:hyperK}, guarantees that $\| \theta^* \|_2 = \Theta(\| \Phi^+ \|_F) = \Theta(\sqrt{n / p})$.

\paragraph{Estimate of $L$ via a bounded set $\mathcal B$ with radius $\beta = \Theta(\norm{\theta^*})$.} Denoting with $(x, y)$ a generic input-label pair, we can write
\begin{equation}\label{eq:Lipconstworst}
    L = \sup_{\theta \in \mathcal B} \norm{\nabla_\theta \ell(\varphi(x)^\top \theta - y)}_2 = \sup_{\theta \in \mathcal B} 2 \norm{\varphi(x)} \left|\varphi(x)^\top \theta - y\right| = \Theta \left(\beta \norm{\varphi(x)}_2^2 \right),
\end{equation}
where we use Cauchy-Schwartz inequality in the last step and the fact that $\norm{\varphi(x)}_2 \beta \gg |y|$ in our data scaling. Then, noting that $\norm{\varphi(x)}_2 = \Theta(\sqrt p)$ (which holds with high probability) 
and plugging the value $\beta = \Theta(\norm{\theta^*})$ obtained in the previous paragraph, we obtain $L = \Theta( \sqrt{np})$, and consequently $L \norm{\theta^*}_2 = \Theta(n)$, as reported in Section \ref{sec:related}.

\paragraph{Estimate of $L$ via the set $\mathcal C$.} Alternatively, one could be interested in an RF model with Lipschitz loss, as the clipped loss defined in \eqref{eq:lclip} with $\clip$ set as in \eqref{eq:tauSigmaapp}. This choice would guarantee that the Lipschitz constant $L$ of the loss scales as $\sqrt p$ (excluding poly-logarithmic factors), which corresponds to the norm of the features $\norm{\varphi(x)}_2$.
In this setting, we can show that the solution of gradient flow on this loss corresponds to the solution obtained by optimizing the quadratic loss, as the dynamics fully happens in the set $\mathcal C$ where the two losses coincide (see Lemma \ref{lemma:lastcliptraj}). Thus, this allows to improve the previous estimate by a factor $\sqrt n$, as we now have $L \norm{\theta^*}_2 = \Theta( \sqrt p \sqrt{n / p}) = \Theta(\sqrt n)$. However, even with this improvement, one cannot obtain meaningful guarantees from \citet{jain14} in the over-parameterized regime, as discussed in Section \ref{sec:related}.

\paragraph{Estimate of $\rank(M)$ and $\norm{M\theta^*}_2$ from \cite{song2021evading}.}
Let $M_D$ be the projector over the column space of $\Phi^\top \Phi$. Then, under sufficient regularity of the data distribution $\mathcal P_X$, it can be shown that $M_D$ spans a subspace of the one spanned by $M$ (this exact argument is used also in the proof of Theorem 3.2 of \cite{song2021evading}). Then, since $M_D \theta^* = \theta^*$ and $\rank(M_D) = n$ from Lemma \ref{lemma:hyperK}, we also have that $\norm{M \theta^*}_2 = \norm{\theta^*}_2$ and $\min \left(\rank (M), n \right) = n$. Then, in the over-parameterized RF model, the bound of \cite{song2021evading} corresponds to the one in \cite{jain14}, and the previous considerations hold.

\section{Proofs for Section \ref{sec:dpgd}}\label{app:dp}

We start by introducing the notion of \emph{sensitivity}. Let $\mu: \mathcal D \to \R^p$ be an arbitrary, deterministic, $p$-dimensional function, where $\mathcal D$ represents the space of datasets. We define its $\ell_2$ sensitivity to be
\begin{equation}
    \Delta_2 \mu = \sup_{D \textup{ adjacent with } D'} \norm{\mu(D) - \mu(D')}_2,
\end{equation}
where $D \textup{ adjacent with } D'$ means that the two datasets $D$ and $D'$ differ by only one training sample.
The DP-GD algorithm is obtained through the composition of $T$ independent Gaussian mechanisms (see \cite{Dwork2014}, Theorem A.1), as every iteration includes a (clipped) gradient update and a Gaussian noise injection. In this case, we have that Algorithm \ref{alg:dp-gd} is a randomized mechanism $\mathcal A$ consisting of a sequence of adaptive (Gaussian) mechanisms $\mathcal M_1, \ldots, \mathcal M_T$ where $\mathcal M_t \colon \R^p \times \mathcal D \to \R^p$ for all $t \in [T]$. One option is then to use the advanced composition Theorem (see Theorem 3.20 in \citep{Dwork2014}) 
to compute the privacy guarantees of the algorithm $\mathcal A$. Unfortunately, this approach is not fruitful in our case,
as it would involve an additional $\log T$ term in our final result in Theorem \ref{prop:alg1dp}, which would make the variance of the noise added at each iteration to diverge, in the limit of $\eta \to 0$, $T \to + \infty$, with $\eta T \to \tau$.
Thus, to prove Proposition \ref{prop:alg1dp}, we calculate the privacy guarantees through the moment accountant method, also used by \cite{Abadi2016} to show their Theorem 2 (which we restate in Theorem \ref{thm:moment}, for the sake of completeness and notation). This allows for a tighter tracking of the privacy budget, exploiting the independence between the Gaussian mechanisms at every iteration of Algorithm \ref{alg:dp-gd}.

While our approach is conceptually similar to the one used in Theorem 1 of \cite{Abadi2016}, in contrast with \cite{Abadi2016} we consider full batch gradient descent and make explicit the dependence on the learning rate $\eta$, given our choice of scaling for Algorithm \ref{alg:dp-gd}.
We rely on the notion of \emph{privacy loss}, which describes the difference that a randomized mechanism has on two adjacent datasets. More specifically, given two adjacent datasets $D, D' \in \mathcal D$, a randomized mechanism $\mathcal M_t \colon \R^p \times \mathcal D \to \R^p$ and an auxiliary input $\theta_{t-1} \in \R^p$, the \emph{privacy loss} at the output $\theta$ takes the form
\begin{equation}\label{eq:privacyloss}
  \gamma(\theta; \mathcal M_t,  \theta_{t-1}, D, D') = \log \frac{p \left( \mathcal M_t ( \theta_{t - 1}, D) = \theta \right)
}{p \left( \mathcal M_t( \theta_{t - 1}, D') = \theta \right)},
\end{equation}
where the notation $p \left( Z = \theta \right)$ indicates the value of the law of a random variable $Z$ evaluated in $\theta$. 
As in \cite{Abadi2016}, we define the $\lambda$-th moment $\alpha_{\mathcal M_t}(\lambda; \theta_{t - 1}, D, D')$ as the logarithm of the moment generating function of the privacy loss evaluated in $\lambda$, \emph{i.e.},
\begin{equation}\label{eq:logmgf}
\alpha_{\mathcal M_t}(\lambda; \theta_{t - 1}, D, D') = \log \E_{\theta \sim \mathcal M_t(\theta_{t - 1}, D)}\left [\exp(\lambda \gamma(\theta; \mathcal M_t, \theta_{t - 1}, D, D')) \right].
\end{equation}
Note that the expectation is taken with respect to the probability distribution of the parameters given by the mechanism applied to the dataset $D$. Then, we can define
\begin{equation}\label{eq:alpha}
\alpha_{\mathcal M_t}(\lambda) = \sup_{\theta_{t - 1}, D \textup{ adjacent with } D'} \alpha_{\mathcal M_t}(\lambda; \theta_{t - 1}, D, D'),
\end{equation}
where the supremum is taken over all possible $\theta_{t - 1}$ and all the adjacent datasets $D, D'$.

We now re-state Theorem 2 of \cite{Abadi2016} which translates the moment accountant to $(\varepsilon, \delta)$ privacy guarantees. Next, we bound the sensitivity of each iteration of Algorithm \ref{alg:dp-gd} and, therefore, the corresponding value of $\alpha_{\mathcal M_t}(\lambda)$ (see Lemmas \ref{lemma:sensitivitygrad} and \ref{lemma:mgf}). The proof of Proposition \ref{prop:alg1dp} will then follow.

\begin{theorem}[Theorem 2 of \cite{Abadi2016}]\label{thm:moment}
Let $\mathcal A$ consists of a sequence of independent adaptive mechanisms $\mathcal M_1, \ldots, \mathcal M_T$ where $\mathcal M_t \colon \R^p \times \mathcal D \to \R^p$. Let $\alpha_{\mathcal M_t}(\lambda)$ be defined according to \eqref{eq:alpha}. Then, the following properties hold.
\begin{enumerate}
\item \textbf{Composability.}
For any $\lambda$,
\begin{equation}
\alpha_{\mathcal A}(\lambda) \leq \sum_{t=1}^T \alpha_{\mathcal M_t}(\lambda).
\end{equation}
\item \textbf{Tail bound.}
For any $\varepsilon > 0$, $\mathcal A$ is $(\varepsilon, \delta)$-differentially private for
\begin{equation}
\delta = \inf_{\lambda} \exp(\alpha_{\mathcal A}(\lambda) - \lambda \varepsilon).
\end{equation}
\end{enumerate}
\end{theorem}

\begin{lemma}\label{lemma:sensitivitygrad}
Consider the $t$-th iteration of Algorithm \ref{alg:dp-gd} on a dataset $D$ without noise, \emph{i.e.},
\begin{equation}
     \mu_{\theta_{t - 1}, D} = \theta_{t-1} - \frac{\eta}{n} \sum_{(x_i, y_i) \in D} g_{\clip}(x_i, y_i, \theta_{t-1}).
\end{equation}
Then, we have
\begin{equation}
    \Delta_2 \mu = \sup_{\theta_{t - 1} \in \R^p, D \textup{ adjacent with } D'} \norm{\mu_{\theta_{t - 1}, D} - \mu_{\theta_{t - 1}, D'}}_2 \leq \frac{2 \eta \clip}{n}.
\end{equation}
\end{lemma}
\begin{proof}
    We have
    \begin{equation}
        \begin{aligned}
            \Delta_2 \mu  &= \frac{\eta}{n} \sup_{\theta_{t - 1} \in \R^p, D \textup{ adjacent with } D'} \norm{\sum_{(x_i, y_i) \in D} g_{\clip}(x_i, y_i, \theta_{t-1}) - \sum_{(x_i, y_i) \in D'} g_{\clip}(x_i, y_i, \theta_{t-1})}_2 \\
            &= \frac{\eta}{n} \sup_{\theta_{t - 1} \in \R^p , (x, y), (x', y') \in (\mathcal X, \mathcal Y)} \norm{g_{\clip}(x, y, \theta_{t-1}) -  g_{\clip}(x', y', \theta_{t-1})}_2 \\
            &\leq \frac{\eta}{n} \sup_{\theta_{t - 1} \in \R^p, (x, y), (x', y') \in (\mathcal X, \mathcal Y)} \left( \norm{g_{\clip}(x, y, \theta_{t-1})}_2 + \norm{g_{\clip}(x', y', \theta_{t-1})}_2 \right) \\
            &\leq \frac{2 \eta \clip}{n},
        \end{aligned}
    \end{equation}
    where the last inequality comes from $\norm{g_{\clip}(x, y, \theta_{t-1})}_2 \leq \clip$ for any $(x, y) \in (\mathcal X, \mathcal Y)$ and any $\theta_{t-1} \in \R^p$.
\end{proof}

\begin{lemma}\label{lemma:mgf}
Let $\mathcal M_t$ be the randomized mechanism induced by the $t$-th iteration of Algorithm \ref{alg:dp-gd}. Then, for every $t \in [T]$ we have
\begin{equation}
\alpha_{\mathcal M_t}(\lambda) = \sup_{\theta_{t - 1}, D \textup{ adjacent with } D'} \alpha_{\mathcal M_t}(\lambda; \theta_{t - 1}, D, D') \leq \frac{\eta}{2 \sigma^2} \left( \lambda + \lambda^2 \right).
\end{equation}
\end{lemma}
\begin{proof}
    As $\mathcal M_t$ is a Gaussian mechanism with parameter $\varsigma = \sqrt \eta \, \frac{2 \clip}{n} \sigma$, we have
    \begin{equation}\label{eq:distrmecht}
        p \left( \mathcal M_t ( \theta_{t - 1}, D) = \theta \right) = \frac{1}{\sqrt{ ( 2 \pi \varsigma^2 )^p}} \exp \left( - \frac{ \norm{\theta - \mu_{\theta_{t - 1}, D}}_2^2}{2 \varsigma^2} \right),
    \end{equation}
    where $\mu_{\theta_{t - 1}, D}$ represents the update without noise
    \begin{equation}\label{eq:muD}
        \mu_{\theta_{t - 1}, D} = \theta_{t-1} - \frac{\eta}{n} \sum_{(x_i, y_i) \in D} g_{\clip}(x_i, y_i, \theta_{t-1}).
    \end{equation}
    Equivalently, \eqref{eq:distrmecht} also holds for $\P \left( \mathcal M_t ( \theta_{t - 1}, D') = \theta \right)$, with $\mu_{\theta_{t - 1}, D'} = \theta_{t-1} - \eta g_{\theta_{t-1}, D'}$.
    Then, we have that the privacy loss defined in \eqref{eq:privacyloss} takes the form
    \begin{equation}
    \begin{aligned}
    \gamma&(\theta; \mathcal M_t, \theta_{t - 1}, D, D') \\
    &= \log \frac{p \left( \mathcal M_t ( \theta_{t - 1}, D) = \theta \right)
}{p \left( \mathcal M_t( \theta_{t - 1}, D') = \theta \right)} \\
    &= - \frac{1}{2 \varsigma^2} \left ( \norm{\theta - \mu_{\theta_{t - 1}, D}}_2^2 - \norm{\theta - \mu_{\theta_{t - 1}, D'}}_2^2 \right) \\
    &= - \frac{1}{2 \varsigma^2} \left ( 2 \theta^\top (\mu_{\theta_{t - 1}, D'} - \mu_{\theta_{t - 1}, D}) + \norm{\mu_{\theta_{t - 1}, D}}_2^2 - \norm{\mu_{\theta_{t - 1}, D'}}_2^2 \right) \\
    &= - \frac{1}{2 \varsigma^2} \left ( 2 (\theta - \mu_{\theta_{t - 1}, D})^\top (\mu_{\theta_{t - 1}, D'} - \mu_{\theta_{t - 1}, D}) - \norm{\mu_{\theta_{t - 1}, D}}_2^2 + 2 \mu_{\theta_{t - 1}, D}^\top \mu_{\theta_{t - 1}, D'} - \norm{\mu_{\theta_{t - 1}, D'}}_2^2 \right) \\
    &= - \frac{1}{2 \varsigma^2} \left ( 2 (\theta - \mu_{\theta_{t - 1}, D})^\top \Delta_{\theta_{t - 1}, D, D'} - \norm{\Delta_{\theta_{t - 1}, D, D'}}_2^2 \right),
    \end{aligned}
    \end{equation}
    where we introduce the shorthand $\Delta_{\theta_{t - 1}, D, D'} = \mu_{\theta_{t - 1}, D'} - \mu_{\theta_{t - 1}, D}$.
    We can now compute the moment generating function of the privacy loss
    \begin{equation}
    \begin{aligned}
        \alpha&_{\mathcal M_t}(\lambda;\theta_{t - 1},D,D') \\
        &= \log \E_{\theta \sim \mathcal M_t(\theta_{t - 1}, D)}\left [\exp(\lambda \gamma(\theta; \mathcal M_t, \theta_{t - 1}, D, D')) \right] \\
        &= \log \left( \exp \left( \lambda \frac{\norm{\Delta_{\theta_{t - 1}, D, D'}}_2^2}{2 \varsigma^2} \right) \E_{\theta \sim \mathcal M(\theta_{t - 1}, D)}\left [\exp \left( - \lambda \frac{ (\theta - \mu_{\theta_{t - 1}, D})^\top \Delta_{\theta_{t - 1}, D, D'}}{\varsigma^2} \right) \right] \right) \\
        &= \lambda \frac{\norm{\Delta_{\theta_{t - 1}, D, D'}}_2^2}{2 \varsigma^2} + \log \E_{\theta \sim \mathcal N(\mu_{\theta_{t - 1}, D}, \varsigma^2 I_p )}\left [\exp \left( - \lambda \frac{ (\theta - \mu_{\theta_{t - 1}, D})^\top \Delta_{\theta_{t - 1}, D, D'}}{\varsigma^2} \right) \right] \\
        &= \lambda \frac{\norm{\Delta_{\theta_{t - 1}, D, D'}}_2^2}{2 \varsigma^2} + \log \E_{\theta' \sim \mathcal N(0, I_p)}\left [\exp \left( - \lambda \frac{{\theta'}^\top \Delta_{\theta_{t - 1}, D, D'}}{\varsigma} \right) \right] \\
        &= \lambda \frac{\norm{\Delta_{\theta_{t - 1}, D, D'}}_2^2}{2 \varsigma^2} + \log \E_{\rho \sim \mathcal N(0, 1)}\left [\exp \left( - \lambda \frac{\norm{\Delta_{\theta_{t - 1}, D, D'}}_2}{\varsigma} \rho \right) \right] \\
        &= \lambda \frac{\norm{\Delta_{\theta_{t - 1}, D, D'}}_2^2}{2 \varsigma^2} + \lambda^2 \frac{\norm{\Delta_{\theta_{t - 1}, D, D'}}_2^2}{2 \varsigma^2} \\
        &= \frac{\norm{\Delta_{\theta_{t - 1}, D, D'}}_2^2}{2 \varsigma^2} \left( \lambda + \lambda^2 \right),
    \end{aligned}
    \end{equation}
    where we perform a change of variable in the fifth line, use the rotational invariance of the Gaussian distribution in the sixth line, and compute the moment generating function of a standard Gaussian on the seventh line.
    Thus, we can write
    \begin{equation}\label{eq:lemmamgf}
    \alpha_{\mathcal M_t}(\lambda) = \sup_{\theta_{t - 1}, D \textup{ adjacent with } D'} \alpha_{\mathcal M_t}(\lambda; \theta_{t - 1}, D, D') = \frac{(\Delta_2 \mu)^2}{2 \varsigma^2} \left( \lambda + \lambda^2 \right),
    \end{equation}
    where
    \begin{equation}
        \Delta_2 \mu = \sup_{\theta_{t - 1}, D \textup{ adjacent with } D'} \norm{\Delta_{\theta_{t - 1}, D, D'}}_2
    \end{equation}
    is the $\ell_2$ sensitivity of the gradient updates in Algorithm \ref{alg:dp-gd}. By Lemma \ref{lemma:sensitivitygrad}, we have that $\Delta_2 \mu \leq 2 \eta \clip / n$. Thus, plugging $\varsigma = \sqrt \eta \, \frac{2 \clip}{n} \sigma$ in \eqref{eq:lemmamgf}, we readily get the desired result.
    \end{proof}

\paragraph{Proof of Proposition \ref{prop:alg1dp}.}
Algorithm \ref{alg:dp-gd} is a randomized algorithm $\mathcal A$ consisting of a sequence of independent adaptive mechanisms $\mathcal M_1, \ldots, \mathcal M_T$ such that, for all $t \in [T]$, by Lemma \ref{lemma:mgf}, we have
\begin{equation}
    \alpha_{\mathcal M_t}(\lambda) \leq \frac{\eta}{2 \sigma^2} \left( \lambda + \lambda^2 \right).
\end{equation}
Thus, by composability (see Theorem \ref{thm:moment}), we have
\begin{equation}
    \alpha_{\mathcal A}(\lambda) \leq \frac{\eta T}{2 \sigma^2}  \left( \lambda + \lambda^2 \right).
\end{equation}

Set $\delta \in (0, 1)$ and $\varepsilon \in (0, 8 \log (1 / \delta))$, and suppose that
\begin{equation}\label{eq:thmsupp}
     \frac{\eta T}{2 \sigma^2} \leq \frac{\varepsilon^2}{16 \log(1 / \delta)}.
\end{equation}
Then, considering $\lambda > 0$, we have that
\begin{equation}
\begin{aligned}
    \exp \left( \alpha_{\mathcal A}(\lambda) - \lambda \varepsilon \right) &\leq \exp \left( \frac{\varepsilon^2}{16 \log(1 / \delta)} \left( \lambda + \lambda^2 \right) - \lambda \varepsilon \right) \\
    &\leq \exp \left( \frac{\varepsilon^2}{16 \log(1 / \delta)}\lambda^2 - \left( 1 - \frac{\varepsilon}{16 \log(1 / \delta)} \right) \lambda \varepsilon \right) \\
    &\leq \exp \left( \frac{\varepsilon^2}{16 \log(1 / \delta)}\lambda^2 - \frac{\lambda \varepsilon}{2} \right).
\end{aligned}
\end{equation}
Setting $\lambda^* = 4 \log(1 / \delta) / \varepsilon$ we get
\begin{equation}
\begin{aligned}
    \exp \left( \alpha_{\mathcal A}(\lambda^*) - \lambda^* \varepsilon \right) &\leq \exp \left( \frac{\varepsilon^2}{16 \log(1 / \delta)}{\lambda^*}^2 - \frac{\lambda^* \varepsilon}{2} \right) \\
    &= \exp \left( \log(1 / \delta) - 2 \log(1/\delta) \right) \\
    &= \exp \left( - \log(1 / \delta) \right) = \delta.
\end{aligned}
\end{equation}
Then, by the tail bound in Theorem \ref{thm:moment}, we have that $\mathcal A$ is $(\varepsilon, \delta)$-differentially private. In our argument, we consider $\delta \in (0, 1)$, $\varepsilon \in (0, 8 \log (1 / \delta))$, and the inequality in \eqref{eq:thmsupp}, which can be rewritten as 
\begin{equation}
    \sigma \geq \sqrt {\eta T} \frac{\sqrt{8 \log(1 / \delta)}}{\varepsilon},
\end{equation}
which readily gives the desired result. \qed

We now move to the proof of Proposition \ref{prop:clipandcliploss}, and then present the auxiliary Proposition \ref{prop:Lclipisclip} that will be useful in the discussion on the SDE defined in \eqref{eq:sde1}. Next, in Proposition \ref{prop:convergenceEM}  we formalize the convergence of Algorithm \ref{alg:dp-gd} in the limit of small learning rates $\eta \to 0$. To conclude, we state the auxiliary Lemma \ref{lemma:convind}, useful for the final proof of Proposition \ref{prop:dpsde}.

\paragraph{Proof of Proposition \ref{prop:clipandcliploss}.}
At every iteration $t$ of Algorithm \ref{alg:dp-gd}, we have
\begin{equation}
    \begin{aligned}
        g_{\clip}(x_i, y_i, \theta_{t-1}) &=  g(x_i, y_i, \theta_{t-1}) / \max\left(1, \frac{\|g(x_i, y_i, \theta_{t-1})\|_2}{\clip}\right)\\
        &= \frac{\varphi(x_i) \ell'(\varphi(x_i)^\top \theta_{t-1} - y_i)}{\max\left(1, \frac{\norm{\varphi(x_i)}_2 \left| \ell' \left(\varphi(x_i)^\top \theta_{t-1} - y_i \right) \right| }{\clip}\right)} \\
        &= \varphi(x_i) \ell'_{i, \clip} \left( \varphi(x_i)^\top \theta_{t-1} - y_i\right) \\
        &= \nabla_{\theta} \ell_{i, \clip} \left( \varphi(x_i)^\top \theta_{t-1} - y_i \right). \qed
    \end{aligned}
\end{equation}

\begin{proposition}\label{prop:Lclipisclip}
    Let $\ell: \R \to \R$ be a differentiable function with Lipschitz-continuous derivative. Then, for all $\theta \in \R^p$, we have
    \begin{equation}
        \norm{\nabla \mathcal L_{\clip}(\theta)}_2 \leq \clip.
    \end{equation}
    Furthermore, there exists a constant $K$ such that, for all $\theta, \theta' \in \R^p$, we have
    \begin{equation}
        \norm{\nabla \mathcal L_{\clip}(\theta) - \nabla \mathcal L_{\clip}(\theta')}_2 \leq K \norm{\theta - \theta'}_2,
    \end{equation}
    \simonenew{with $K \leq L \sum_{i = 1}^n \norm{\varphi(x_i)}_2^2 / n$, where $L$ is the Lipschitz constant of $\ell'$.}
\end{proposition}
\begin{proof}
    Following the definition in \eqref{eq:lclip}, we have that, 
    for all $i \in [n]$, 
    \begin{equation}
        \nabla_\theta \ell_{i, \clip} \left( \varphi(x_i)^\top \theta - y_i \right) = \nabla_\theta \ell \left( \varphi(x_i)^\top \theta - y_i \right) / \max\left(1, \frac{\norm{ \nabla_\theta \ell \left( \varphi(x_i)^\top \theta - y_i \right)}_2}{\clip}\right),
    \end{equation}
    which readily gives
    \begin{equation}
        \norm{\nabla_\theta \ell_{i, \clip} \left( \varphi(x_i)^\top \theta - y_i \right)}_2 \leq \clip.
    \end{equation}

    Then, we can conclude
    \begin{equation}\label{eq:Lclipislip}
    \begin{aligned}
        \norm{\nabla \mathcal L_{\clip}(\theta)}_2 &= \norm{\nabla_\theta \frac{1}{n} \sum_{i = 1}^n \ell_{i, \clip}(\varphi(x_i)^\top \theta - y_i)}_2 \\
        &\leq \frac{1}{n} \sum_{i = 1}^n \norm{\nabla_\theta \ell_{i, \clip}(\varphi(x_i)^\top \theta - y_i)}_2 \\
        &\leq \clip.
    \end{aligned}
    \end{equation}

    For the second part of the claim, following a similar argument as before and exploiting the chain rule, it is sufficient to show that, for every $i \in [n]$, $\ell_{i, \clip}$ has Lipschitz-continuous derivative.
    For all $z \in \R$ such that $\ell'(z) > 0$, we can write $ \ell_{i, \clip}'(z) = \min \left( \ell' (z), \clip / \norm{\varphi(x_i)}_2 \right)$. Then, if $\ell '$ is Lipschitz-continuous, also $\ell_{i, \clip}'$ is, as it can be written as composition of Lipschitz-continuous functions. A similar argument holds for all $z \in \R$ such that $\ell'(z) < 0$. Furthermore, for all $z$ such that $\left| \ell'(z) \right| < \clip / \norm{\varphi(x_i)}_2$, we simply have $\ell_{i, \clip}' = \ell '$. Thus, to prove that $\ell_{i, \clip}'$ is Lipschitz-continuous, it suffices that $\ell '$ is Lipschitz-continuous, as assumed by the proposition. \simonenew{Notice that this argument also proves that the Lipschitz constant of $\ell_{i, \clip}'$ is smaller or equal than the Lipschitz constant $L$ of $\ell'$. Thus, we have
    \begin{equation*}
        \begin{aligned}
            \norm{\nabla \mathcal L_{\clip}(\theta) - \nabla \mathcal L_{\clip}(\theta')}_2 &= \norm{\nabla \frac{1}{n} \sum_{i = 1}^n \ell_{i, \clip}(\varphi(x_i)^\top \theta - y_i) - \nabla \frac{1}{n} \sum_{i = 1}^n \ell_{i, \clip}(\varphi(x_i)^\top \theta' - y_i)}_2 \\
            &= \norm{\frac{1}{n} \sum_{i = 1}^n \varphi(x_i) \left( \ell'_{i, \clip}(\varphi(x_i)^\top \theta - y_i) -  \ell'_{i, \clip}(\varphi(x_i)^\top \theta' - y_i) \right)}_2 \\
            &\leq \frac{1}{n} \sum_{i = 1}^n \norm{ \varphi(x_i)}_2 \left |  \ell'_{i, \clip}(\varphi(x_i)^\top \theta - y_i) -  \ell'_{i, \clip}(\varphi(x_i)^\top \theta' - y_i)  \right| \\
            &\leq \frac{L}{n} \sum_{i = 1}^n \norm{ \varphi(x_i)}_2^2 \norm{\theta - \theta'}_2,
        \end{aligned}
    \end{equation*}
    which proves the last statement.}
\end{proof}

Due to Proposition \ref{prop:Lclipisclip}, the SDE in \eqref{eq:sde1} satisfies the assumptions A1-4 in Section 4.5 of \cite{kloeden2011numerical}. This means that \eqref{eq:sde1} admits a unique strong solution $\Theta(t)$ (see Theorem 4.5.3 in \cite{kloeden2011numerical}). Furthermore, we also have that the assumptions in Equations (1.5), (1.6), (1.7) of \cite{menozzi2021density} are verified, implying that at every fixed time $t > 0$, the probability density function $p(\Theta(t))$ is continuous with respect to the coordinate $\theta$ (see Theorem 1.2 in \cite{menozzi2021density}). This property will be used later in Lemma \ref{lemma:convind} and Theorem \ref{prop:dpsde}.

In Section \ref{sec:dpgd}, we have mentioned that \eqref{eq:discretization} is the Euler-Maruyama discretization of \eqref{eq:sde1}. 
This can be formalized defining a family of random variables $\{ \theta^1_T, \theta^2_T, \ldots, \}$, where $\theta^j_T$ is the output of Algorithm \ref{alg:dp-gd} with learning rate $\eta_j := \eta / j$ and number of iterations $T_j := jT$, with $j$ being a positive integer. Algorithm \ref{alg:dp-gd} is defined such that, when performing the $t$-th iteration, we introduce the Gaussian random variable
\begin{equation}
    \frac{2 \clip}{n} \sigma \left( B(t \eta / j) - B((t - 1) \eta / j) \right),
\end{equation}
where $B(t)$ is the standard $p$-dimensional Wiener process in \eqref{eq:sde1}. This is coherent with the definition of Algorithm \ref{alg:dp-gd}, since these random variables are distributed as
\begin{equation}
    \frac{2 \clip}{n} \sigma \mathcal N(0, \eta I / j) = \sqrt {\eta_j} \frac{2 \clip}{n} \sigma \mathcal N(0, I),
\end{equation}
and they are independent with each other. Since \eqref{eq:sde1} respects all the necessary assumptions in Theorem 10.2.2 in \cite{kloeden2011numerical}, 
the following result holds:
\begin{proposition}\label{prop:convergenceEM}
Let $\tau = \eta T$. Then, for every $h > 0$, there exists $j^*$, such that for all $j > j^*$, we have
\begin{equation}\label{eq:strongconvergenceapp}
    \E \left[ \norm{\theta^j_T - \Theta(\tau) }_2 \right] < h.
\end{equation}
\end{proposition}
Proposition \ref{prop:convergenceEM} formalizes in which sense the SDE in \eqref{eq:sde1} is a continuous limit of Algorithm \ref{alg:dp-gd}. 
This is done through the family of algorithms with progressively smaller learning rates and larger number of iterations. It is important to notice that all these algorithms provide the same DP guarantee, as $\eta_j T_j = \eta T$ for every $j$, and the result of Proposition \ref{prop:alg1dp} depends on  learning rate and number of iterations only through their product $\eta T$.

\begin{lemma}\label{lemma:convind}
    We have that, for any open ball $S \subseteq \R^p$,
    \begin{equation}
        \lim_{j \to \infty} \P \left( \theta^j_T \in S \right) = \P \left( \Theta_\tau \in S \right).
    \end{equation}
\end{lemma}
\begin{proof}
    By contradiction, let the thesis be false. Then, there exists a ball $S(r_0, r)$ with center $r_0$ and radius $r$ s.t.\ there is $h^* > 0$ and $j$ arbitrarily large that satisfy
    \begin{equation}
         \left| \P \left( \theta^j_T \in S(r_0, r) \right) - \P \left( \Theta_\tau \in S(r_0, r) \right) \right| > h^*.
    \end{equation}

    We first suppose that we have
    \begin{equation}\label{eq:suppose1ball}
         \P \left( \theta^j_T  \in S(r_0, r) \right) > \P \left( \Theta_\tau \in S(r_0, r) \right) + h^*,
    \end{equation}
    for $j$ arbitrarily large.
    Then, there is a sequence of events $\omega_j$ s.t.\ $\P \left( \omega_j \right) > h^*$ and
    \begin{equation}
        \theta^j_T (\omega_j) \in S(r_0, r), \qquad \Theta_\tau(\omega_j) \notin S(r_0,r).
    \end{equation}

    As mentioned previously in this section, the law $p(\Theta_\tau)$ is continuous, and therefore bounded in the closed set $S(r_0, r')$ for any $r'$. This implies that there exists $r^* > 0$, such that
    \begin{equation}
        \P \left( \Theta_\tau \notin S(r_0, r) \, \textup{ and } \, \Theta_\tau \in S(r_0, r+r^*) \right) < \frac{h^*}{2},
    \end{equation}
    which in turn implies that there is a sequence of events $\omega^*_j$, such that $\P \left( \omega_j^* \right) > h^* / 2$ and
    \begin{equation}
        \theta^j_T (\omega^*_j) \in S(r_0,r), \qquad \Theta_\tau(\omega^*_j) \notin S(r_0,r+r^*).
    \end{equation}

    Note that, for all such events, by triangle inequality, we have
    \begin{equation}
    \begin{aligned}
        \norm{\theta^j_T (\omega^*_j) - \Theta_\tau(\omega^*_j)}_2 &\geq \norm{\Theta_\tau(\omega^*_j) - r_0}_2 - \norm{\theta^j_T (\omega^*_j) - r_0}_2 \\
        &\geq \left(r + r^* - r_0 \right) - \left(r - r_0\right) \\
        & = r^*,
    \end{aligned}
    \end{equation}
    which implies
    \begin{equation}\label{eq:contradiction1ball}
        \E \left[ \norm{\theta^j_T  - \Theta_\tau}_2 \right] \geq \P \left( \omega_j^* \right) \E \left[ \left. \norm{\theta^j_T  - \Theta_\tau}_2 \right| \omega_j^* \right] \geq \frac{h^* r^*}{2}.
    \end{equation}
    
    This last equation holds for arbitrarily large $j$, which provides the desired contradiction with Proposition \ref{prop:convergenceEM}.

    If instead of \eqref{eq:suppose1ball} we have
    \begin{equation}
         \P \left( \Theta_\tau \in S(r_0, r) \right) > \P \left( \theta^j_T \in S(r_0, r) \right) + h^*,
    \end{equation}
    for $j$ arbitrarily large,
    then the argument follows in the same way. In fact, we exploit the continuity of the law $p(\Theta_\tau)$ to state that there exists $r^* > 0$, such that
    \begin{equation}
        \P \left( \Theta_\tau \notin S(r_0, r - r^*) \, \textup{ and } \, \Theta_\tau \in S(r_0, r) \right) < \frac{h^*}{2},
    \end{equation}
    thus giving a sequence of events $\omega^*_j$, such that $\P \left( \omega_j \right) > h^* / 2$ and
    \begin{equation}
        \theta^j_T (\omega^*_j) \notin S(r_0,r), \qquad \Theta_\tau(\omega^*_j) \in S(r_0,r-r^*).
    \end{equation}
    At this point, we can again apply the triangle inequality to obtain the same contradiction as in \eqref{eq:contradiction1ball}.
\end{proof}

\paragraph{Proof of Proposition \ref{prop:dpsde}.}
The assumption on $\Sigma$ guarantees, due to Theorem \ref{prop:alg1dp}, that for all $j>0$, the discretizations $\theta^j_T$ of the SDE \eqref{eq:sde1} are, for any $\delta \in (0, 1)$ and $\varepsilon \in (0, 8 \log (1 / \delta))$, $(\varepsilon, \delta)$-differentially private, since
\begin{equation}
    \sigma = \frac{n}{2 \clip} \Sigma \geq \sqrt {\tau} \frac{\sqrt{8 \log(1 / \delta)}}{\varepsilon} = \sqrt {\eta_j T_j} \frac{\sqrt{8 \log(1 / \delta)}}{\varepsilon},
\end{equation}
where the first step is due to  \eqref{eq:Sigmahp}.

By contradiction, suppose that $\Theta_\tau$ is not $(\varepsilon, \delta)$-DP.    
This means that there exists a point in the parameters space $\bar \theta$ such that, on adjacent datasets $D$ and $D'$, we have
\begin{equation}
    p \left( \Theta_\tau(D) = \bar \theta \right) > e^\varepsilon p \left( \Theta_\tau(D') = \bar \theta \right) + \delta.
\end{equation}
Since both the laws in the previous equation are continuous, we also have that there exists an open ball $S$ centered in $\bar \theta$, such that
\begin{equation}
    \P \left( \Theta_\tau(D) \in S \right) > e^\varepsilon \P \left( \Theta_\tau(D') \in S \right) + \delta,
\end{equation}
which we rewrite as
\begin{equation}\label{eq:dpsdecontr1}
    \P \left( \Theta_\tau(D) \in S \right) = e^\varepsilon \P \left( \Theta_\tau(D') \in S \right) + \delta + h,
\end{equation}
for some $h > 0$. Now, by Lemma \ref{lemma:convind}, we have that there exists $j^*$ large enough s.t.\
\begin{equation}\label{eq:dpsdecontr2}
    \left| \P \left( \theta^{j^*}_T(D) \in S \right) - \P \left( \Theta_\tau(D) \in S \right) \right| < \frac{h}{2},
\end{equation}
and
\begin{equation}\label{eq:dpsdecontr3}
    \left| \P \left( \theta^{j^*}_T(D') \in S \right) - \P \left( \Theta_\tau(D') \in S \right) \right| < \frac{h e^{-\varepsilon}}{2}.
\end{equation}

Hence, putting together \eqref{eq:dpsdecontr1}, \eqref{eq:dpsdecontr2}, and \eqref{eq:dpsdecontr3}, we have
\begin{equation}
\begin{aligned}
    \P \left( \theta^{j^*}_T(D) \in S \right) &> \P \left( \Theta_\tau(D) \in S \right) - \frac{h}{2} \\
    &= e^\varepsilon \P \left( \Theta_\tau(D') \in S \right) + \delta + h - \frac{h}{2} \\
    &> e^\varepsilon \left( \P \left( \theta^{j^*}_T(D') \in S \right) - \frac{h e^{-\varepsilon}}{2} \right) + \delta + h - \frac{h}{2} \\
    &= e^\varepsilon  \P \left( \theta^{j^*}_T(D') \in S \right) + \delta.
\end{aligned}
\end{equation}
The previous equation implies that the discretization $\theta^{j^*}_T$ is not $(\varepsilon, \delta)$-DP, therefore contradicting the thesis of Theorem \ref{prop:alg1dp}. This provides the desired result. \qed

\section{Auxiliary Lemmas}\label{app:matconc} 

\simonesolved{This section uses the following additional notation. We denote by $\mathbf 1 : \R \to \{0, 1\}$ the indicator function, \emph{i.e.}, $\mathbf{1}(z) = 1$ if $z > 0$, and $0$ otherwise.
Given a matrix $A$, we indicate with $\sigma_{\min}(A) = \sqrt{\evmin{A^\top A}}$ its smallest singular value. \simonenew{Given a p.s.d.\ matrix $A$, $\lambda_j(A)$ denotes its $j$-th eigenvalue sorted in non-increasing order ($\evmax{A} = \lambda_1(A) \geq \lambda_2(A) \geq \ldots \geq \lambda_s(A) = \evmin{A}$).}
Given two matrices $A \in \R^{s\times s_1}$ and $B \in \R^{s\times s_2}$ , we denote by 
$A \ast B=[(A_{1:}\otimes B_{1:}),\ldots,(A_{s:}\otimes B_{s:})]^\top \in \R^{s\times s_1 s_2}$ their row-wise Kronecker product (also known as Khatri-Rao product), and $A_{j:}$ denotes the $j$-th row of $A$. Given a natural number $l$, we denote $A^{*l} = A \ast A^{*(l-1)}$, with $A^{*1} = A$. We remark that $\left[ \left( A^{*l} \right) \left(  A^{*l} \right)^\top \right]_{ij} = \left(\left[AA^\top\right]_{ij}\right)^l$. Given $A\in \R^{s \times s}$ and $B \in \R^{s \times s}$, we use the notation $A \preceq B$ ($B \succeq A$) to indicate that $B - A$ is p.s.d.} As in Section \ref{sec:proof}, we indicate with $P_\Lambda \in \R^{p \times p}$ the projector on the space spanned by the eigenvectors associated with the $d$ largest eigenvalues of $\Phi^\top \Phi$, and we condition on the high probability event described by Lemma \ref{lemma:hyperK}, which guarantees a spectral gap in $\Phi^\top \Phi$ (and, hence, the well-posedness of $P_\Lambda$).

\begin{lemma}\label{lemma:normofx}
    Let Assumption \ref{ass:data} hold, and let $x \sim \mathcal P_X$. Then, we have that
    \begin{equation}
        \norm{\E[x]}_2 = \bigO{1},
    \end{equation}
    and
    \begin{equation}
        \opnorm{\E[xx^\top]} = \bigO{1}.
    \end{equation}
\end{lemma}
\begin{proof}
    The first statement is a direct consequence of $x$ being sub-Gaussian. In fact, this implies
    \begin{equation}
        \E \left[ \left( x^\top \frac{\E[x]}{\norm{\E[x]}_2} \right)^2 \right] = \bigO{1},
    \end{equation}
    as the second moment of sub-Gaussian random variables is bounded. Thus, we have
    \begin{equation}
        \norm{\E[x]}_2^2 = \E \left[  x^\top \frac{\E[x]}{\norm{\E[x]}_2} \right]^2 \leq \E \left[ \left( x^\top \frac{\E[x]}{\norm{\E[x]}_2} \right)^2 \right] = \bigO{1},
    \end{equation}
    which provides the desired result.

    For the second statement, we have, for every $u \in \R^d$ such that $\norm{u}_2 = 1$,
    \begin{equation}
        u^\top \E \left[x x^\top \right] u = \E \left[ \left( x^\top u \right)^2 \right] = \bigO{1},
    \end{equation}
    where the second step holds as the second moment of sub-Gaussian random variables is bounded. Since $\E \left[x x^\top \right]$ is p.s.d., its operator norm is $\sup_{\norm{u}_2 = 1} u^\top \E \left[x x^\top \right] u$, and the second statement readily follows.
\end{proof}

\begin{lemma}\label{lemma:evminX}
    Let Assumption \ref{ass:data} hold, and let $d = o(n)$ and $d = o(p)$.
    Then, we have that
    \begin{equation}
        \evmin{X^\top X} = \Theta(n),
    \end{equation}
    and
    \begin{equation}
        \opnorm{X} = \bigO{\sqrt{n}},
    \end{equation}
    with probability at least $1 - 2 \exp \left( -cd \right)$ over $X$, where $c$ is an absolute constant.
    
    Also, we have that
    \begin{equation}
        \evmin{V^\top V} = \Theta \left(\frac{p}{d} \right),
    \end{equation}
    and
    \begin{equation}
        \opnorm{V} = \bigO{\sqrt{\frac{p}{d}}},
    \end{equation}
    with probability at least $1 - 2 \exp \left( -cd \right)$ over $V$.
\end{lemma}
\begin{proof}
    Notice that $X \in \R^{n \times d}$ is a matrix with i.i.d. sub-Gaussian rows with second moment matrix $\Psi = \E_{x \sim \mathcal P_X} \left[ xx^\top \right] \in \R^{d \times d}$. Thus, by Remark 5.40 in \cite{vershrandmat}, we have that
    \begin{equation}
        \opnorm{X^\top X - n \Psi} = \bigO{n \frac{d}{n}} = o(n),
    \end{equation}
    with probability at least $1 - 2 \exp \left( -c_2 d \right)$. 
    Then, conditioning on this high probability event, by Weyl's inequality, we have
    \begin{equation}\label{eq:evminXeq1}
        \evmin{X^\top X} \leq \opnorm{X^\top X} \leq \opnorm{n \Psi} + \opnorm{X^\top X - n \Psi} = \bigO{n},
    \end{equation}
    where the last step is a consequence of Lemma \ref{lemma:normofx}. This proves the upper bound in the first statement as well as the second statement.
    Furthermore, again by Weyl's inequality, we have
    \begin{equation}
        \evmin{X^\top X} \geq \evmin{n \Psi} - \opnorm{X^\top X - n \Psi} = \Omega(n),
    \end{equation}
    where the last step is a consequence of $\evmin{\Psi} = \Omega(1)$, true by Assumption \ref{ass:data}. Merging this with \eqref{eq:evminXeq1} concludes the proof of the first statement.

    The third statement can be proven in the same exact way, considering $\sqrt d V \in \R^{p \times d}$, a matrix with i.i.d. standard Gaussian (and therefore sub-Gaussian) rows with second moment matrix equal to the identity $I \in \R^{d \times d}$. Finally, the fourth statement is a consequence of Theorem 4.4.5 of \cite{vershynin2018high} 
    and the scaling $d = o(p)$, and it holds with probability $1 - 2 \exp \left( -c_3 d \right)$. This concludes the proof and provides the desired results.
\end{proof}

\begin{lemma}\label{lemma:hyperEK}
    Let Assumptions \ref{ass:data} and \ref{ass:activation} hold, and let $n \log^3 n = o \left( d^{3/2} \right)$ and $n = \omega(d)$.
    Then with probability at least $1 - 2 \exp \left( -c \log^2 n\right)$ over $X$, all the following hold
    \begin{equation}\label{eq:gap1}
        \lambda_d(\E_V [K]) = \Omega \left(\frac{pn}{d} \right),
    \end{equation}
    \begin{equation}\label{eq:gap2}
        \lambda_{d+1}(\E_V [K]) = \bigO p,
    \end{equation}
    \begin{equation}\label{eq:gap3}
        \lambda_n(\E_V [K]) = \evmin{\E_V[K]} = \Omega(p),
    \end{equation}
    where $c$ is an absolute constant.
\end{lemma}
\begin{proof}
    As $K_{ij} = \varphi(x_i)^\top \varphi(x_j)$, we have
    \begin{equation}
        \E_V [K]_{ij} = \sum_{k=1}^p \E_{v_k} [\phi(v_k^\top x_i) \phi(v_k^\top x_j)] = p  \, \E_{v} [\phi(v^\top x_i) \phi(v^\top x_j)],
    \end{equation}
    where in the last step we exploited that the $v_k$-s are identically distributed, and introduced the shorthand $v$ to indicate a random variable distributed as $v_1$.
    Since $\norm{x_i} = \sqrt{d}$ for all $i\in [n]$, and $v \sim \mathcal N(0, 1/d)$, we have that $\rho_1 := v^\top x_i$ and $\rho_2 := v^\top x_j$ are two standard Gaussian variables with correlation $x_i^\top x_j / d$. Thus, exploiting the Hermite expansion of $\phi$, we can write
    \begin{equation}
        \E_V [K]_{ij} = p \sum_{l = 0}^{+ \infty} \mu_l^2 \frac{\left(x_i^\top x_j\right)^l}{d^l} = p \sum_{l = 0}^{\infty} \mu_l^2 \frac{ \left[ \left( X^{*l} \right)  \left( X^{*l} \right)^\top \right]_{ij}}{d^l},
    \end{equation}
    where $\mu_l$ is the $l$-th Hermite coefficient of $\phi$. 
    Since $\phi$ is such that $\mu_l = 0$ for $l = \{0 , 2\}$, $\mu_1 \neq 0$, and there exists $l \geq 3$ such that $\mu_l \neq 0$ (since $\phi$ is non-linear), we can write the previous sum with the two non-0 terms
    \begin{equation}\label{eq:twotermshyperEK}
        \E_V [K]_{ij} = \frac{\mu_1^2 p}{d} XX^\top + p \sum_{l = 3}^{+ \infty} \mu_l^2 \frac{ \left[ \left( X^{*l} \right)  \left( X^{*l} \right)^\top \right]_{ij}}{d^l}.
    \end{equation}
    We analyze the terms of the sum separately.
    By Lemma \ref{lemma:evminX} we have that, with probability at least $1 - 2 \exp \left( -c_1 d \right)$ over $X$,
    \begin{equation}
        \evmin{X^\top X} = \Theta(n).
    \end{equation}
    As $X \in \R^{n \times d}$, with $d < n$, this means that (conditioning on this high probability event) $XX^\top$ is a matrix with rank equal to $d$, and that its first $d$ eigenvalues (when sorted in non-increasing order) are all $\Omega(n)$. This implies that
    \begin{equation}\label{eq:XXhyper}
        M_1 := \frac{\mu_1^2 p}{d} XX^\top
    \end{equation}
    has the first $d$ eigenvalues of order $\Omega(pn/d)$, and the remaining $n - d$ equal to 0.

    We now consider the second term of \eqref{eq:twotermshyperEK}:
    \begin{equation}
        M_2 = p \sum_{l = 3}^{\infty} \mu_l^2 \frac{ \left[ \left( X^{*l} \right)  \left( X^{*l} \right)^\top \right]_{ij}}{d^l}.
    \end{equation}
    Let us define
    \begin{equation}\label{eq:tildeM2}
        \tilde M_2 = p \left( \sum_{l = 3}^{\infty} \mu_l^2 \right) I = C p I,
    \end{equation}
    corresponding to the diagonal elements of $M_2$, where $C$ is a constant depending only on $\phi$. Notice that $C$ is strictly positive, as $\phi$ is non-linear.
    For $i \neq j$, since $x_i$ is sub-Gaussian and independent from $x_j$, we have
    \begin{equation}
        \P \left( \left| x_i^\top x_j \right| > \log n \sqrt{d} \right) < 2 \exp \left( -c_2 \log^2 n \right).
    \end{equation}
    Performing a union bound we also get
    \begin{equation}
        \P \left( \max_{i,j}\left| x_i^\top x_j \right| > \log n \sqrt{d} \right) < 2 n^2 \exp \left( -c_2 \log^2 n \right) < 2 \exp \left( -c_3 \log^2 n \right).
    \end{equation}
    This implies that
    \begin{equation}\label{eq:difftildeM2}
    \begin{aligned}
        \opnorm{M_2 - \tilde M_2} &\leq \norm{M_2 - \tilde M_2}_F \\
        &= p \norm{ \sum_{l = 3}^\infty \mu_l^2 \left( \frac{\left( X^{*l} \right)  \left( X^{*l} \right)^\top}{d^l}  - I \right) }_F \\
        &\leq p \sum_{l = 3}^\infty \mu_l^2 \norm{  \frac{\left( X^{*l} \right)  \left( X^{*l} \right)^\top}{d^l}  - I  }_F \\
        &\leq p \sum_{l = 3}^\infty \mu_l^2 \sqrt{ n^2 \left( \frac{\max_{i,j}\left| x_i^\top x_j \right| ^l} {d^{l}}\right)^2  } \\
        &\leq p \sqrt{ n^2 \left( \frac{\max_{i,j}\left| x_i^\top x_j \right| ^3} {d^{3}}\right)^2} \sum_{l = 3}^\infty \mu_l^2 \\
        &= \bigO{p \frac{n \log^3 n}{d^{3/2}}} \\
        & = o(p).
    \end{aligned}
    \end{equation}

    Merging \eqref{eq:tildeM2} and \eqref{eq:difftildeM2} together, a standard application of Weyl's inequality gives
    \begin{equation}\label{eq:perturbM2}
        \evmax{M_2} = \Theta(p),  \qquad \evmin{M_2} = \Theta(p), 
    \end{equation}
    with probability at least $1 - 2 \exp \left( -c_3 \log^2 n \right)$.
    Thus, recalling that $\lambda_i(\cdot)$ denotes the $i$-th eigenvalue of a matrix sorted in non-increasing order and applying Weyl's inequality, we have that
    \begin{equation}
        \begin{aligned}
            \lambda_i \left(M_1 + M_2 \right) &\geq \lambda_i \left(M_1 \right) - \evmax{M_2} = \Omega(pn / d), \qquad \textup{for } i \in [d], \\
            \lambda_i \left(M_1 + M_2 \right) &\leq \lambda_i \left(M_1 \right) + \evmax{M_2} = \bigO p, \qquad \textup{for } d < i \leq n,  \\
        \end{aligned}
    \end{equation}
    where in both lines we used our argument in \eqref{eq:XXhyper} and \eqref{eq:perturbM2}. Since $\E_V [K] = M_1 + M_2$, we have proved that the spectrum of $\E_V [K]$ has a gap, as
    \begin{equation}
        \lambda_d \left(\E_V [K]\right) = \Omega \left( \frac{pn}{d} \right) = \omega (p),
    \end{equation}
    and
    \begin{equation}
        \lambda_{d+1} \left(\E_V [K]\right) = \bigO{p}.
    \end{equation}
    This proves \eqref{eq:gap1} and \eqref{eq:gap2}. To show \eqref{eq:gap3}, it suffices to note that
    \begin{equation}
        \evmin{\E_V [K]} \geq \evmin{M_2} = \Theta(p),
    \end{equation}
    where the first step is true since $M_1$ is p.s.d., and the second step follows from \eqref{eq:perturbM2}.
\end{proof}

\begin{lemma}\label{lemma:renormalizedKconc}
    Let Assumptions \ref{ass:data} and \ref{ass:activation} hold, and let $n = \bigO{p / \log^4 p}$, $n \log^3 n = o \left(d^{3/2} \right)$ and $n = \omega(d)$.
    Then, we have that 
    \begin{equation}\label{eq:renormalizedKconc}
        \opnorm{\E_V[K]^{-1/2} \left( K - \E_V[K] \right) \E_V[K]^{-1/2}} = \bigO{\sqrt{\frac{n}{p}} \log n \log p},
    \end{equation}
    with probability at least $1 - 2 \exp \left( -c \log^2 n \right)$ over $X$ and $V$, where $c$ is an absolute constant.
\end{lemma}
\begin{proof}

Consider the truncated function
\begin{equation}\label{eq:truncationnew}
    \bar \phi (z) := \phi (z) \mathbf 1 \left(|z| \leq \frac{\log p}{L}\right),
\end{equation}
where $L$ is the Lipschitz constant of $\phi$. Define also the truncated kernel $\bar K$ accordingly
\begin{equation}
    \bar K_{ij} :=  \bar \phi (Vx_i)^\top \bar \phi (Vx_j).
\end{equation}
We now compare $\E_V \left[ \bar K \right]$ and $\E_V [K]$. Let $v \sim \mathcal N(0, I/d)$ and define the random variable
\begin{equation}
    E_{ij} := \mathbf 1 \left( |v^\top x_i| > \frac{\log p}{ L} \; \textup{ or } \; |v^\top x_j| > \frac{\log p}{ L} \right).
\end{equation}
As $E_{ij}$ is an indicator, $E_{ij} = E_{ij}^2$ and
\begin{equation}\label{eq:eventEnew}
    \P_V (E_{ij} = 1) \leq \P_V \left( |v^\top x_i| > \frac{\log p}{ L} \right) + \P_V \left( |v^\top x_j| > \frac{\log p}{ L} \right) \leq 2 \exp \left( - c_1 \log^2 p \right),
\end{equation}
where the second step holds as $v^\top x_j \sim \mathcal N(0, 1)$ and $v^\top x_i \sim \mathcal N(0, 1)$ in the probability space of $V$. Thus, a union bound over $i$ and $j$ gives
\begin{equation}\label{eq:KbarKsamehp}
    K = \bar K,
\end{equation}
with probability at least $1 - 2 n^2 \exp \left( - c_1 \log^2 p \right) \geq 1 - 2 \exp \left( - c_2 \log^2 n \right)$ over $V$. Furthermore, we can write
\begin{equation}
\begin{aligned} 
    \left| \E_V \left[ \bar K_{ij} \right] - \E_V [ K_{ij}] \right| &= p \left| \E_v \left[ \phi(v^\top x_i) \phi(v^\top x_j) \, E_{ij} \right] \right| \\
    &\leq p \,\E_v \left[ \left( \phi(v^\top x_i) \phi(v^\top x_j) \right)^2 \right] ^{1/2} \E_v \left[ E_{ij}^2  \right] ^{1/2} \\
    &\leq p\, \E_v \left[ \left( \phi(v^\top x_i)\right)^4 \right] ^{1/2} \P_V (E_{ij} = 1)^{1/2} \\
    &\leq 2 p \, C_1^{1/2} \exp \left( - \frac{c_1 \log^2 p}{2} \right) \\
    &\leq 2 p \exp \left( - c_3 \log^2 p \right).
\end{aligned}
\end{equation}
Here, the second and third line follow from Cauchy-Schwartz inequality; the fourth line is a consequence of \eqref{eq:eventEnew} and $\phi(v^\top x_i)$ being a sub-Gaussian random variable ($\phi(z)$ is Lipschitz) and thus with bounded fourth moment (see Equation 2.11 in \cite{vershynin2018high}).
This result holds for any $i, j \in [n]$, and therefore implies
\begin{equation}\label{eq:EKbarKclose}
\begin{aligned}
    \opnorm{ \E_V \left[ \bar K \right] - \E_V [K] } &\leq \norm{\E_V \left[ \bar K \right] - \E_V [K]}_F \\
    &\leq 2 pn \exp \left( - c_3 \log^2 p \right) \\
    &\leq 2 \exp \left( - c_4 \log^2 p \right).
\end{aligned}
\end{equation}
Then, by Weyl's inequality, we have
\begin{equation}\label{eq:evminbarEK}
    \evmin{\E_V \left[ \bar K \right]} \geq \evmin{\E_V \left[ K \right]} - \opnorm{ \E_V \left[ \bar K \right] - \E_V [K] } = \Omega(p),
\end{equation}
where the last step follows from Lemma \ref{lemma:hyperEK} and it holds with probability at least $1 - 2 \exp \left( -c_5 \log^2 n\right)$ over $X$.

We now prove the bound in \eqref{eq:renormalizedKconc} on the truncated kernel $\bar K$. Let us then define the matrix $H_k \in \R^{n \times n}$ as
\begin{equation}
    H_k = \E_V[\bar K]^{-1/2} \bar \phi(X v_k) \bar \phi(X v_k)^\top \E_V[\bar K]^{-1/2}.
\end{equation}
This definition requires $\E_V[\bar K]$ to be invertible, and it is therefore conditioned on the event described by \eqref{eq:evminbarEK}. We will condition on this high probability event over $X$ until the end of the proof. Then, we have
\begin{equation}\label{eq:bernmat1}
    \E_V[\bar K]^{-1/2} \left( \bar K - \E_V[\bar K] \right) \E_V[\bar K]^{-1/2} = \sum_{k = 1}^p H_k - \E_V\left[H_k \right].
\end{equation}

Note that there exists a constant $C$ such that
\begin{equation}\label{eq:rKc1}
\begin{aligned}
    \sup_{v_k} \opnorm{H_k - \E_V[H_k]} &\leq 2 \sup_{v_k} \opnorm{H_k} \leq 2 \opnorm{\E_V[ \bar K]^{-1/2}}^2 \sup_{v_k}  \norm{\bar \phi(X v_k)}_2^2 \\
    &\leq C \frac{\sup_{v_k} \norm{\bar \phi(X v_k)}_2^2}{p} \leq C \frac{n}{p} \log^2 p,
\end{aligned}
\end{equation}
where the first step is true because of Jensen and triangle inequality, the third step is true because of \eqref{eq:evminbarEK}, and the last step holds because, for every $i \in [n]$,
\begin{equation}
    \left| \bar \phi(v_k^\top x_i) \right| \leq \left| \phi(0) \right| + L \frac{\log p}{L} = \left| \phi(0) \right| + \log p \leq C_3 \log p,
\end{equation}
where we use that $\phi$ is $L$-Lipschitz and the definition in \eqref{eq:truncationnew}.
We also have
\begin{equation}\label{eq:rKc2}
\begin{aligned}
    \E_V \left[ H_k \right] &= \E_V[\bar K]^{-1/2} \E_V\left[ \bar \phi(X v_k) \bar \phi(X v_k^\top) \right] \E_V[\bar K]^{-1/2} \\
    &= \frac{1}{p} \E_V[\bar K]^{-1/2} \E_V[\bar K] \E_V[\bar K]^{-1/2} \\
    &= \frac{1}{p} I,
\end{aligned}
\end{equation}
which allows us to write
\begin{equation}\label{eq:boundH2}
\begin{aligned}
    \E_V \left[ \left( H_k - \E_V [H_k] \right)^2 \right] &=  \E_V \left[ H_k^2 \right] - \E_V \left[ H_k \right]^2 \\
    &\preceq \E_V \left[ H_k^2 \right] \\
    &\preceq  \E_V \left[ \sup_{v_k} \opnorm{H_k} H_k \right] \\
    &\preceq  C \frac{n \log^2 p}{p} \E_V \left[ H_k \right] \\
    &= C \frac{n \log^2 p}{p^2} I,
\end{aligned}
\end{equation}
where the third line holds since $H_k$ is p.s.d. for all $v_k$, the fourth line is a consequence of \eqref{eq:rKc1}, and the last step of \eqref{eq:rKc2}. This readily gives
\begin{equation}
    \opnorm{\E_V \left[ \left( H_k - \E_V [H_k] \right)^2 \right]} \leq \frac{Cn \log^2 p}{p^2}.
\end{equation}
Thus, \eqref{eq:bernmat1} is the sum of $p$ independent (in the probability space of $V$), mean-0, $n \times n$ symmetric random matrices, such that $\opnorm{H_k - \E_V[H_k]} \leq \frac{Cn \log^2 p}{p}$ almost surely for all $k$. Furthermore, we also have $\opnorm{\E_V \left[ \left( H_k - \E_V [H_k] \right)^2 \right]} \leq \frac{Cn \log^2 p}{p^2}$. Then, by Theorem 5.4.1 \cite{vershynin2018high}, we get
\begin{equation}
    \P_V \left( \opnorm{\E_V[\bar K]^{-1/2} \left( \bar K - \E_V[\bar K] \right) \E_V[\bar K]^{-1/2}}  \geq t \right) \leq 2n \exp \left( - \frac{t^2 / 2}{\frac{Cn \log^2 p }{p} + \frac{Cn \log^2 p }{3p} t }\right).
\end{equation}
Setting
\begin{equation}
    t = \sqrt{\frac{n}{p}} \log p \log n \leq C_4,
\end{equation}
where the last step holds since $n = \bigO{p / \log^4 p}$, we get
\begin{equation}\label{eq:whatwasfinalbeftrunc}
\begin{aligned}
    \P_V \left( \opnorm{\E_V[\bar K]^{-1/2} \left( \bar K - \E_V[\bar K] \right) \E_V[\bar K]^{-1/2}}  \geq \sqrt{\frac{n}{p}} \log n \log p \right)  \\
    \leq 2 \exp \left( \log n - \frac{\log^2 n}{2C + 2 C t / 3} \right) \\
    \leq 2 \exp \left( \log n - \frac{\log^2 n}{2C + 2 C C_4 / 3} \right) \\
    \leq 2 \exp \left( - c_6 \log^2 n \right),
\end{aligned}
\end{equation}
which gives the result in \eqref{eq:renormalizedKconc} on the truncated kernel $\bar K$. We will now translate this in the desired result.
Note that, by \eqref{eq:KbarKsamehp}, we have
\begin{equation}\label{eq:finalnewtrunc1}
    \opnorm{\E_V[K]^{-1/2} \left( K - \E_V[K] \right) \E_V[K]^{-1/2}} = \opnorm{\E_V[K]^{-1/2} \left( \bar K - \E_V[K] \right) \E_V[K]^{-1/2}},
\end{equation}
with probability at least $1 - 2 \exp \left( - c_2 \log^2 n \right)$ over $V$. Applying the triangle inequality we get
\begin{equation}\label{eq:finalnewtrunc2}
\begin{aligned}
    &\opnorm{\E_V[K]^{-1/2} \left( \bar K - \E_V[K] \right) \E_V[K]^{-1/2}} \\
    & \leq \opnorm{\E_V[K]^{-1/2} \left( \bar K - \E_V[ \bar K] \right) \E_V[K]^{-1/2}} + \opnorm{\E_V[K]^{-1/2} \left( \E_V[ \bar K] - \E_V[K] \right) \E_V[K]^{-1/2}} \\
    &\leq \opnorm{\E_V[K]^{-1/2} \left( \bar K - \E_V[ \bar K] \right) \E_V[K]^{-1/2}} + \opnorm{\E_V[K]^{-1/2}}^2 \opnorm{ \E_V[ \bar K] - \E_V[K] } \\
    &\leq \opnorm{\E_V[K]^{-1/2} \left( \bar K - \E_V[ \bar K] \right) \E_V[K]^{-1/2}} + \frac{C_5}{p} \exp \left( - c_4 \log^2 p \right),
\end{aligned}
\end{equation}
where the last step follows from \eqref{eq:EKbarKclose} and Lemma \ref{lemma:hyperEK}, and it holds with probability at least $1 - 2 \exp \left( -c_7 \log^2 n\right)$ over $X$. Then, we have
\begin{equation}\label{eq:finalnewtrunc3}
\begin{aligned}
    &\opnorm{\E_V[K]^{-1/2} \left( \bar K - \E_V[ \bar K] \right) \E_V[K]^{-1/2}} \\
    &\opnorm{\E_V[K]^{-1/2} \E_V[\bar K]^{1/2} \E_V[\bar K]^{-1/2} \left( \bar K - \E_V[ \bar K] \right) \E_V[\bar K]^{-1/2} \E_V[\bar K]^{1/2} \E_V[ K]^{- 1/2}} \\
    &\leq \opnorm{\E_V[K]^{-1/2} \E_V[\bar K]^{1/2}}^2 \opnorm{\E_V[\bar K]^{-1/2} \left( \bar K - \E_V[ \bar K] \right) \E_V[\bar K]^{-1/2}} \\
    &\leq 2 \opnorm{\E_V[\bar K]^{-1/2} \left( \bar K - \E_V[ \bar K] \right) \E_V[\bar K]^{-1/2}},
\end{aligned}
\end{equation}
where the last step holds since
\begin{equation}\label{eq:laststep}
\begin{aligned}
    \opnorm{\E_V[K]^{-1/2} \E_V[\bar K]^{1/2}}^2 &= \opnorm{\E_V[K]^{-1/2} \E_V[\bar K] \E_V[K]^{-1/2}} \\
    &\leq \opnorm{I} + \opnorm{\E_V[K]^{-1/2} \left( \E_V[\bar K] - \E_V[K] \right) \E_V[K]^{-1/2}} \\
    & \leq 1 + \frac{C_5}{p} \exp \left( - c_4 \log^2 p \right),
\end{aligned}
\end{equation}
and the last step of \eqref{eq:laststep} follows from \eqref{eq:EKbarKclose} and Lemma \ref{lemma:hyperEK}, and it holds with probability at least $1 - 2 \exp \left( -c_7 \log^2 n\right)$ over $X$.
Thus, using consecutively \eqref{eq:finalnewtrunc1}, \eqref{eq:finalnewtrunc2}, and \eqref{eq:finalnewtrunc3}, we get
\begin{equation}
\begin{aligned}
    &\opnorm{\E_V[K]^{-1/2} \left( K - \E_V[K] \right) \E_V[K]^{-1/2}} \\
    &= \opnorm{\E_V[K]^{-1/2} \left( \bar K - \E_V[K] \right) \E_V[K]^{-1/2}} \\
    &\leq \opnorm{\E_V[K]^{-1/2} \left( \bar K - \E_V[ \bar K] \right) \E_V[K]^{-1/2}} + \frac{C_5}{p} \exp \left( - c_4 \log^2 p \right) \\
    &\leq 2 \opnorm{\E_V[\bar K]^{-1/2} \left( \bar K - \E_V[ \bar K] \right) \E_V[\bar K]^{-1/2}} + \frac{C_5}{p} \exp \left( - c_4 \log^2 p \right) \\
    & = \bigO{\sqrt{\frac{n}{p}} \log n \log p},
\end{aligned}
\end{equation}
where the last step follows from \eqref{eq:whatwasfinalbeftrunc}, and the steps jointly hold with probability at least $1 - 2 \exp \left( - c_8 \log^2 n \right)$ over $X$ and $V$. This concludes the proof.
\end{proof}

\begin{lemma}\label{lemma:hyperK}
    Let Assumptions \ref{ass:data} and \ref{ass:activation} hold, and let $n = o \left( p / \log^4 p \right)$, $n \log^3 n = o \left(d^{3/2}\right)$ and $n = \omega(d)$.
    Then, with probability at least $1 - 2 \exp \left( -c \log^2 n\right)$ over $X$ and $V$, all the following hold
    \begin{equation}
        \lambda_d(K) = \Omega \left(\frac{pn}{d} \right),
    \end{equation}
    \begin{equation}
        \lambda_{d+1}(K) = \bigO p,
    \end{equation}
    \begin{equation}
        \lambda_n(K) = \evmin{K} = \Omega(p),
    \end{equation}
    where $c$ is an absolute constant.
\end{lemma}
\begin{proof}
    By Lemma \ref{lemma:renormalizedKconc}, we have that, with probability at least $1 - 2\exp \left( -c_1 \log^2 n \right)$ over $X$ and $V$, for all $u \in \R^n$ we have
    \begin{equation}
        u^\top \left(\E_V[K]^{-1/2} \left( K - \E_V[K] \right) \E_V[K]^{-1/2} \right) u \leq C \sqrt{\frac{n}{p}} \log n \log p \norm{u}_2^2,
    \end{equation}
    where $C$ is an absolute constant.
    Then, setting $u = \E_V[K]^{1/2} \hat u$, we have that, for all $\hat u \in \R^n$,
    \begin{equation}
        \hat u^\top  \left( K - \E_V[K] \right) \hat u \leq \hat u ^\top \left( C \sqrt{\frac{n}{p}} \log n \log p \, \E_V[K] \right) \hat u,
    \end{equation}
    which reads
    \begin{equation}
        K -  \E_V[K] \preceq C \sqrt{\frac{n}{p}} \log n \log p \, \E_V[K].
    \end{equation}
    In the same way, considering instead $\left(\E_V[K] - K \right)$ in the very first equation, we can also derive
    \begin{equation}
        \E_V[K] - K  \preceq C \sqrt{\frac{n}{p}} \log n \log p \, \E_V[K],
    \end{equation}
    which therefore gives
    \begin{equation}\label{eq:sandwich}
        \left( 1 - C \sqrt{\frac{n}{p}} \log n \log p \right) \E_V[K] \preceq K \preceq \left( 1 + C \sqrt{\frac{n}{p}} \log n \log p \right) \E_V[K].
    \end{equation}
By the Courant–Fischer–Weyl min-max principle, we can write
    \begin{equation}
        \lambda_i(K) = \max_{S} \min_{u \in S, \, \norm{u}_2 = 1} \left(u^\top K u\right),
    \end{equation}
    where $S$ is any $i$-dimensional subspace of $\R^n$. Let $S^* = \arg \max_{S} \min_{u \in S, \, \norm{u}_2 = 1} \left(u^\top K u\right)$. Then, by \eqref{eq:sandwich}, we have
    \begin{equation}
    \begin{aligned}
        \lambda_i(K) &= \min_{u \in S^*, \, \norm{u}_2 = 1} \left(u^\top K u\right) \\
        &\leq \min_{u \in S^*, \, \norm{u}_2 = 1} \left(u^\top \left(\left( 1 + C \sqrt{\frac{n}{p}} \log n \log p \right) \E_V[K] \right) u\right) \\
        &\leq \max_{S} \min_{u \in S, \, \norm{u}_2 = 1} \left(u^\top \left(\left( 1 + C \sqrt{\frac{n}{p}} \log n \log p \right) \E_V[K] \right) u \right)\\
        & = \left( 1 + C \sqrt{\frac{n}{p}} \log n \log p \right)  \lambda_i \left(\E_V[K] \right), 
    \end{aligned}
    \end{equation}
    where in the last line we use the same principle to express the $i$-th eigenvalue of $\E_V[K]$.
    Following the same strategy, we can therefore conclude that
    \begin{equation}
        \left( 1 - C \sqrt{\frac{n}{p}} \log n \log p \right) \lambda_i \left( \E_V[K] \right) \leq \lambda_i \left( K \right) \leq \left( 1 + C \sqrt{\frac{n}{p}} \log n \log p \right) \lambda_i \left( \E_V[K] \right).
    \end{equation}
    
    By Lemma \ref{lemma:hyperEK}, we have that, with probability at least $1 - 2 \exp \left( -c_2 \log n\right)$ over $X$, $\lambda_d \left( \E_V[K] \right) = \Omega(np / d)$, and that $\lambda_{i} \left( \E_V[K] \right) = \Theta(p)$, for $d+1 \leq i \leq n$. This implies
    \begin{equation}
        \lambda_d(K) = \Omega \left( \left( 1 - C \sqrt{\frac{n}{p}} \log n \log p \right) \frac{np}{d} \right) = \Omega \left(\frac{pn}{d} \right),
    \end{equation}
    \begin{equation}
        \lambda_{d+1}(K) = \bigO{  \left( 1 + C \sqrt{\frac{n}{p}} \log n \log p \right) p} = \bigO{   p},
    \end{equation}
    \begin{equation}
        \lambda_n(K) = \evmin{K} = \Omega\left( \left( 1 - C \sqrt{\frac{n}{p}} \log n \log p \right) p \right) = \Omega(p),
    \end{equation}
    where we use $n \log^2 n = o(p / \log^2 p)$. This gives the desired result.
\end{proof}

\begin{lemma}\label{lemma:opnormPhi}
    Let Assumptions \ref{ass:data} and \ref{ass:activation} hold, and let $d = o(p)$ and $d = o(n)$.
    Then, we have
    \begin{equation}
        \opnorm{\Phi} = \bigO{\sqrt{\frac{np}{d}}},
    \end{equation}
    with probability at least $1 - 2 \exp \left( -c d \right)$ over $X$ and $V$, where $c$ is an absolute constant.
\end{lemma}
\begin{proof}
    The $k$-th row of $\Phi^\top$ takes the form $\phi(X v_k) \in \R^n$. Since the Gaussian distribution is Lipschitz concentrated (see Theorem 5.22 in \cite{vershynin2018high}), we have that
    \begin{equation}\label{eq:Gausslipconc}
        \subGnorm{\phi(X v_k)} = \subGnorm{\phi(X v_k) - \E_{v_k} \left[ \phi(X v_k) \right]} \leq C_1 \frac{\opnorm{X}}{\sqrt d},
    \end{equation}
    where the first step holds since \simonesolved{the 0-th Hermite coefficient of $\phi$ is zero}, and the second step holds since $\phi$ is Lipschitz. Notice that the term $\sqrt d$ is due to the fact that $\sqrt d \, v_k$ is standard Gaussian. We remark that the sub-Gaussian norm in this equation is intended in the probability space of $v_k$.
    By Lemma \ref{lemma:evminX}, we have that
    \begin{equation}
        \opnorm{X} = \bigO{\sqrt{n}},
    \end{equation}
    with probability at least $1 - 2 \exp \left( -c_1 d \right)$ over $X$. Thus, conditioning on this high probability event, due to \eqref{eq:Gausslipconc}, we have that $\Phi^\top$ is a $p \times n$ matrix whose rows are i.i.d.\ mean-0 random vectors with sub-Gaussian norm $\bigO{\sqrt{n / d}}$. Then, by Lemma B.7 of \cite{bombari2022memorization}, we have
    \begin{equation}
        \opnorm{\Phi^\top} = \bigO{\sqrt{\frac{np}{d}}},
    \end{equation}
    with probability at least $1 - 2 \exp(-c_2 n)$ over $V$, which provides the desired result.
\end{proof}

\begin{lemma}\label{lemma:opnormtildeE}
    Let Assumptions \ref{ass:data} and \ref{ass:activation} hold, and let $n \log^3 n =\bigO{d^{3/2}}$. Then, we have
    \begin{equation}
        \opnorm{\E_V \left[ \tilde K  \right]} = \bigO{p},
    \end{equation}
    with probability at least $1 - 2 \exp \left( -c \log^2 n\right)$ over $X$, where $c$ is an absolute constant.
\end{lemma}
\begin{proof}
    The proof follows the same strategy as the proof of Lemma \ref{lemma:hyperEK}, with the difference that now there is no term $M_1$. In fact, $\tilde \phi$ has the same Hermite coefficients as $\phi$, except that $\tilde \mu_1 = 0$. The thesis then follows from \eqref{eq:perturbM2}. We remark that to prove this result, differently from the proof of Lemma \ref{lemma:hyperEK}, $n = \omega (d)$ is not required. \simonesolved{Also, we require only $n \log^3 n =\bigO{d^{3/2}}$ instead of $n \log^3 n =o\left( d^{3/2} \right)$. In fact, we do not need a bound on the smallest eigenvalue of $M_2$ (see \eqref{eq:perturbM2}), and therefore we just need $\opnorm{M_2 - \tilde M_2} = \bigO{p}$ (see \eqref{eq:difftildeM2}).}
\end{proof}

\begin{lemma}\label{lemma:tildetilde}
    Let Assumptions \ref{ass:data} and \ref{ass:activation} hold, and let $n = \bigO{p / \log^4 p}$ and $n \log^3 n =\bigO{d^{3/2}}$.
    Then, we have
    \begin{equation}
        \opnorm{\tilde \Phi} = \bigO{\sqrt p},
    \end{equation}
    with probability at least $1 - 2 \exp \left( -c \log^2 n \right)$ over $X$ and $V$, where $c$ is an absolute constant.
\end{lemma}
\begin{proof}
    Consider the truncated function
    \begin{equation}\label{eq:truncation}
        \bar \phi (z) := \tilde \phi (z) \mathbf 1 \left(|z| \leq \frac{\log p}{\tilde L}\right),
    \end{equation}
    where $\tilde L$ is the Lipschitz constant of $\tilde \phi$. Define the truncated kernel $\bar K$ accordingly
    \begin{equation}
        \bar K_{ij} :=  \bar \phi (Vx_i)^\top \bar \phi (Vx_j).
    \end{equation}
    
    We now compare $\E_V \left[ \bar K \right]$ and $\E_V [\tilde K]$. Let $v \sim \mathcal N(0, I/d)$ and define the random variable
    \begin{equation}
        E_{ij} := \mathbf 1 \left( |v^\top x_i| > \frac{\log p}{\tilde L} \; \textup{ or } \; |v^\top x_j| > \frac{\log p}{\tilde L} \right).
    \end{equation}
    Note that since $E_{ij}$ is an indicator we have $E_{ij} = E_{ij}^2$, and
    \begin{equation}\label{eq:eventE}
        \P_V (E_{ij} = 1) \leq \P_V \left( |v^\top x_i| > \frac{\log p}{\tilde L} \right) + \P_V \left( |v^\top x_j| > \frac{\log p}{\tilde L} \right) \leq \exp \left( - c_1 \log^2 p \right),
    \end{equation}
    where the second step holds as $v^\top x_i \sim \mathcal N(0, 1)$ in the probability space of $V$.
    
    Thus, we can write
    \begin{equation}
    \begin{aligned} 
        \left| \E_V \left[ \bar K_{ij} \right] - \E_V [ \tilde K_{ij}] \right| &= p \left| \E_v \left[ \tilde \phi(v^\top x_i)\tilde \phi(v^\top x_j) \, E_{ij} \right] \right| \\
        &\leq p \,\E_v \left[ \left( \tilde \phi(v^\top x_i) \tilde \phi(v^\top x_j) \right)^2 \right] ^{1/2} \E_v \left[ E_{ij}^2  \right] ^{1/2} \\
        &\leq p\, \E_v \left[ \left( \tilde \phi(v^\top x_i)\right)^4 \right] ^{1/2} \P_V (E_{ij} = 1)^{1/2} \\
        &\leq p \, C^{1/2} \exp \left( - \frac{c_1 \log^2 p}{2} \right) \\
        &\leq p \exp \left( - c_2 \log^2 p \right).
    \end{aligned}
    \end{equation}
    Here, the second and third lines follow from Cauchy-Schwartz inequality; the fourth line is a consequence of \eqref{eq:eventE} and $\tilde \phi(v^\top x_i)$ being a sub-Gaussian random variable ($\tilde \phi(z)$ is Lipschitz), and thus with bounded fourth moment (see Equation 2.11 in \cite{vershynin2018high}). This result holds for any $i, j$, and therefore implies
    \begin{equation}
    \begin{aligned}
        \opnorm{ \E_V \left[ \bar K \right] - \E_V [\tilde K] } &\leq \norm{\E_V \left[ \bar K \right] - \E_V [\tilde K]}_F \\
        &\leq pn \exp \left( - c_2 \log^2 p \right) \\
        &\leq \exp \left( - c_3 \log^2 p \right),
    \end{aligned}
    \end{equation}
    where in the last step we used $n = \bigO{p}$, which follows from Assumption \ref{ass:scalings}.
    This last equation, together with Lemma \ref{lemma:opnormtildeE}, gives
    \begin{equation}\label{eq:opnormbarEK}
        \opnorm{ \E_V \left[ \bar K \right] } = \bigO{p},
    \end{equation}
    with probability at least $1 - \exp \left( -c_4 \log^2 n \right)$. We will condition on this high probability event until the end of the proof.

    We now define the matrix $\bar H_k \in \R^{n \times n}$ as
    \begin{equation}
        \bar H_k = \bar \phi(X v_k) \bar \phi(X v_k)^\top,
    \end{equation}
    which implies
    \begin{equation}\label{eq:bernmat2}
        \bar \Phi \bar \Phi^\top - \E_V\left[ \bar \Phi \bar \Phi^\top \right] = \sum_{k = 1}^p \bar H_k - \E_V\left[ \bar H_k \right].
    \end{equation}
    Now, we have 
    \begin{equation}\label{eq:rKc11}
        \sup_{v_k} \opnorm{\bar H_k - \E_V[ \bar H_k]} \leq 2 \sup_{v_k} \opnorm{ \bar H_k} \leq 2 \sup_{v_k}  \norm{\bar \phi(X v_k)}_2^2 \leq 2 n \log^2 p,
    \end{equation}
    where the first step is true because of Jensen and triangle inequality, and the last step holds because, for every $i \in [n]$,
    \begin{equation}
        \left| \bar \phi(v_k^\top x_i) \right| \leq \tilde L \frac{\log p}{\tilde L} = \log p,
    \end{equation}
    given the definition in \eqref{eq:truncation}.
    Also, note that 
    \begin{equation}
        \E_V \left[ \bar H_k \right] = \E_V \left[ \bar \phi(X v_k) \bar \phi(X v_k)^\top \right] = \frac{1}{p} \E_V[\bar K],
    \end{equation}
    which gives, together with \eqref{eq:opnormbarEK},
    \begin{equation}\label{eq:opnormbarEK2}
        \opnorm{\E_V \left[ \bar H_k \right]} = \bigO{1}.
    \end{equation}
    Then, we can use \eqref{eq:rKc11} and \eqref{eq:opnormbarEK2}, together with the same argument in \eqref{eq:boundH2}, to get
    \begin{equation}
        \opnorm {\E_V \left[ \left( \bar H_k - \E_V [\bar H_k] \right)^2 \right]} =  \bigO{ n \log^2 p}.
    \end{equation}
    Thus, \eqref{eq:bernmat2} is the sum of $p$ independent (in the probability space of $V$), mean-0, $n \times n$ symmetric random matrices, such that $\opnorm{\bar H_k - \E_V[\bar H_k]} \leq 2 n \log^2 p$ almost surely for all $k$, and $\opnorm{\E_V \left[ \left( \bar H_k - \E_V [\bar H_k] \right)^2 \right]} \leq C n \log^2 p$, for some constant $C$. Then, by Theorem 5.4.1 of \citep{vershynin2018high}, we get
    \begin{equation}
        \P \left( \opnorm{\bar \Phi \bar \Phi^\top - \E_V\left[ \bar \Phi \bar \Phi^\top \right]}  \geq t \right) \leq 2n \exp \left( - \frac{t^2 / 2}{C n p \log^2 p + 2 t n \log^2 p  / 3}\right).
    \end{equation}
    Setting $t = p$, we obtain
    \begin{equation}\label{eq:bernsteinonbar}
        \P \left( \opnorm{\bar \Phi \bar \Phi^\top - \E_V\left[ \bar \Phi \bar \Phi^\top \right]}  \geq p \right) \leq 2 \exp \left( - c_5 \frac{p}{n \log^2 p} \right) \leq 2 \exp \left( - c_6 \log^2 n \right),
    \end{equation}
    where in the last step we used $n = \bigO{p / \log^4 p}$. 
    To conclude, we have that
    \begin{equation}\label{eq:equalhp}
    \begin{aligned}
        \P_V (\tilde \Phi \tilde \Phi^\top \neq \bar \Phi \bar \Phi^\top) &\leq \P_V \left( \max_{i \in [n], k \in [p]} |v_k^\top x_i| > \log p /\tilde L\right) \\
        &\leq 2 np \exp \left( -c_6 \log^2 p \right) \\
        &\leq 2 \exp \left( -c_7 \log^2 p \right),
    \end{aligned}
    \end{equation}
    where the second step holds as $v_k^\top x_i \sim \mathcal N(0, 1)$ in the probability space of $V$, and the last step holds because of the assumption $n = \bigO{p / \log^4 p}$. 
    Putting \eqref{eq:bernsteinonbar} and \eqref{eq:equalhp} together, we finally get
    \begin{equation}
        \opnorm{\tilde \Phi \tilde \Phi^\top} = \opnorm{\bar \Phi \bar \Phi^\top} \leq \opnorm{\E_V\left[ \bar \Phi \bar \Phi^\top \right]} + \opnorm{\bar \Phi \bar \Phi^\top - \E_V\left[ \bar \Phi \bar \Phi^\top \right]} = \bigO{p},
    \end{equation}
    with probability at least $1 - 2 \exp \left( -c \log^2 n \right)$ over $X$ and $V$, which gives the thesis. 
\end{proof}

\begin{lemma}\label{lemma:MuMu1}
    Let Assumptions \ref{ass:data} and \ref{ass:activation} hold, and $p = \omega(n)$. Then, we have
    \begin{equation}
        \left| \norm{\varphi(x_1)}_2^2 - M p \right| = \bigO{\sqrt p \log n},
    \end{equation}
    \begin{equation}
        \left| \norm{\tilde \varphi(x_1)}_2^2 - M_1 p \right| = \bigO{\sqrt p \log n},
    \end{equation}
    with $M$ and $M_1$ being two positive constants only depending on $\phi$, with probability at least $1 - 2 \exp(-c \log^2 n)$ over $V$, where $c$ is an absolute constant.
    
    Furthermore, we have that $M - M_1 = \mu_1^2 > 0$, which, conditioning on the previous result, implies
    \begin{equation}
        \norm{\varphi(x_1)}_2 - \norm{\tilde \varphi(x_1)}_2 = \Omega(\sqrt p).
    \end{equation}
\end{lemma}
\begin{proof}
    We have
    \begin{equation}
        \norm{\varphi(x_1)}_2^2 = \sum_{k = 1}^p \phi(v_k^\top x_1)^2,
    \end{equation}
    where we use the shorthand $v_k$ to indicate the $k$-th row of $V$. As $\phi$ is Lipschitz, $v_j \sim \mathcal N(0, I/d)$ and $\norm{x_1}_2 = \sqrt d$, we have that $\norm{\varphi(x_1)}_2^2$ is the sum of $p$ independent sub-exponential random variables, in the probability space of $V$.    
    Thus, by Bernstein inequality (see Theorem 2.8.1 in \cite{vershynin2018high}), we have
    \begin{equation}\label{eq:concentrationfeaturevector1new}
        \left| \norm{\varphi(x_1)}_2^2 - \E_V \left[ \norm{\varphi(x_1)}_2^2 \right] \right| = \bigO{\sqrt p \log n},
    \end{equation}
    with probability at least $1 - \exp(-c_1 \log^2 n)$, over the probability space of $V$. Exploiting the Hermite expansion of $\phi$, we get
    \begin{equation}\label{eq:concentrationfeaturevector2new}
        \E_V \left[ \norm{\varphi(x_1)}_2^2 \right] = p \, \E_{\rho \sim \mathcal(0, 1)} \left[ \phi(\rho)^2 \right] = p M,
    \end{equation}
    where we set $M = \sum_{l=0}^\infty \mu_l^2$.

    The same argument applied on $\tilde \phi$, which is also Lipschitz, guarantees
    \begin{equation}
        \left| \norm{\tilde \varphi(x_1)}_2^2 - M_1 p \right| = \bigO{\sqrt p \log n},
    \end{equation}
    where $M_1 = \sum_{l=3}^\infty \mu_l^2$, since the 0-th, 1st and 2nd Hermite coefficients of $\tilde \phi$ are 0. Thus, we readily get
    \begin{equation}
        M - M_1 = \mu_1^2.
    \end{equation}

    To conclude, it is sufficient to notice that, conditioning on the previous two high probability events, we have that $\norm{\varphi(x_1)}_2$ and $\norm{\tilde \varphi(x_1)}_2$ are both $\bigO{\sqrt p}$. Then,
    \begin{equation}
    \begin{aligned}
         \norm{\varphi(x_1)}_2 - \norm{\tilde \varphi(x_1)}_2 &= \frac{\norm{\varphi(x_1)}_2^2 - \norm{\tilde \varphi(x_1)}_2^2}{\norm{\varphi(x_1)}_2 + \norm{\tilde \varphi(x_1)}_2} \\
         &\geq \frac{(M - M_1) p - \left| \norm{\varphi(x_1)}_2^2 - p M \right| -  \left| \norm{\tilde \varphi(x_1)}_2^2 - M_1 p \right|}{\norm{\varphi(x_1)}_2 + \norm{\tilde \varphi(x_1)}_2} \\
         &= \Omega(\sqrt p),
    \end{aligned}
    \end{equation}
    where we use the triangle inequality twice in the second line and conclude using $\sqrt p = \omega(\log n)$.
\end{proof}

\begin{lemma}\label{lemma:facts}
    Let Assumptions \ref{ass:data} and \ref{ass:activation} hold, and let $n = o \left( p / \log^4 p \right)$, $n \log^3 n = o\left(d^{3/2}\right)$ and $n = \omega(d)$. 
    Then, we jointly have
    \begin{equation}
        \opnorm{X^+} = \bigO{\frac{1}{\sqrt n}},
    \end{equation}
    \begin{equation}
        \opnorm{V^+} = \bigO{\sqrt{\frac{d}{p}}},
    \end{equation}
    \begin{equation}
        \opnorm{\mu_1 V^\top - X^+ \Phi} = \bigO{\sqrt{\frac{p}{n}}},
    \end{equation}
    \begin{equation}
        \opnorm{\mu_1 V^\top \Phi^+} = \bigO{\frac{1}{\sqrt n}},
    \end{equation}
    with probability at least $1 - 2 \exp \left( - c \log^2 n \right)$ over $X$ and $V$, where $c$ is an absolute constant.
\end{lemma}
\begin{proof}
    The first two statements easily follow from Lemma \ref{lemma:evminX}, which gives
    \begin{equation}
        \opnorm{X^+} = \evmin{X^\top X}^{-1/2} = \bigO{\frac{1}{\sqrt n}},
    \end{equation}
    \begin{equation}
        \opnorm{V^+} = \evmin{V^\top V}^{-1/2} = \bigO{\sqrt{\frac{d}{p}}},
    \end{equation}
    with probability at least $1 - 2 \exp \left( -c_1 d \right) \geq 1 - 2 \exp \left( -c_2 \log^2 n \right)$ over $X$ and $V$. 
    For the third statement, we condition on this high probability event, and notice that
    \begin{equation}
        \Phi = \mu_1 X V^\top + \tilde \Phi,
    \end{equation}
    which gives
    \begin{equation}
        \opnorm{\mu_1 V^\top - X^+ \Phi} = \opnorm{X^+ \tilde \Phi} \leq \opnorm{X^+} \opnorm{\tilde \Phi} = \bigO{\sqrt{\frac{p}{n}}},
    \end{equation}
    where the last step holds because of Lemma \ref{lemma:tildetilde} with probability at least $1 - 2 \exp \left( - c_3 \log^2 n \right)$. For the fourth statement, we write
    \begin{equation}
    \begin{aligned}
        \opnorm{\mu_1 V^\top \Phi^+} &\leq \opnorm{\left(\mu_1 V^\top - X^+ \Phi \right) \Phi^+} + \opnorm{X^+ \Phi \Phi^+} \\
        &\leq \opnorm{\mu_1 V^\top - X^+ \Phi} \opnorm{\Phi^+} + \opnorm{X^+} \\
        &= \bigO{\sqrt{\frac{p}{n}} \, \sqrt{\frac{1}{p}} + \frac{1}{\sqrt n}} = \bigO{\frac{1}{\sqrt n}},
    \end{aligned}
    \end{equation}
    where the second step follows from
    \begin{equation}
        \opnorm{\Phi^+} = \evmin{K}^{-1/2} = \bigO{\frac{1}{\sqrt p}},
    \end{equation}
    which holds with probability $1 - 2 \exp \left( - c_4 \log^2 n \right)$ because of Lemma \ref{lemma:hyperK}. This final result provides the thesis.
\end{proof}

\begin{lemma}\label{lemma:PPV}
    Let Assumptions \ref{ass:data} and \ref{ass:activation} hold, and let $n = o \left( p / \log^4 p \right)$, $n \log^3 n =o\left(d^{3/2}\right)$ and $n = \omega(d)$. Let $P_\Lambda \in \R^{p \times p}$ be the projector on the span of the $d$ eigenvectors corresponding to the $d$ highest eigenvalues of $\Phi^\top \Phi$, and $P_V$ be the projector on the span of the columns of $V$. Then, we have that
    \begin{equation}
        \opnorm{P_{\Lambda}^\perp V} = \bigO{\sqrt{\frac{p}{n}}},
    \end{equation}
    and
    \begin{equation}
        \opnorm{P_{\Lambda}^\perp P_V} = \bigO{\sqrt{\frac{d}{n}}},
    \end{equation}
    jointly hold with probability at least $1 - 2 \exp \left(-c \log^2 n \right)$ over $X$ and $V$, where $c$ is an absolute constant.
\end{lemma}
\begin{proof}
    Let $u \in \R^p$ be a vector such that $\norm{u}_2 = 1$. Then, we have
    \begin{equation}
      \norm{  \Phi P_{\Lambda}^\perp u}_2 \leq \sqrt{\lambda_{d+1}(\Phi^\top \Phi)} = \sqrt{\lambda_{d+1}(K)} = \bigO{\sqrt p},
    \end{equation}
    where the first inequality follows by the definition of $P_\Lambda$, and $\lambda_{d+1}(K)$ ($\lambda_{d+1}(\Phi^\top \Phi)$) is the $(d+1)$-th eigenvalue of $K$ ($\Phi^\top \Phi$) when sorting them in non-increasing order. Then, the last equality holds with probability at least $1 - 2 \exp \left( -c_1 \log^2 n\right)$ over $X$ and $V$, because of Lemma \ref{lemma:hyperK}. 
    Then, conditioning on the previous high probability event, since $\Phi = \tilde \Phi + \mu_1 X V^\top$, we can write
    \begin{equation}\label{eq:startingfromevmin}
    \begin{aligned}
        \mu_1 \sqrt{\evmin{X^\top X}} \norm{V^\top P_\Lambda^\perp u}_2 &\leq \norm{\mu_1 X V^\top P_\Lambda^\perp u}_2 \\
        &\leq \norm{\tilde \Phi P_\Lambda^\perp u}_2 + \norm{\Phi P_\Lambda^\perp u}_2 \\
        &\leq \opnorm{\tilde \Phi} \opnorm{P_\Lambda^\perp} \norm{u}_2 + \norm{\Phi P_\Lambda^\perp u}_2 = \bigO{\sqrt p}, 
    \end{aligned}
    \end{equation}
    where the last step is a consequence of Lemma \ref{lemma:tildetilde}, and it holds with probability $1 - 2 \exp \left(-c_2 \log^2 n \right)$.
    By Lemma \ref{lemma:evminX}, we have that $\evmin{X^\top X} = \Theta(n)$ with probability at least $1 - 2 \exp \left(-c_3 d \right)$. Conditioning on such high probability events, \eqref{eq:startingfromevmin} gives
    \begin{equation}
        \norm{V^\top P_\Lambda^\perp u}_2 = \bigO{\sqrt{\frac{p}{n}}},
    \end{equation}
    which provides the first part of the thesis, since it uniformly holds for every $u$ (as we did not condition on any event dependent on $u$ itself).
    To conclude, since we have
    \begin{equation}
        P_V = (V^+)^\top V^\top,
    \end{equation}
    we can use Lemma \ref{lemma:facts} to write
    \begin{equation}
        \norm{P_V P_\Lambda^\perp u}_2 = \norm{(V^+)^\top V^\top P_\Lambda^\perp u}_2 \leq \opnorm{V^+} \norm{V^\top P_\Lambda^\perp u}_2 = \bigO{\sqrt{\frac{d}{n}}},
    \end{equation}
    with probability at least $1 - 2 \exp \left(-c_4 d \right)$ over $V$. Since again this result holds uniformly on all $\norm{u}_2 = 1$ and $\opnorm{P_\Lambda^\perp P_V} = \opnorm{P_V  P_\Lambda^\perp}$, the thesis readily follows.
\end{proof}

\begin{lemma}\label{lemma:Pvarphi}
    Let Assumptions \ref{ass:data} and \ref{ass:activation} hold, and let $n = o \left( p / \log^4 p \right)$, $n \log^3 n = o \left(d^{3/2}\right)$ and $n = \omega(d)$. Let $P_\Lambda \in \R^{p \times p}$ be the projector on the span of the $d$ eigenvectors corresponding to the $d$ highest eigenvalues of $\Phi^\top \Phi$. Then, we have
    \begin{equation}
        \norm{P_\Lambda \varphi(x_1)}_2 = \Omega(\sqrt p),
    \end{equation}
    with probability at least $1 - 2 \exp \left(-c \log^2 n \right)$ over $X$ and $V$, where $c$ is an absolute constant.
\end{lemma}
\begin{proof}
    Multiple applications of the triangle inequality yield
    \begin{equation}
    \begin{aligned}
        \norm{P_\Lambda \varphi(x_1)}_2 &\geq \norm{\varphi(x_1)}_2 - \norm{P_\Lambda^\perp \varphi(x_1)}_2 \\
        &= \norm{\varphi(x_1)}_2 - \norm{P_\Lambda^\perp \left( \mu_1 V x_1 + \tilde \varphi(x_1) \right)}_2 \\
        &\geq \norm{\varphi(x_1)}_2 - \mu_1 \norm{P_\Lambda^\perp V x_1}_2 - \norm{P_\Lambda^\perp \tilde \varphi(x_1)}_2 \\
        &\geq \norm{\varphi(x_1)}_2 - \mu_1 \opnorm{P_\Lambda^\perp V} \norm{x_1}_2 - \norm{\tilde \varphi(x_1)}_2.
    \end{aligned}
    \end{equation}

    Let us condition on the results of Lemma \ref{lemma:PPV} and Lemma \ref{lemma:MuMu1}. Then, we have that
    \begin{equation}
        \norm{\varphi(x_1)}_2 - \norm{\tilde \varphi(x_1)}_2 = \Omega(\sqrt p),
    \end{equation}
    and
    \begin{equation}
        \opnorm{P_\Lambda^\perp V} \norm{x_1}_2 = \bigO{\sqrt{\frac{p}{n}}} \sqrt d = o(\sqrt p),
    \end{equation}
    which readily gives the thesis.
\end{proof}

\begin{lemma}\label{lemma:rectbern}
    We have that
    \begin{equation}
        \opnorm{\tilde \Phi V} = \bigO{\sqrt{\frac{pn}{d}} \log n},
    \end{equation}
    with probability at least $1 - 2 \exp \left( -c \log^2 n \right)$ over $X$ and $V$, where $c$ is an absolute constant.
\end{lemma}
\begin{proof}

    Note that $\tilde \phi$ is Lipschitz (since $\phi$ is Lipschitz by Assumption \ref{ass:activation}). During the proof, we condition on the event $\opnorm{X} = \bigO{\sqrt n}$, which is true with probability at least $1 - 2 \exp \left( -c_1 d \right)$ by Lemma \ref{lemma:evminX}, and we use the shorthand $v \in \R^{d}$ to denote a random vector such that $\sqrt{d} \, v$ is a standard Gaussian vector, \emph{i.e.}, it has the same distribution as the rows of $V$.
    This implies
    \begin{equation}
        \mathbb E_v \left[ \norm{\tilde \phi \left( X v \right)}_2 \right] = \bigO{\sqrt n}, \qquad \subGnorm{\norm{\tilde \phi \left( X v \right)}_2 - \mathbb E_v \left[ \norm{\tilde \phi \left( X v \right)}_2 \right]} = \bigO{\sqrt{\frac{n}{d}}},
    \end{equation}
    and
    \begin{equation}
        \mathbb E_v \left[ \norm{v}_2 \right] = \bigO{1}, \qquad \subGnorm{\norm{v}_2 - \mathbb E_v \left[\norm{v}_2 \right]} = \bigO{\frac{1}{\sqrt d}},
    \end{equation}
    where both sub-Gaussian norms are meant on the probability space of $v$. Here, the very first equation follows from the discussion in Lemma C.3 in \cite{bombari2022memorization}, and the upper bounds on the sub-Gaussian norms follow from the Lipschitz concentration property of $\sqrt d v$.
    Then, there exists an absolute constant $C_1$ such that we jointly have
    \begin{equation}
        \norm{\tilde \phi \left( X v \right)}_2 \leq C_1 \sqrt n, \qquad \norm{v}_2 \leq C_1,
    \end{equation}
    with probability at least $1 - 2 \exp \left( -c_2 d \right)$ over $v$.

    Let $E_k$ be the indicator defined on the high probability event above with respect to the random variable $v_k := V_{:k}$ (the $k$-th row of $V$), \emph{i.e.},
    \begin{equation}\label{eq:ek}
        E_{k} := \mathbf 1 \left( \norm{v_k}_2 \leq C_1 \textup{ and } \; \norm{\tilde \phi \left( X v_k \right)}_2 \leq C_1 \sqrt n \right),
    \end{equation}
    and we define $E \in \R^{p \times p}$ as the diagonal matrix containing $E_k$ in its $k$-th entry. Notice that we have $\opnorm{I - E} = 0$ with probability at least $1 - 2 p \exp \left( -c_2 d \right)$, and $\E_V \left[ \opnorm{I - E }\right] \leq 2p \exp \left( -c_2 d \right)$.

    Thus, we have
    \begin{equation}\label{eq:polyvsexp}
    \begin{aligned}
        \opnorm{\E_V \left[ \tilde \Phi \left( I - E \right) V \right]} &\leq \E_V \left[ \opnorm{\tilde \Phi} \opnorm{1 - E} \opnorm{V} \right] \\
        &\leq \E_V \left[ \opnorm{\tilde \Phi}^2  \opnorm{V}^2 \right]^{1 / 2}  \E_V \left[ \opnorm{I - E}^2 \right]^{1/2} \\
        &\leq \E_V \left[ \opnorm{\tilde \Phi}^4 \right]^{1/4} \E_V \left[ \opnorm{V}^4 \right]^{1 / 4} \left( 2p \exp \left( -c_2 d \right) \right) ^{1/2} \\
        &\leq \E_V \left[ \norm{\tilde \Phi}^4_F \right]^{1/4} \E_V \left[ \norm{V}^4_F \right]^{1 / 4} \left( 2p \exp \left( -c_2 d \right) \right) ^{1/2} \\
        & = o(1),
    \end{aligned}
    \end{equation}
    where the last step holds because of our initial conditioning on $X$: the first two terms are the sum of finite powers of sub-Gaussian random variables (the entries of $\tilde \Phi$ and $V$), and thus (see Proposition 2.5.2 in \cite{vershynin2018high}) the first two factors in the third line of the previous equation will be $\bigO{p^\alpha}$ for some finite $\alpha$, which gives the last line due to Assumption \ref{ass:scalings}.

    Exploiting the Hermite expansion of $\tilde \phi$ we can write
    \begin{equation}
        \left[ \E_V \left[ \tilde \Phi V \right] \right]_{ij} = p \left[ \E_v \left[ \tilde \phi (Xv) v^\top \right] \right]_{ij} = \frac{p}{\sqrt{d}} \E_v \left[ \tilde\phi \left( x_i^\top v \right)  \left( e_j^\top \left( \sqrt{d} v\right) \right) \right] = 0,
    \end{equation}
    where the last step holds since the first Hermite coefficient of $\tilde \phi$ is 0. Thus, the application of the triangle inequality to this last equation and \eqref{eq:polyvsexp} gives
    \begin{equation}\label{eq:forlater}
        \opnorm{\E_V \left[ \tilde \Phi E V \right] } \leq \opnorm{\E_V \left[ \tilde \Phi V \right] } + \opnorm{\E_V \left[ \tilde \Phi \left( I - E \right) V \right]} = o(1), 
    \end{equation}
    with probability at least $1 - 2 \exp \left( -c_3 \log^2 d\right)$ over $X$.

    Let's now look at
    \begin{equation}\label{eq:sumforB}
        \tilde \Phi E V  - \E_V \left[ \tilde \Phi E V \right] = \sum_{k = 1}^p \tilde \phi (X v_k) E_k v_k^\top - \E_{v_k} \left[ \tilde \phi (X v_k) E_k v_k^\top \right] =: \sum_{k = 1}^p W_k,
    \end{equation}
    where we defined the shorthand $W_k = \tilde \phi (X v_k) E_k v_k^\top - \E_{v_k} \left[ \tilde \phi (X v_k) E_k v_k^\top \right]$. \eqref{eq:sumforB} is the sum of $p$ i.i.d.\ mean-0 random matrices $W_k$ (in the probability space of $V$), such that
    \begin{equation}
    \begin{aligned}
        \sup_{v_k} \opnorm{\tilde \phi (X v_k) E_k v_k^\top - \E_{v_k} \left[ \tilde \phi (X v_k) E_k v_k^\top \right]} &\leq 2 \sup_{v_k} \opnorm{\tilde \phi (X v_k) E_k v_k^\top} \\
        &= 2 \sup_{v_k} \left( \norm{\tilde \phi (X v_k)}_2 \norm{v_k}_2 E_k \right) \\
        &\leq 2C_1^2 \sqrt n,
    \end{aligned}
    \end{equation}
    because of \eqref{eq:ek}. Then, by matrix Bernstein's inequality for rectangular matrices (see Exercise 5.4.15 in \cite{vershynin2018high}), we have that
    \begin{equation}\label{eq:bernstein}
        \P_V \left( \opnorm{\tilde \Phi E V  - \E_V \left[ \tilde \Phi E V \right]} \geq t  \right) \leq \left( n + d \right) \exp \left( - \frac{t^2 / 2}{\sigma^2 +2C_1^2 \sqrt n t / 3 } \right),
    \end{equation}
    where $\sigma^2$ is defined as
    \begin{equation}\label{eq:sigmamax}
        \sigma^2 = p \max\left( \opnorm{\E_{v_k} \left[ W_k W_k^\top \right]}, \opnorm{\E_{v_k} \left[ W_k^\top W_k \right]} \right).
    \end{equation}
    For every matrix $A$, we have $\E \left[ \left(A - \E [A] \right) \left(A - \E [A] \right)^\top \right] = \E \left[ AA^\top \right] - \E [A] \E[A]^\top \preceq \E \left[ AA^\top \right]$. Thus,
    \begin{equation}
    \begin{aligned}
         \opnorm{\E_{v_k} \left[ W_k W_k^\top \right]} &\leq \opnorm{\E_{v_k} \left[ \tilde \phi (X v_k) E_k v_k^\top v_k E_k \tilde \phi (X v_k)^\top \right]} \\
         &\leq \opnorm{\E_{v_k} \left[ \tilde \phi (X v_k) \tilde \phi (X v_k)^\top \right]} \sup_{v_k} \left( E_k \norm{v_k}_2^2 \right) \\
         &\leq C_1^2 \opnorm{\E_{v_k} \left[ \tilde \phi (X v_k) \tilde \phi (X v_k)^\top \right]} \\
         &= \frac{C_1^2}{p} \opnorm{\E_V \left[ \tilde K \right]} \\
         &= \bigO{1},
    \end{aligned}
    \end{equation}
    where the last step is a direct consequence of Lemma \ref{lemma:opnormtildeE}, and holds with probability at least $1 - 2 \exp \left( -c_4 \log^2 n\right)$ over $X$.
    For the other argument in the $\max$ in \eqref{eq:sigmamax}, we similarly have
    \begin{equation}
    \begin{aligned}
         \opnorm{\E \left[ W_k^\top W_k \right]} &\leq \opnorm{\E_{v_k} \left[ v_k E_k \tilde \phi (X v_k)^\top \tilde \phi (X v_k) E_k v_k^\top  \right]} \\
         &\leq \opnorm{\E_{v_k} \left[ v_k v_k^\top \right]} \sup_{v_k} \left( E_k \norm{\tilde \phi (X v_k)}_2^2 \right) \\
         &\leq \frac{1}{d} \, C_1^2 n  \\
         &= \bigO{\frac{n}{d}}.
    \end{aligned}
    \end{equation}
    Then, plugging these last two equations in \eqref{eq:bernstein} we get
    \begin{equation}\label{eq:cnew}
    \begin{aligned}
        \P_V \left( \opnorm{\tilde \Phi E V  - \E_V \left[ \tilde \Phi E V \right]} \geq \sqrt{\frac{pn}{d}} \log n  \right) &\leq \left( n + d \right) \exp \left( - \frac{(pn/d) \log^2 n / 2}{C_2 (pn/d) + 2C_1^2 \sqrt{np}\sqrt{n/d} \log n / 3 } \right) \\
        &\leq 2 \exp \left( -c_5 \log^2 n \right),
    \end{aligned}
    \end{equation}
    where we used Assumption \ref{ass:scalings}. Then, applying the triangle inequality and using \eqref{eq:forlater} and \eqref{eq:cnew}, we get
    \begin{equation}
        \opnorm{\tilde \Phi E V} = \bigO{\sqrt{\frac{pn}{d}} \log n},
    \end{equation}
    with probability at least $1 - 2 \exp \left( -c_6 \log^2 n \right)$ over $X, V$.
    To conclude, since $E = I$ with probability at least $1 - 2 p \exp \left( -c_2 d \right)$, using Assumption \ref{ass:scalings} gives the desired result.
\end{proof}

\begin{lemma}\label{lemma:earlystoppingtermnew}
    We have that
    \begin{equation}
        \opnorm{V^\top  P_\Lambda^\perp \Phi^+} = \bigO{\frac{\sqrt d}{n} +  \frac{\sqrt n \log^3 d}{d^{3/2}}},
    \end{equation}
    with probability at least $1 - 2 \exp \left( -c \log^2 n \right)$ over $X$ and $V$, where $c$ is an absolute constant.
\end{lemma}
\begin{proof}
    Let ${\tilde \mu}^2 := \sum_{l = 3}^\infty \mu_l^2$, where $\mu_l$ denotes the $l$-th Hermite coefficient of $\phi$, and let $E \in \R^{n \times n}$ be the matrix defined as
    \begin{equation}
        E = K - p \left( \mu_1^2 \frac{XX^\top}{d} + {\tilde \mu}^2 I \right).
    \end{equation}
    Note that
    \begin{equation}\label{eq:opnormEdiff}
    \begin{aligned}
        \opnorm{E K^{-1}} &\leq \opnorm{ K - \E_{V} \left[ K \right] } \opnorm{ K^{-1}} + \opnorm{\E_{V} \left[ K \right] - p \left( \mu_1^2 \frac{XX^\top}{d} + {\tilde \mu}^2 I \right)} \opnorm{K^{-1}} \\
        &= \opnorm{\E_V \left[ K\right]^{-1/2} \left(K - \E_{V} \left[ K \right] \right) \E_V \left[ K\right]^{-1/2}} \opnorm{ \E_V \left[ K\right]} \opnorm{ K^{-1}} + \bigO{p \frac{n \log^3 n}{d^{3/2}} \, \frac{1}{p}} \\
        &=\bigO{\sqrt{\frac{n}{p}} \log n \log p \, \frac{np}{d} \, \frac{1}{p}} + \bigO{\frac{n \log^3 n}{d^{3/2}}} \\
        &= \bigO{\frac{\sqrt n}{d} \log^2 n + \frac{n \log^3 n}{d^{3/2}}} \\
        &= \bigO{\frac{n \log^3 n}{d^{3/2}}},
    \end{aligned}
    \end{equation}
    where in the second line we used \eqref{eq:difftildeM2} and $\evmin{K} = \Omega(p)$, which holds by Lemma \ref{lemma:hyperK}; in the third line we used Lemma \ref{lemma:renormalizedKconc}, $\evmin{K} = \Omega(p)$ and $\opnorm{\E_V \left[ K\right]} = \bigO{np/d}$, which follows from Weyl inequality applied to \eqref{eq:XXhyper} and \eqref{eq:perturbM2} (conditioning on $\opnorm{X} = \bigO{\sqrt n}$, given by Lemma \ref{lemma:evminX}); in the fourth line we used $p = \Omega(n^2)$, and the last step holds due to Assumption \ref{ass:scalings}. The full equation as a whole holds with probability at least $1 - 2 \exp \left( -c_1 \log^2 n \right)$ over $X$ and $V$.
    
    By the Woodbury matrix identity (or Hua's identity), we have
    \begin{equation}\label{eq:woodbury}
    \begin{aligned}
        K^{-1} &= \left( p \left( \mu_1^2 \frac{XX^\top}{d} + {\tilde \mu}^2 I \right) + E \right)^{-1} \\
        &= \left(  \mu_1^2  p \frac{XX^\top}{d} + {\tilde \mu}^2 p  I \right)^{-1} - \left(  \mu_1^2  p \frac{XX^\top}{d} + {\tilde \mu}^2 p  I \right)^{-1} E  K^{-1}.
    \end{aligned}
    \end{equation}
    Then, since we have
    \begin{equation}\label{eq:Xstuffsmall}
        \opnorm{\frac{p X^\top}{d} \left(  \mu_1^2  p \frac{XX^\top}{d} + {\tilde \mu}^2 p  I \right)^{-1}} \leq \frac{1}{\mu_1^2 \sqrt{\evmin{X^\top X}}} = \bigO{\frac{1}{\sqrt n}},
    \end{equation}
    with probability at least $1 - 2 \exp \left( -c_2 \log^2 n \right)$ over $X$ due to Lemma \ref{lemma:evminX}, \eqref{eq:woodbury} allows us to write
    \begin{equation}\label{eq:likeXstuff}
    \begin{aligned}
        \opnorm{\frac{p X^\top}{d} \left( K^{-1} - \left(  \mu_1^2  p \frac{XX^\top}{d} + {\tilde \mu}^2 p  I \right)^{-1} \right)} &= \opnorm{\frac{p X^\top}{d} \left(  \mu_1^2  p \frac{XX^\top}{d} + {\tilde \mu}^2 p  I \right)^{-1} E  K^{-1}} \\
        &= \bigO{\frac{\sqrt n \log^3 n}{d^{3/2}}},
    \end{aligned}
    \end{equation}
    where we used \eqref{eq:opnormEdiff} and \eqref{eq:Xstuffsmall} in the last step, which then holds with probability at least $1 - 2 \exp \left( -c_3 \log^2 n \right)$ over $X$ and $V$. Notice that, with this same probability, the application of the triangle inequality to \eqref{eq:Xstuffsmall} and \eqref{eq:likeXstuff}, together with Assumption \ref{ass:scalings}, gives
    \begin{equation}\label{eq:XK-1}
        \opnorm{\frac{p X^\top}{d} K^{-1}} = \bigO{\frac{1}{\sqrt n}}.
    \end{equation}

    Recalling that $\Phi = \mu_1 X V^\top + \tilde \Phi$, another triangle inequality yields
    \begin{equation}\label{eq:usingbernst}
    \begin{aligned}
        \opnorm{V^\top \Phi^\top K^{-1} - \frac{\mu_1 p X^\top}{d} K^{-1}} &\leq \opnorm{V^\top \tilde \Phi^\top K^{-1}} + \mu_1 \opnorm{\frac{d}{p} V^\top V - I} \opnorm{\frac{p X^\top}{d} K^{-1}} \\
        &= \bigO{\sqrt{\frac{pn}{d}} \log n \, \frac{1}{p} + \sqrt{\frac{d}{p}} \, \frac{1}{\sqrt n}} \\
        &= \bigO{\frac{\log n}{\sqrt{dn}}},
    \end{aligned}
    \end{equation}
    where the second line holds with probability at least $1 - 2 \exp \left( -c_4 \log^2 n \right)$ over $X$ and $V$ because of Lemma \ref{lemma:rectbern}, Theorem 4.6.1 in \cite{vershynin2018high} and \eqref{eq:XK-1},  and the last step is a consequence of Assumption \ref{ass:scalings}.

    We now define $\varrho_\Lambda \in \R^{n \times n}$ as the projector on the span of the eigenvectors associated with the $d$ largest eigenvalues of $K$. This implies
    $P_\Lambda \Phi^\top = \Phi^\top \varrho_\Lambda$, and therefore $P_\Lambda^\perp \Phi^+ = \Phi^\top K^{-1} \varrho_\Lambda^\perp$. Hence, we have
    \begin{equation}
        \opnorm{\mu_1 V X^\top \varrho^\perp_\Lambda} \leq \opnorm{\Phi^\top \varrho_\Lambda^\perp} + \opnorm{\tilde \Phi^\top \varrho_\Lambda^\perp} \leq \sqrt{\lambda_{d+1}(K)} + \opnorm{\tilde \Phi} \opnorm{\varrho_\Lambda^\perp} = \bigO{\sqrt p},
    \end{equation}
    where the last step is a consequence of Lemma \ref{lemma:hyperK} and Lemma \ref{lemma:tildetilde}, and holds with probability at least $1 - 2 \exp \left( -c_5 \log^2 n \right)$ over $X$ and $V$. Then, conditioning on this high probability event and on the event $\evmin{V^\top V} = \Omega(p / d)$, which holds with probability at least $1 - 2 \exp \left( -c_6 d \right)$ over $V$ due to Lemma \ref{lemma:evminX}, we have
    \begin{equation}\label{eq:pperpX}
        \opnorm{X^\top \varrho^\perp_\Lambda} \leq \frac{1}{\mu_1 \sqrt{\evmin{V^\top V}}} \opnorm{\mu_1 V X^\top \varrho^\perp_\Lambda} = \bigO{\sqrt d}.
    \end{equation}

    Thus, with probability at least $1 - 2 \exp \left( -c_7 \log^2 n \right)$ over $X$ and $V$, we have
    \begin{equation}
    \begin{aligned}
        \opnorm{V^\top P^\perp_\Lambda \Phi^+} &\leq \opnorm{\frac{\mu_1 p X^\top}{d} K^{-1} \varrho^\perp_\Lambda} + \opnorm{V^\top \Phi^\top K^{-1} - \frac{\mu_1 p X^\top}{d} K^{-1}} \opnorm{\varrho^\perp_\Lambda} \\
        &\leq \opnorm{\frac{\mu_1 p X^\top}{d} \left(  \mu_1^2  p \frac{XX^\top}{d} + {\tilde \mu}^2 p  I \right)^{-1} \varrho^\perp_\Lambda}  \\
        & \hspace{1cm} + \opnorm{\frac{\mu_1 p X^\top}{d} \left( K^{-1} - \left(  \mu_1^2  p \frac{XX^\top}{d} + {\tilde \mu}^2 p  I \right)^{-1} \right)} \opnorm{\varrho^\perp_\Lambda} + \bigO{\frac{\log n}{\sqrt{dn}}} \\
        &= \opnorm{\left(  \mu_1 X^\top X + {\tilde \mu}^2 d  I / \mu_1 \right)^{-1} X^\top \varrho^\perp_\Lambda} + \bigO{\frac{\sqrt n \log^3 n}{d^{3/2}}} + \bigO{\frac{\log n}{\sqrt{dn}}} \\
        &\leq \opnorm{\left(  \mu_1 X^\top X + {\tilde \mu}^2 d  I / \mu_1 \right)^{-1}} \opnorm{X^\top \varrho^\perp_\Lambda} + \bigO{\frac{\sqrt n \log^3 n}{d^{3/2}}} \\
        &= \bigO{\frac{1}{n} \, \sqrt d} + \bigO{\frac{\sqrt n \log^3 n}{d^{3/2}}} \\
        &= \bigO{\frac{\sqrt d}{n} +  \frac{\sqrt n \log^3 n}{d^{3/2}}},
    \end{aligned}
    \end{equation}
    where the second step holds due to \eqref{eq:usingbernst}, the third step due to \eqref{eq:likeXstuff}, the fourth due to Assumption \ref{ass:scalings}, and the fifth due to Lemma \ref{lemma:evminX} and \eqref{eq:pperpX}. This, together with Assumption \ref{ass:scalings}, provides the desired result.
\end{proof}

\begin{lemma}\label{lemma:Eopnormsmall}
Let Assumptions \ref{ass:data} and \ref{ass:activation} hold, and let $p = \omega(d)$ and $\log p = \Theta(\log n) = \Theta(\log d)$. 
Let $x \sim \mathcal P_X$. Then, we have
\begin{equation}
    \opnorm{\E_{x} \left[ \tilde \varphi(x) \tilde \varphi(x)^\top \right]} = \bigO{\log^4 n + \frac{p \log^3 d}{d^{3/2}}},
\end{equation}
with probability at least $1 - 2 \exp \left( -c \log^2 n \right)$ over $V$, where $c$ is an absolute constant.
\end{lemma}
\begin{proof}
    Let $N$ be a positive natural number that will be defined later, and let $\tilde \Phi_N \in \R^{N \times p}$ be a matrix containing $\tilde \varphi(\hat x_i)$ in its $i$-th row, where every $\{ \hat x_i \}_{i = 1}^N$ is sampled independently from $\mathcal P_X$. Importantly, $\hat x_i$ is different from $x_i$ (as they are defined as auxiliary random variables only useful for the purposes of this proof), but $\tilde \varphi(\hat x_i) = \tilde \phi(V \hat x_i)$ is defined with the same random features $V$ in \eqref{eq:rfmodelintro}.
    Set
    \begin{equation}\label{eq:PhiNscalings}
        n' = \min \left( \left \lfloor \frac{p}{\log^4 p} \right \rfloor, \left \lfloor \frac{d^{3/2}}{\log^3 d} \right \rfloor \right) , \qquad N = p^2 n',
    \end{equation}
    where $n'$ is a positive integer (as $p$ is large enough). Note that this definition guarantees that the triple $(n', d, p)$ satisfies the scalings
    \begin{equation}
        n' = \bigO{\frac{p}{\log^4 p}},  \qquad  n' \log^3 n' = \bigO{d^{3/2}},
    \end{equation}
    which provide the sufficient hypotheses to apply Lemma \ref{lemma:tildetilde} on a $n' \times p$ block of $\tilde \Phi_N$. 
    
    Then, $\tilde \Phi_N$ can be seen as the vertical stacking of $p^2$ matrices with size $n' \times p$. All these matrices are independent copies of each other (in the probability space of the $\hat x_i$-s), and each of these has operator norm $\bigO{\sqrt p}$ by Lemma \ref{lemma:tildetilde}, with probability at least $1 - 2 \exp \left( -c_1 \log^2 n' \right)$. Thus, performing a union bound over these $p^2$ matrices, we get
    \begin{equation}\label{eq:PhiNopnormstack}
        \opnorm{\tilde \Phi_N^\top \tilde \Phi_N} = \bigO{p^2 \, p} = \bigO{\frac{Np}{n'}} = \bigO{N \log^4 p + \frac{N p \log^3 d}{d^{3/2}} },
    \end{equation}
    with probability at least
    \begin{equation}
        1 - 2 p^2 \exp \left( -c_1 \log^2 n' \right) \geq 1 - 2 p^2 \exp \left( -c_2  \min \left( \log^2 p, \log^2 d \right) \right)
    \end{equation}
    over $V$ and $\{ \hat x_i \}_{i = 1}^N$.
    \simonesolved{Note that the rows of $\tilde \Phi_N$ are identically distributed and such that
    \begin{equation}
        \sup_{\hat x_i} \norm{\tilde \varphi(\hat x_i)}_2 \leq \norm{\tilde \phi(\mathbf 0)}_2 + \tilde L \sup_{\hat x_i} \norm{V \hat x_i}_2 \leq \norm{\tilde \phi(\mathbf 0)}_2 + \tilde L \opnorm{V} \sup_{\hat x_i} \norm{\hat x_i}_2 = \bigO{\sqrt p},
    \end{equation}
    where we denote with $\mathbf 0 \in \R^p$ a vector of zeros, $\tilde L$ is the Lipschitz constant of $\tilde \phi$, and the last step follows from the bound on $\opnorm{V}$ given by Lemma \ref{lemma:evminX}, which holds with probability at least $1 - 2 \exp \left( -c_3 d \right)$ over $V$ (high probability event over which we will condition until the end of the proof). This readily gives, for some constant $C_1$,
    \begin{equation}\label{eq:subGrowN}
        \subGnorm{\tilde \varphi(\hat x_i)} \leq C_1 \sup_{\hat x_i} \norm{\tilde \varphi(\hat x_i)}_2 = \bigO{\sqrt p}.
    \end{equation}}
    We remark that the sub-Gaussian norm in \eqref{eq:subGrowN} is intended in the probability space of $\hat x_i$, and that the bound holds jointly for every $i \in [N]$. Then, there exists a sufficiently small absolute constant $C_2$ such that $C_2 \tilde \Phi_N / \sqrt p$ is a matrix with independent sub-Gaussian rows, with unit sub-Gaussian norm. Then, by Theorem 5.39 in \cite{vershrandmat} (see their Remark 5.40 and Equation (5.25)), we have that
    \begin{equation}\label{eq:PhiNconc}
        \frac{C_2^2}{p}  \opnorm{\frac{\tilde \Phi_N^\top \tilde \Phi_N}{N} - \E_{x} \left[ \tilde \varphi(x) \tilde \varphi(x)^\top \right]} = \bigO{\sqrt{\frac{p}{N}}},
    \end{equation}
    with probability at least $1 - 2 \exp \left( -c_4 p \right)$ over $\{ \hat x_i \}_{i = 1}^N$.

    Then, we have
    \begin{equation}
        \begin{aligned}
            \opnorm{\E_{x} \left[ \tilde \varphi(x) \tilde \varphi(x)^\top \right]} &\leq \opnorm{\frac{\tilde \Phi_N^\top \tilde \Phi_N}{N} - \E_{x} \left[ \tilde \varphi(x) \tilde \varphi(x)^\top \right]} + \frac{\opnorm{\tilde \Phi_N^\top \tilde \Phi_N}}{N} \\
            &= \bigO{p \sqrt{\frac{p}{N}}} + \bigO{\log^4 p + \frac{p \log^3 d}{d^{3/2}}} \\
            &= \bigO{\sqrt p \, \sqrt{ \frac{\log^4 p}{p}} + \sqrt p \, \sqrt{\frac{\log^3 d}{d^{3/2}}}} + \bigO{\log^4 p + \frac{p \log^3 d}{d^{3/2}}} \\
            &= \bigO{\log^4 p + \frac{p \log^3 d}{d^{3/2}}},
        \end{aligned}
    \end{equation}
    where the first step follows from the triangle inequality, the second step is a consequence of \eqref{eq:PhiNconc} and \eqref{eq:PhiNopnormstack}, and the third step is a consequence of \eqref{eq:PhiNscalings}. 
    Taking the intersection between the high probability events in \eqref{eq:PhiNopnormstack} and \eqref{eq:PhiNconc}, we have 
    \begin{equation}
        \opnorm{\E_{x} \left[ \tilde \varphi(x) \tilde \varphi(x)^\top \right]} = \bigO{\log^4 n + \frac{p \log^3 d}{d^{3/2}}},
    \end{equation}
    with probability at least $1 - 2 p^2 \exp \left( -c_5  \min \left( \log^2 p, \log^2 d \right) \right) \geq 1 - 2 \exp \left( -c_6   \log^2 n \right)$ 
    over $\{ \hat x_i \}_{i = 1}^N$ and $V$.
    Note that the LHS of the previous equation does not depend on $\{ \hat x_i \}_{i = 1}^N$, which were introduced as auxiliary random variables. Thus, the high probability bound holds restricted to the probability space of $V$, and the desired result follows.
\end{proof}

\begin{lemma}\label{lemma:outputorder1}
    Let Assumptions \ref{ass:data} and \ref{ass:activation} hold, and let $n = o \left( p / \log^4 p \right)$, $n \log^3 n = o\left(d^{3/2}\right)$, $n = \omega(d)$ and $\log n = \Theta(\log p)$.
    Then, we have
    \begin{equation}\label{eq:bdex}
        \E_{x} \left[  \left(  \varphi(x)^\top \theta^* \right)^2  \right] = \bigO{1},
    \end{equation}
    with probability at least $1 - 2 \exp \left( -c \log^2 n \right)$ over $X$ and $V$, where $c$ is an absolute constant.
\end{lemma}
\begin{proof}
    First, we can upper bound the LHS of \eqref{eq:bdex} as
    \begin{equation}\label{eq:o1termfirst}
        \E_{x} \left[  \left(  \varphi(x)^\top \theta^* \right)^2  \right] \leq 2 \mu_1^2 \E_{x} \left[  \left(  x^\top V^\top \theta^* \right)^2  \right] + 2 \E_{x} \left[  \left(  \tilde \varphi(x)^\top \theta^* \right)^2  \right].
    \end{equation}
    We bound the two terms separately.
    Since $x$ is distributed according to $\mathcal P_X$, it is sub-Gaussian with $\subGnorm{x} = \bigO{1}$. Then, we can bound its second moment (see \cite{vershynin2018high}, Proposition 2.5.2) as
    \begin{equation}\label{eq:o1termfirst1}
        \E_{x} \left[  \left(  x^\top V^\top \theta^* \right)^2  \right] \leq C \norm{V^\top \theta^*}_2^2 \leq C \opnorm{V^\top \Phi^+}^2 \norm{Y}_2^2 = \bigO{\frac{1}{n} \, n} = \bigO{1},
    \end{equation}
    where  the third step holds with probability at least $1 - 2 \exp \left( -c_1 \log^2 n \right)$ over $X$ and $V$ because of Lemma \ref{lemma:facts}.
    Next, we bound the second term of \eqref{eq:o1termfirst} as
    \begin{equation}
    \begin{aligned}
        \E_{x} \left[  \left(  \tilde \varphi(x)^\top \theta^* \right)^2  \right] &= \left(\theta^*\right)^\top \E_{x} \left[ \tilde \varphi(x) \tilde \varphi(x)^\top \right] \theta^* \\
        &\leq \opnorm{\E_{x} \left[ \tilde \varphi(x) \tilde \varphi(x)^\top \right]} \norm{\theta^*}_2^2 \\
        &\leq \opnorm{\E_{x} \left[ \tilde \varphi(x) \tilde \varphi(x)^\top \right]} \opnorm{\Phi^+}^2 \norm{Y}_2^2 \\
        &= \bigO{\left(\log^4 n + \frac{p \log^3 d}{d^{3/2}} \right) \, \frac{1}{p} \, n} \\
        &= \bigO{\frac{n \log^4 n}{p} + \frac{n \log^3 d}{d^{3/2}}} = o\left( 1 \right),
    \end{aligned}
    \end{equation}
    where the fourth line holds because of Lemmas \ref{lemma:Eopnormsmall} and \ref{lemma:hyperK} with probability at least $1 - 2 \exp \left( -c_2 \log^2 n \right)$ over $X$ and $V$. The last line provides the desired result. 
\end{proof}

\end{document}